\documentclass[twoside,11pt]{article}

\usepackage{blindtext}

\usepackage{subcaption}
\usepackage{tikz}
\usepackage{pgfplots,pgfplotstable}
\usepgfplotslibrary{fillbetween}

\usepackage{jmlr2e}

\usepackage{amsmath}
\usepackage{algorithm}  
\usepackage{algpseudocode}  
\usepackage{bookmark}
\hypersetup{
	colorlinks   = true, 
	urlcolor     = blue, 
	linkcolor    = blue, 
	citecolor    = red 
}

\usepackage{todonotes} 
\usepackage{cleveref} 
\usepackage{multicol}
\usepackage{bbm}  

\usepackage{constants}

\renewconstantfamily{normal}{
symbol=c,
reset={subsection}
}

\usepackage[acronym]{glossaries}
\glsdisablehyper
\makeglossaries
\newacronym
[
    longplural={Markov decision processes}
]
{MDP}{MDP}{Markov decision process}
\newacronym{iid}{i.i.d.}{independent and identically distributed}
\newacronym{MC}{MC}{Markov chain}
\newacronym{MCMC}{MCMC}{Markov Chain Monte Carlo}
\newacronym{MRF}{MRF}{Markov random field}
\newacronym{PGM}{PGM}{probabilistic graphical model}
\newacronym{RL}{RL}{reinforcement learning}
\newacronym{SAGE}{SAGE}{score-aware gradient estimator}
\newacronym{SGA}{SGA}{stochastic gradient ascent}
\newacronym{SGD}{SGD}{stochastic gradient descent}

\newtheorem{assumption}{Assumption}

\crefname{assumption}{assumption}{assumptions}

\usepackage{amsfonts}  

\newcommand{\bN}{\mathbb{N}}
\newcommand{\bR}{\mathbb{R}}

\newcommand{\cA}{\mathcal{A}}

\newcommand{\cD}{\mathcal{D}}

\newcommand{\cL}{\mathcal{L}}
\newcommand{\cM}{\mathcal{M}}
\newcommand{\cR}{\mathcal{R}}
\newcommand{\cS}{\mathcal{S}}
\newcommand{\cV}{\mathcal{V}}

\newcommand{\fn}{\mathfrak{n}}

\newcommand{\rd}{\mathrm{d}}
\newcommand{\rD}{\mathrm{D}}
\newcommand{\re}{\mathrm{e}}

\newcommand{\uc}{\underline{c}}
\newcommand{\un}{\underline{n}}

\let\originalleft\left
\let\originalright\right
\renewcommand{\left}{\mathopen{}\mathclose\bgroup\originalleft}
\renewcommand{\right}{\aftergroup\egroup\originalright}
\newcommand\prb[1]{\mathbb{P} \left[ #1 \right] }
\newcommand\esp[1]{\mathbb{E} \left[ #1 \right] }
\newcommand\cov[1]{\mathrm{Cov} \left[ #1 \right] }

\newcommand\gradient[1]{\nabla_{\kern -0.2em #1}\hspace{-.02cm}}
\newcommand\accept{\textrm{admit}}
\newcommand\reject{\textrm{reject}}
\newcommand\rdisc{r_{\textrm{disc}}}
\newcommand\rcont{r_{\textrm{cont}}}
\newcommand\flip{\textrm{flip}}
\newcommand\notflip{\textrm{not flip}}
\newcommand\tleft{\textrm{left}}
\newcommand\tright{\textrm{right}}

\definecolor{pltblue}{HTML}{0083b5}
\definecolor{pltcyan}{HTML}{17becf}
\definecolor{pltorange}{HTML}{ff772a}
\definecolor{pltgreen}{HTML}{00a13e}
\definecolor{pltred}{HTML}{ec012b}
\definecolor{pltviolet}{HTML}{9467db}
\definecolor{pltyellow}{HTML}{ffd700}

\usepackage{enumitem}

\newcommand{\R}{\mathbb{R}}
\newcommand{\E}{\mathbb{E}}
\newcommand{\N}{\mathbb{N}}

\usepackage[super]{nth}

\usepackage{dsfont} 

\renewcommand{\d}[1]{\ensuremath{\operatorname{d}\!{#1}}}

\newcommand{\expectation}[1]{ \mathbb{E} [ #1 ] }

\newcommand{\expectationbig}[1]{ \mathbb{E} \bigl[ #1 \bigr] }

\newcommand{\expectationBig}[1]{ \mathbb{E} \Bigl[ #1 \Bigr] }

\newcommand{\ind}{\mathds{1}}
\newcommand{\indicator}[1]{ \mathds{1} [ #1 ] }

\newcommand{\indicatorbig}[1]{ \mathds{1} \bigl[ #1 \bigr] }
\newcommand{\indicatorBig}[1]{ \mathds{1} \Bigl[ #1 \Bigr] }

\newcommand{\process}[2]{ \{ #1 \}_{ #2 } }

\newcommand{\probability}[1]{ \mathbb{P} [ #1 ] }

\newcommand{\probabilityBig}[1]{ \mathbb{P} \Bigl[ #1 \Bigr] }

\newcommand{\naturalNumbersPlus}{ \mathbb{N}_{+} }

\newcommand{\refTheorem}[1]{{\textrm{Theorem~\ref{#1}}}}

\newcommand{\refLemma}[1]{{\textrm{Lemma~\ref{#1}}}}

\newcommand{\QuodEratDemonstrandum}{\hfill \ensuremath{\Box}}

\def\eqcom#1{\overset{\textnormal{(#1)}}}

\def\E{{\mathbb E}}

\def\({{\Bigl(}}
        \def\){{\Bigr)}}

\newcommand{\ba}{\begin{array}}
    \newcommand{\ea}{\end{array}}

\newcommand{\xdeleted}[1]{\deleted{}} 

\allowdisplaybreaks 

\usepackage{longtable}

\jmlrheading{26}{2025}{1-74}{6/24; Revised 1/25}{6/25}{24-1009}{C\'eline Comte, Matthieu Jonckheere, Jaron Sanders, Albert Senen-Cerda}
\ShortHeadings{Score-Aware Policy-Gradient Methods and Performance Guarantees}{Comte, Jonckheere, Sanders, Senen-Cerda}

\firstpageno{1}

\begin{document}

\title{Score-Aware Policy-Gradient and Performance Guarantees using Local Lyapunov Stability}

\author{%
	\name C{\'e}line Comte \email{celine.comte@cnrs.fr}\\
	\name {Matthieu Jonckheere} \email{matthieu.jonckheere@laas.fr}\\
	\addr LAAS–CNRS, Université de Toulouse, CNRS, Toulouse, France \\
	7 Avenue du Colonel Roche, 31400 Toulouse, France
	\AND
	\name{Jaron Sanders} \email jaron.sanders@tue.nl\\
	\addr Eindhoven University of Technology, Eindhoven, The Netherlands \\
	MetaForum, Groene Loper 5, 5612 AZ Eindhoven, The Netherlands
	\AND
	\name{Albert Senen-Cerda} \email albert.senen-cerda@irit.fr\\
	\addr LAAS-CNRS, IRIT, and Universit\'e de Toulouse, Toulouse, France \\
	7 Avenue du Colonel Roche, 31400 Toulouse, France
}

\editor{Nan Jiang}

\maketitle

\begin{abstract}%
In this paper, we introduce a policy-gradient method for model-based reinforcement learning (RL) that exploits a type of stationary distributions commonly obtained from Markov decision processes (MDPs) in stochastic networks, queueing systems, and statistical mechanics. Specifically, when the stationary distribution of the MDP belongs to an exponential family that is parametrized by policy parameters, we can improve existing policy gradient methods for average-reward RL. Our key identification is a family of gradient estimators, called score-aware gradient estimators (SAGEs), that enable policy gradient estimation without relying on value-function estimation in the aforementioned setting. We show that SAGE-based policy-gradient locally converges, and we obtain its regret. This includes cases when the state space of the MDP is countable and unstable policies can exist. Under appropriate assumptions such as starting sufficiently close to a maximizer and the existence of a local Lyapunov function, the policy under SAGE-based stochastic gradient ascent has an overwhelming probability of converging to the associated optimal policy. Furthermore, we conduct a numerical comparison between a SAGE-based policy-gradient method and an actor--critic method on several examples inspired from stochastic networks, queueing systems, and models derived from statistical physics. Our results demonstrate that a SAGE-based method finds close--to--optimal policies faster than an actor--critic method.
\end{abstract}


\section{Introduction}
\label{sec:Intro}

\Gls{RL} has become the primary tool for optimizing controls in uncertain environments.
Model-free \gls{RL}, in particular, can be used to solve generic \glspl{MDP} with unknown dynamics with an agent that learns to maximize a reward incurred upon acting on the environment.
In stochastic systems, examples of possible applications of \gls{RL} can be found in stochastic networks, queueing systems, and particle systems, where an optimal policy is desirable. For example, a policy yielding a good routing policy, an efficient scheduling, or an annealing schedule to reach a desired state.

As stochastic systems expand in size and complexity, however, the \gls{RL} agent must deal with large state and action spaces.
This leads to several computational concerns, namely, the combinatorial explosion of action choices, the computationally intensive exploration and evaluation of policies \citep{qian2019}, and a more complex optimization landscape.

One way to circumvent issues pertaining to large state spaces and/or nonconvex objective functions is to include features of the underlying \gls{MDP} in the \gls{RL} algorithm.
If the model class of the environment is known, a model-based \gls{RL} approach estimates first an approximate model of the environment in the class that can later be used to solve an \gls{MDP} describing its approximate dynamics. This approach is common in queueing networks \citep{liu2022rl, anselmi2023learning}.
Nevertheless, solving an approximate \gls{MDP} adds a computational burden if the number of states is large.

Policy-gradient methods are learning algorithms that instead directly optimize policy parameters through \gls{SGA} \citep{SB18}.
These methods have gained attention and popularity due to their perceived ability to handle large state and action spaces in model-free settings \citep{daneshmand2018,Khadka2018}.
Policy-gradient methods rely on the estimation of value functions, which encode reward-weighted representations of the underlying model dynamics. Computing such functions, however, is challenging in high-dimensional settings and, different from a model-based approach, key model features are initially unknown.

In this paper, we improve policy-gradient methods for some stochastic systems by incorporating model-specific information of the \gls{MDP} into the gradient estimator.
Specifically, we exploit the fact that long-term average behavior of such systems are described using exponential families of distributions.
In the context of stochastic networks and queueing systems, this
typically means that the Markov chains associated to fixed policies
have a product-form stationary distribution. This structural assumption holds in various relevant scenarios, including Jackson and Whittle networks \cite[Chapter~1]{S99}, BCMP networks \citep{BCMP75}, and more recent models arising in datacenter scheduling and online matching \citep{GR20}.
By encoding this key model feature into policy-gradient methods, we aim to expand the current model-based \gls{RL} techniques for control policies of stochastic systems.

Our primary contributions are the following:
\begin{itemize}
	\item We present a new gradient estimator for policy-gradient methods that incorporates information from the stationary measure of the \gls{MDP}.
	Under an average-reward and infinite-horizon learning setting, we namely consider policy parametrizations such that
	there is a known relationship between
	the policy on the one hand,
	and the \gls{MDP}'s stationary distribution on the other hand.
	In practice, this translates to assuming that
	the stationary distribution forms
	an exponential family explicitly depending on the policy parameters.
	Using this structure, we define \glspl{SAGE},
	a class of estimators that exploit the aforementioned assumption to estimate the policy gradient \emph{without} relying on value or action--value functions.
	\item We show the local convergence and bound the regret of a \gls{SAGE}-based policy-gradient under broad assumptions, such as a countable state space, nonconvex objective functions, and unbounded rewards. 
	To do so, we first generalize the approach of \cite{fehrman2020convergence} to a general \gls{RL} setting that includes Markovian updates with a countable state space and does not require the stationary distribution to be exponential.
	We show local convergence by first, using a local Lyapunov function that guarantees stability (i.e., positive recurrence) of the Markov chain as long as the iterates are close to the optimum, and second, by using the nondegeneracy of the Hessian at the optimum, which allows to keep track of the updates locally.
	These two key elements allow us to show convergence if the bias of the gradient estimator and its variance can be controlled.
	Remarkably, our local assumptions may be satisfied even when unstable policies exist.
	For policy gradient with a \gls{SAGE} in particular, we can then crucially estimate its bias and variance explicitly due to the exponential family assumption, and show convergence with large probability by using the aforementioned approach, whenever the trajectory of the iterates gets close enough to an optimum.
	The convergence proof approach is of independent interest, and can also be adapted to other policy gradient-based methods as long as the bias and variance of the policy-gradient estimator can be controlled.
	
	\item We numerically evaluate the performance of \gls{SAGE}-based policy-gradient on several models from stochastic networks, queueing systems, and statistical physics.
	We observe that, compared to an actor--critic algorithm, \gls{SAGE}-based policy-gradient methods exhibit faster convergence and lower variance.
\end{itemize}
Our results suggest that exploiting model-specific information is a promising approach to improve \gls{RL} algorithms, especially for stochastic networks and queueing systems.
\Cref{sec:introduction_sage,sec:introduction_convergence}
below describe our contributions in more details.

\subsection{Score-Aware Gradient Estimators (SAGEs)}
\label{sec:introduction_sage}

We introduce \glspl{SAGE} for \glspl{MDP} following the exponential-family assumption in Section~\ref{sec:sage}.
These estimators leverage the structure of the stationary distribution,
with the goal of reducing variance and favoring stable learning.
Notably, their usage requires neither knowledge nor explicit estimation of model parameters, ensuring practical applicability.
The key step of the derivation exploits information on the form of the \emph{score} of exponential families---that is, the gradient of the logarithm of the probability mass function.

We can illustrate the working principle using a toy example
on a countable state space~$\cS$:
given a \emph{sufficient statistic} $x : \cS \to \bR^n$,
the associated exponential family in canonical form
is the family of distributions
with probability mass functions
$
p( \cdot | \theta)
\propto
\mathrm{exp}( \theta^\intercal x( \cdot) )
$
parametrized by $\theta \in \bR^n$.
Observe now that these distributions satisfy the relation
\begin{equation}
	\frac{d\log(p(s|\theta))}{d \theta}
	=
	x(s) - \E_{S \sim p(\,\cdot\,| \theta) }[x(S)]
	,\quad s \in \cS,
	\label{eqn:exponential_family_introduction}
\end{equation}
and that \eqref{eqn:exponential_family_introduction} gives an exact expression for the gradient of the score.

Now, a more general version of \eqref{eqn:exponential_family_introduction} that is also applicable beyond this toy example---see \Cref{theo:sage} below---allows us to bypass the commonly used policy-gradient theorem \citep[Section~13.2]{SB18}, which ties the estimation of the gradient with that of first estimating value or action--value functions.
A key aspect that \glspl{SAGE} practically exploit is that, in the models from queueing and statistical physics that we will study, we know the sufficient statistic~$x$ fully.
Furthermore, such models commonly possess an `effective dimension' that is much lower than the size of the state space and is reflected by the sufficient statistic. For example, in a load-balancing model we consider in the numerical section, an agnostic model-free \gls{RL} algorithm would learn a value function defined over the complete state space, whose size grows exponentially
in the number~$n$ of servers. On the contrary, a \gls{SAGE} only requires learning the expected value of a single $n$-dimensional vector---the sufficient statistics.

\subsection{Convergence of Policy-Gradient Methods}
\label{sec:introduction_convergence}

We examine the convergence properties of the \gls{SAGE}-based policy-gradient method  theoretically in Section~\ref{sec:convergence}.
Specifically, we consider the setting of policy-gradient \gls{RL} with average rewards, which consists of finding a parameter $\theta$ such that the parametric policy $\pi(\theta) = \pi(\,\cdot\,|\,\cdot\,, \theta)$ maximizes
\begin{equation}
	J(\theta)
	=
	\lim_{T \to \infty} \frac{1}{T}\expectationBig{\sum_{t=1}^{T} R_t}
	.
	\label{eqn:definition_J_intro}
\end{equation}
Here, $R_{t+1}$ denotes the reward that is given after choosing action~$A_t$ while being in state~$S_t$, which happens with probability $\pi(A_t| S_t, \theta)$.
As is common in episodic \gls{RL}, we consider epochs, that is, time intervals where the parameter $\theta$ is fixed and a trajectory of the Markov chain is observed.
For each epoch $m$, and under the exponential-family assumption for the stationary distribution, a \gls{SAGE} yields a gradient estimator $H_m$ from a trajectory of state--action--reward tuples $(S_t, A_t, R_{t+1})$ sampled from a policy with $\Theta_m$ as an epoch-dependent parameter.
Convergence analysis of the \gls{SAGE}-based policy-gradient method aligns with ascent algorithms like \gls{SGA} by considering updates at the end of epoch~$m$ with step-size $\alpha_{m} > 0$,
\begin{equation}
	\Theta_{m+1} = \Theta_{m} + \alpha_{m} H_{m}.
	\label{eqn:update_SGA_intro}
\end{equation}

Convergence analyses for policy-gradient \gls{RL} and \gls{SGA} are quite standard; see Section~\ref{sec:Related works}. 
Our work specifically aligns with the framework of
\cite{fehrman2020convergence}, who study local convergence of unbiased \gls{SGD},
that is, when the conditional estimator $H_m$ of $\nabla J(\Theta_m)$ on the past $\mathcal{F}$ is unbiased.
This occurs typically in supervised learning.
A main contribution in our work consists in expanding the results of \cite{fehrman2020convergence} to the case of Markovian data, leading to biased estimators
(i.e., $\E[H_m| \mathcal{F}] \neq \nabla J( \Theta_m)$).
In our \gls{RL} setting, we handle potentially unbounded rewards and unbounded state spaces as well as the existence of unstable policies.
We also assume an online application of the policy-gradient method, where restarts are impractical or costly: the last state of the prior epoch is used as the initial state for the next, distinguishing our work from typical episodic \gls{RL} setups where an initial state $S_0$ is sampled from a predetermined distribution.

Our main result in Section~\ref{sec:convergence} shows convergence of the \gls{SAGE} parameter update in~\eqref{eqn:update_SGA_intro} to the set~$\mathcal{M}$ that attains the maximum $J^{\star}$ of \eqref{eqn:definition_J_intro}, assuming nondegeneracy of $J$ on $\mathcal{M}$ and existence of a local Lyapunov function. 
If the trajectory of \gls{SGA} ends up within a sufficiently small neighborhood $V$ of a maximizer $\theta^{\star} \in \mathcal{M}$, then with appropriate epoch length and step-sizes, convergence to $\mathcal{M}$ occurs with large probability: for any epoch $m > 0$ and $\epsilon > 0$, if $\Theta_0$ is the first iterate in $V$,
\begin{align}
	\prb{J^{\star} - J(\Theta_{m}) > \epsilon| \Theta_0 \in V}
	&\le
	O\Bigl(
	\epsilon^{-2} m^{-\sigma - \kappa}
	+ m^{1 - \sigma/2 - \kappa/2}
	+ m^{-\kappa/2} + \frac{\alpha^2}{\ell}
	\Bigr),
	\label{eqn:convergence_rate_intro}
\end{align}
where the parameters $\sigma \in (2/3, 1), \kappa > 0, \alpha \in (0, \alpha_0]$, and $\ell \in [\ell_0, \infty)$ depend on the step and batch sizes and can be tuned to make the bound in \eqref{eqn:convergence_rate_intro} arbitrarily small.
While our focus is on the global optimum, the bound \eqref{eqn:convergence_rate_intro} also holds under the same assumptions in case $J^{\star}$ is a local optimum instead.

Our key assumption relies on the existence of a local Lyapunov function in the neighborhood~$V$, which ensures stability.
Crucially, this assumption is required only for policies that are close to an optimal policy.
This sets our work further apart from others in the \gls{RL} literature, which typically require the existence of a global Lyapunov function and/or the state space to be finite. 
In fact, our numerical results in \Cref{sec:examples} show an instance where stability around a global optimizer is sufficient, highlighting the benefits of \gls{SAGE}. 
The set~$\mathcal{M}$ of global maxima is also not required to be finite or convex, thanks to the local nondegeneracy assumption.

For a large batch index~$m$, the bound in \eqref{eqn:convergence_rate_intro} can be made arbitrarily small by setting the initial step size~$\alpha$ small and batch size~$\ell$ large.
In \eqref{eqn:convergence_rate_intro}, the chance that the policy escapes the set $V$, outside which stability cannot be guaranteed, does not vanish when $m \to \infty$; it remains as $\alpha^2/\ell$.
We show that this term is inherent to the statistical estimation error of the gradient and cannot be avoided.
Specifically, for any $\beta > 0$, there are objective functions $f$ such that  $\prb{f(\theta^\star)- f(\Theta_{m}) > \epsilon| \Theta_0 \in V} > c \alpha^{2 + \beta}/\ell$ for some $c>0$.
Hence, a lower bound shows that the proof method cannot be improved without using additional structure of $H_m$ or $J$. Furthermore, our proof can be adapted to other generic policy-gradients that have similar bounds on the gradient estimator $H_m$ as those of \gls{SAGE}, thus showing that such phenomenon not just happens with \gls{SAGE} but also with any other policy-gradient algorithm with similar properties of the gradient estimator.

Denoting by~$T \geq 1$ the total number of samples drawn by the algorithm, we obtain from \eqref{eqn:convergence_rate_intro} a regret bound of our algorithm when reaching the set $V$. 
In the case of bounded rewards, we namely show that for any $0 < \epsilon < 1$, if $\Theta_0$ is the first iterate in $V$, then
\begin{equation}
	\expectationBig{TJ^\star - \sum_{t=1}^{T} r(S_t, A_t) ~\Big|~     \Theta_0 \in V} = O\Bigl( (\mathcal{L}^{\star})^{\frac{1}{3}}T^{\frac{2}{3} + \epsilon} + \frac{\alpha^2}{\ell} T\Bigr).
	\label{eqn:regret_intro}
\end{equation}
The linear term in \eqref{eqn:regret_intro} arises from the estimation error of the policy gradient and has been seen in other recent works \citep{abbasi2019politex}, and also for countable state spaces \citep{murthy2024performance}. 
The other term is sublinear, and its coefficient $\mathcal{L}^{\star}$ characterizes an `effective' size of the state space, and directly depends on the local Lyapunov function.
For a given $T$, we can find $\ell$---a term related to the batchsize and thus estimation error---such that the regret becomes $O(T^{3/4 + \epsilon})$. 
Remarkably, while we start from a stable policy, the expectation in \eqref{eqn:regret_intro} includes trajectories where policies may be unstable. When the reward is unbounded, we similarly obtain a bound without a linear term if we restrict to trajectories in $V$.

For cases where the optimum is reached only as $\Theta_m \to \infty$, as with deterministic policies, we show that adding a small relative entropy regularization term to $J(\theta)$ ensures that maxima are bounded and that $\mathcal{M}$ satisfies the nondegeneracy assumption required to show local convergence.

\subsection{Numerical Experiments}

We finally assess the applicability of the \gls{SAGE}-based policy-gradient algorithm in \Cref{sec:examples} by comparing its performance with that of the actor--critic algorithm on three models from queueing systems, stochastic networks, and statistical physics.
Specifically, we consider an admission control problem in the M/M/1 queue, a load balancing system, and the Ising model with Glauber dynamics.

These numerical results suggest that, when applicable,
\glspl{SAGE} can expedite convergence towards an optimal policy (compared to actor--critic) by leveraging the structure of the stationary distribution.
Furthermore, the lower variance of \gls{SAGE} becomes decisive when stability is not guaranteed for all policies.
Namely, we observe in an example that the \gls{SAGE}-based policy-gradient method converges to a close--to--optimal policy even if some policies are unstable, provided that a stable policy is used as initialization.
This behavior contrasts with actor--critic, whose output policies are not always stable.
\gls{SAGE} also reproduces a well-known phenomenon in annealing schedules for Ising models. Specifically, the agent momentarily increases the temperature in order to escape stable states that do not correspond to the global optimum.

\section{Related Works}
\label{sec:Related works}

The work in the present manuscript resides at the intersection
of distinct lines of research.
We therefore broadly review, relate, and position our work to other research in this section.

\subsection{Gradient Estimation, Exponential Families, and Product Forms} \label{sec:rw-gradient}

Operations on high-dimensional probability distributions,
such as marginalization and inference, are numerically intractable in general.
Exponential families---see \Cref{sec:sage-ass} for a definition---are parametric sets of distributions
that lead to more tractable operations and approximations
while also capturing well-known probability distributions,
such as probabilistic graphical models~\citep{WJ08},
popular in machine learning.
In the context of stochastic networks and queueing systems,
the stationary distribution of many product-form systems
can be seen as forming an exponential family.

Our first contribution is related to several works on exponential families, product-form distributions, and probabilistic graphical models. Key performance metrics in these distributions are numerically intractable \textit{a priori}, but can be expressed as expectations of random vectors that can be sampled by simulation.
The most basic and well-known result, which appears in \Cref{sec:introduction_sage} and will be exploited in \Cref{sec:sage-estimator}, rewrites the gradient of the logarithm of the normalizing constant (a.k.a.\ the log-partition function) as the expectation of the model's \emph{sufficient statistics}. In probabilistic graphical models, this relation has been mainly used to learn a distribution that best describes a data set via \gls{SGD}~\citep{WJ08,KFB09}. In stochastic networks, this relation has been applied to analyze systems with \emph{known} parameters, for instance to predict their performance~\citep{SSM88,ZZ99,BV04,S11,SV15}, to characterize their asymptotic behavior in scaling regimes~\citep{S11,SV15}, for sensitivity analysis~\citep{SSM88,LN91}, and occasionally to optimize control parameters via gradient ascent~\citep{LN91,SSG91,S11}.

To the best of our knowledge, an approach similar to ours is that introduced by \cite{SBL16}.
This work derives a gradient estimator and performs \gls{SGA} in a class of product-form reversible networks.
However, the procedure requires first estimating the stationary distribution, convergence is proven only for convex objective functions, and the focus is more on developing a distributed algorithm than on canonical \gls{RL}.
The algorithm of \cite{jiang2009distributed} is similarly noteworthy, although the focus there is on developing a distributed control algorithm specifically for wireless networks and not general product-form networks.

\subsection{Stochastic Gradient Ascent (SGA) and Policy-Gradient Methods}
\label{sec:rw-sga}

When a gradient is estimated using samples from a Markov chain, methods from \gls{MCMC} are commonly used \citep{mohamed2020monte}.
In our case, we have moreover bias from being unable to restart the chain at each epoch.
Convergence of biased \gls{SGD} to approximate stationary points of smooth nonconvex functions---points~$\theta$ such that $|\nabla J(\theta)| < \epsilon$ for some $\epsilon > 0$---has been addressed in the literature \citep{tadic2017asymptotic, atchade2017perturbed, karimi2019non, doan2020finite}.
The asymptotic conditions for local convergence to a stationary point were first investigated by \cite{tadic2017asymptotic}, who assumed conditions for the asymptotic stochastic variance of the gradient estimator and bias (see Assumptions~$2.1$--$2.3$, \citealt{tadic2017asymptotic}).
\cite{karimi2019non} showed a nonasymptotic analysis of biased \gls{SGD}.
Under Lipschitz assumptions on the transition probabilities and bounded variance of the gradient estimator $H_m$, \cite{karimi2019non} showed that, under appropriate step-sizes, for some $m^{\star} \leq M$, we have $\E[ | \nabla J(\Theta_{m^*}) |^2 ] = O(\log(M)/\sqrt{M})$, where $M$ is a time horizon.
\cite{tadic2017asymptotic} and \cite{karimi2019non} applied these results in an \gls{RL} context. 
While these works demonstrate convergence to stationary points, our contribution lies in proving convergence to a maximum, albeit locally. 
This approach is essential for addressing scenarios with only local assumptions and potentially unstable policies (i.e., policies such that the corresponding Markov chain is not positive recurrent).

Finally, several recent works build on gradient domination for policy-gradient methods, addressing convexity limitations and ensuring global convergence \citep{fazel2018global,agarwal2021theory,xiao2022convergence,kumar2024global}. 
Notable differences to our work are their use of a finite state space and that we do not consider an episodic setting with restarts, as well as distinct structural assumptions on policy parametrization like natural gradients.
\cite{murthy2024performance} considered a convergent natural policy-gradient (NPG) method for countable state spaces where the cost is the norm of the state, i.e., the queue length in the queueing setting.
In this paper, a stable max-weight policy is applied with a probability that tends to 1 as the queue size increases, therefore avoiding instability issues. We do not use such a stabilizing policy, and instead require the algorithm to start sufficiently close to an optimal policy.
Another unique aspect of our contribution lies in specialized gradient estimation schemes based on the exponential family assumption on the stationary distribution, which crucially avoids estimating or learning value functions, common in all previous works.

A succinct non-exhaustive table is provided in \Cref{tab:references} to guide interested readers to key references within this vast body of literature.

\begin{table}[htb]
	\tiny \begin{longtable}{|>{\raggedright\arraybackslash}m{2cm}|>{\raggedright\arraybackslash}m{2.8cm}|>{\raggedright\arraybackslash}m{3cm}|>{\raggedright\arraybackslash}m{2.3cm}|>{\raggedright\arraybackslash}m{2.9cm}|}
		\hline \textbf{Method} & \textbf{Context/State Space} & \textbf{Main Assumptions} & \textbf{Convergence} & \textbf{References}
		\\ \hline \hline
		\gls{SGD} & Markovian data with convex objectives & Lipschitz transition probabilities, bounded variance estimators & Local convergence& \cite{tadic2017asymptotic,daneshmand2018,karimi2019non}
		\\ \hline
		\gls{SGD} & Iid data with non convex objectives & Local convexity & Local convergence & \cite{fehrman2020convergence}
		\\ \hline \hline
		Policy Gradient & RL, Finite state space & Gradient domination & Global convergence & \cite{agarwal2021theory}, \cite{kumar2024global}
		\\ \hline
		Policy Gradient & RL, Finite state space & ABC condition & Global convergence & \cite{yuan2022general}
		\\ \hline
		Natural Policy Gradient & RL, Finite state space & Gradient domination / entropy regularization & Global convergence & \cite{Li2020,agarwal2021theory}, \cite{cen2022fast}
		\\ \hline \hline
		Natural Policy Gradient & RL, Countable state space & Stabilization via a known policy independently of the model and policy parameters & Global convergence & \cite{murthy2024performance}
		\\ \hline
		Policy gradient for product-form Networks & Countable state space & Convex objectives, requires knowing the functional form of the gradient in terms of the stationary distribution & Global convergence for convex objectives & \cite{SBL16}
		\\ \hline
		\gls{SAGE}-based policy gradient & Countable state space & Lyapunov assumptions, Non-degenerate Hessian, Application to exponential family stationary distributions& Local convergence & This paper
		\\ \hline
	\end{longtable}
	\caption{This table provides a non-exhaustive list of convergence results to help contextualize our contribution. Compared to the (few) results that can be applied to \glspl{MDP} with a countably-infinite state space, our convergence result only requires local convexity and positive-recurrence assumptions in the parameter space.}
	\label{tab:references}
\end{table}

\section{Problem Formulation}
\label{sec:Problem}

After introducing basic notation in \Cref{sec:Notation}, we introduce \glspl{MDP} in \Cref{sec:MDP} and the infinite-horizon average-reward optimality criterion in \Cref{sec:stationary}. \Cref{sec:policy-gradient} gives a brief introduction to policy-gradient algorithms.

\subsection{Basic Notation}
\label{sec:Notation}

The sets of
nonnegative integers,
positive integers,
reals,
and nonnegative reals
are denoted by
$\bN$, $\bN_+$, $\bR$, and $\bR_{\ge 0}$,
respectively.
For a differentiable function
$f : \theta \in \bR^n \mapsto f(\theta) \in \bR$,
$\nabla f(\theta)$ denotes
the gradient of $f$ taken at $\theta \in \bR^n$,
that is, the $n$-dimensional column vector
whose $j$-th component is
the partial derivative of $f$
with respect to $\theta_j$,
for $j \in \{1, 2, \ldots, n\}$.
If $f$ is twice differentiable,
$\mathrm{Hess}_{\theta} f$ denotes
the Hessian of $f$ at~$\theta$, that is,
the $n \times n$ matrix of second partial derivatives.
For a differentiable vector function
$f : \theta \in \bR^n
\mapsto f(\theta) = (f_1(\theta), \ldots, f_d(\theta)) \in \bR^d$,
$\rD f(\theta)$ is
the Jacobian matrix of $f$ taken at~$\theta$,
that is, the $d \times n$ matrix
whose $i$-th row is ${\nabla f_i(\theta)}^\intercal$,
for $i \in \{1, 2, \ldots, d\}$.
For a vector $x = (x_1, \ldots, x_n) \in \R^{n}$, we denote its $l_2$-norm by $|x| = \sqrt{ x_1^2 + \cdots + x_n^{2}}$.
We define the operator norm of a matrix $A \in \R^{a \times b}$ as $|A|_{\mathrm{op}} = \sup_{x \in \R^{b}: |x| = 1} |Ax|$.
We use uppercase
to denote random variables and vectors,
and a calligraphic font for their sets
of outcomes.

\subsection{Markov Decision Process (MDP)}
\label{sec:MDP}

We consider a \acrfull{MDP} with
countable state, action, and reward spaces~$\cS$,~$\cA$, and $\cR$, respectively,
and transition probability kernel
$P : (s, a, r, s') \in \cS \times \cA \times \cR \times \cS
\mapsto P(r, s' | s, a) \in [0, 1]$.
Thus
$P(r, s' | s, a)$ gives the conditional probability
that the next reward--state pair is $(r, s')$
given that the current state-action pair is $(s, a)$.
With a slight abuse of notation, we introduce
\begin{align*}
	P(r | s, a)
	&= \sum_{s' \in \cS} P(r, s' | s, a),
	\quad \text{for each $s \in \cS$, $a \in \cA$, and $r \in \cR$, and} \\
	P(s' | s, a)
	&= \sum_{r \in \cR} P(r, s' | s, a),
	\quad \text{for each $s, s' \in \cS$ and $a \in \cA$}.
\end{align*}
Our results also generalize
to absolutely continuous rewards;
an example will appear in \Cref{sec:Admission-control-in-an-MM1}.

Following the framework of policy-gradient algorithms
\cite[Chapter~13]{SB18},
we assume that the agent is given a random policy parametrization
$\pi: (s, \theta, a) \in \cS \times \bR^n \times \cA
\to \pi(a | s, \theta) \in (0, 1)$,
such that $\pi(a | s, \theta)$ is
the conditional probability that the next action is~$a \in \cA$
given that the current state is~$s \in \cS$ and the parameter vector~$\theta \in \bR^n$.
We assume that the function $\theta \mapsto \pi(a | s, \theta)$
is differentiable
for each $(s, a) \in \cS \times \cA$.
The goal of the learning algorithm
will be to find a parameter (vector) that maximizes
the long-run average reward;
see \Cref{sec:stationary}.

As a concrete example,
we will often consider a class of softmax policies
that depend on a feature extraction map $\xi : \cS \times \cA \to \bR^n$ as follows:
\begin{align}
	\pi(a | s, \theta)
	&= \frac{e^{\theta^\intercal \xi(s, a)}}{\sum_{a' \in \cA} e^{\theta^\intercal \xi(s, a')}},
	\quad s \in \cS,
	\quad a \in \cA.
	\label{eqn:defintion_softmax_policy}
\end{align}
The feature extraction map~$\xi$
may leverage prior information on the system dynamics.
In queueing systems for instance,
we may decide to make similar decisions in large states,
as these states are typically visited rarely,
and it may be beneficial to aggregate the information collected about them.

\subsection{Stationary Analysis and Optimality Criterion}
\label{sec:stationary}

Given $\theta \in \bR^n$, if the agent applies
the policy~$\pi(\theta) : (s, a) \in \cS \times \cA \mapsto \pi(a | s, \theta)$ at every time step, the random state--action--reward sequence $((S_t, A_t, R_{t+1}), t \in \bN)$ obtained by running this policy is a Markov chain such that,
for each $s, s' \in \cS$,
$a \in \cA$, and $r \in \cR$,
$\prb{A_t = a | S_t = s} = \pi(a | s, \theta)$
and
$\prb{R_{t+1} = r, S_{t+1} = s' | S_t = s, A_t = a}
= P(r, s' | s, a)$.
The dependency of the random variables on the parameter
is left implicit to avoid cluttering notation.
Leaving aside actions and rewards,
the state sequence $(S_t, t \in \bN)$ also defines
a Markov chain,
with transition probability kernel
$P(\theta) : (s, s') \in \cS \times \cS \mapsto P(s' | s, \theta)$
given~by
\begin{align*}
	P(s' | s, \theta)
	&= \sum_{a \in \cA}
	\pi(a | s, \theta)
	P(s' | s, a),
	\quad s, s' \in \cS.
\end{align*}
In the remainder, we assume that \Cref{ass:markov,ass:reward} below
are satisfied.
\begin{assumption} \label{ass:markov}
	There exists an open set $\Omega \subseteq \bR^n$ such that, for each $\theta \in \Omega$, the Markov chain $(S_t, t \in \bN)$ with transition probability kernel~$P(\theta)$ is irreducible and positive recurrent.
\end{assumption}

In the remainder, we use the words
\emph{positive recurrent} and \emph{stable} interchangeably.
Also, with a slight abuse of language, we say that
a policy is stable if the corresponding Markov chain is stable.
Thanks to \Cref{ass:markov},
for each $\theta \in \Omega$,
the corresponding Markov chain $(S_t, t \in \bN)$
has a unique stationary distribution $p(\cdot | \theta)$.
We say that a triplet $(S, A, R)$ of random variables
is a \emph{stationary state--action--reward triplet},
and we write $(S, A ,R) \sim$~\ref{eq:stat},
if $(S, A, R)$ follows the stationary distribution
of the Markov chain $((S_t, A_t, R_{t+1}), t \in \bN)$ given by
\begin{align}
	\tag{\textsc{stat}($\theta$)} \label{eq:stat}
	\prb{S = s, A = a, R = r}
	&= p(s | \theta) \pi(a | s, \theta) P(r | s, a),
	\quad s \in \cS,
	\quad a \in \cA,
	\quad r \in \cR.
\end{align}

\begin{assumption} \label{ass:reward}
	For each $\theta \in \Omega$,
	the stationary state--action--reward triplet
	$(S, A, R) \allowbreak \sim$ \ref{eq:stat} is such that $|R|$, $|R \, \nabla \log p(S | \theta)|$, and $|R \, \nabla \log \pi(A | S, \theta)|$ have a finite expectation.
\end{assumption}

By ergodicity~\cite[Theorem~4.1]{B99},
the running average reward
$\frac1T \sum_{t = 1}^T R_t$
tends to $J(\theta)$ as defined in \eqref{eqn:definition_J_intro} almost surely as $T$ tends to infinity. $J(\theta)$ is called the \emph{long-run average reward}
and is also given by
\begin{align} \label{eq:J}
	J(\theta)
	&= \esp{R}
	= \sum_{s \in \cS} \sum_{a \in \cA} \sum_{r \in \cR}
	p(s | \theta) \pi(a | s, \theta) P(r | s, a) r,
	\quad \theta \in \Omega.
\end{align}
Our end goal, further developed in \Cref{sec:policy-gradient},
is to find a learning algorithm that maximizes the objective function~$J$.
For now, we only observe that
the objective function
$J : \theta \in \Omega \mapsto J(\theta)$
is differentiable thanks to \Cref{ass:reward},
and that its gradient is given by
\begin{align} \label{eq:dJ}
	\nabla J(\theta)
	&= \sum_{s \in \cS} \sum_{a \in \cA} \sum_{r \in \cR}
	p(s | \theta) \pi(a | s, \theta) P(r | s, a) r
	\left(
	\nabla \log p(s | \theta)
	+ \nabla \log \pi(a | s, \theta)
	\right),
	\quad \theta \in \Omega.
\end{align}
In general, computing $\nabla J(\theta)$
using~\eqref{eq:dJ} is challenging:
(i) computing $\nabla \log p(s | \theta)$ is in itself challenging because $p(s | \theta)$ depends in a complex way on the unknown transition kernel~$P(r, s' | s, a)$ and the parameter~$\theta$ via the policy~$\pi(\theta)$, and
(ii) enumerating and thus summing over the state space~$\cS$ is often practically infeasible (for instance, when the state space~$\cS$ is infinite and/or high-dimensional).
Our first contribution, in \Cref{sec:sage}, is precisely a new family of estimators for the gradient~\eqref{eq:dJ}.

\subsection{Learning Algorithm} \label{sec:policy-gradient}

In \Cref{sec:stationary},
we defined the objective function~$J$
by considering trajectories
where the agent applied a policy $\pi(\theta)$
parametrized by a constant vector~$\theta$.
Going back to a learning setting,
we consider a state--action--reward sequence
$((S_t, A_t, R_{t+1}), t \in \bN)$
and a parameter sequence
$(\Theta_m, m \in \bN)$
obtained by updating the parameter periodically
according to the gradient-ascent step
$\Theta_{m+1} = \Theta_m + \alpha_m H_m$
introduced in \eqref{eqn:update_SGA_intro}.
Here, $H_m$ is provided by a family of learning algorithms,
called \emph{policy gradient}.
The pseudocode of a generic policy-gradient algorithm, shown in \Cref{algo:policy-gradient}, is parametrized by a sequence $0 \triangleq t_0 < t_1 < t_2 < \ldots $ of \emph{observation times}
and a sequence
$\alpha_0, \alpha_1, \alpha_2, \ldots > 0$
of \emph{step sizes}.
For each $m \in \bN$,
$\cD_m$ denotes batch~$m$,
obtained by applying policy $\pi(\Theta_m)$ at epoch~$m$, given~by
\begin{align} \label{eq:batch}
	\cD_m = \bigl( (S_t, A_t, R_{t+1}), t \in \{t_m, \ldots, t_{m+1} - 1\} \bigr).
\end{align}
Given some initialization $\Theta_0$,
\Cref{algo:policy-gradient} calls the \textsc{Gradient} procedure that computes an estimate $H_m$ of $\nabla J(\Theta_m)$ from $\cD_m$,
and it updates the parameter according to \eqref{eqn:update_SGA_intro}.

\begin{algorithm}[t]
	\caption{Generic policy-gradient algorithm.
		Examples of \textsc{Gradient} procedure, based on different estimators for the gradient $\nabla J$, are given in \Cref{algo:sage,algo:actor-critic}.
		All variables of \Cref{algo:policy-gradient}
		are accessible within the \textsc{Gradient} procedure.
	}
	\label{algo:policy-gradient}
	
	\begin{algorithmic}[1]
		\State \textbf{Input:}
		\begin{minipage}[t]{.8\textwidth}
			$\bullet$
			Observation times $0 \triangleq t_0 < t_1 < t_2 < \ldots$ \\
			$\bullet$
			Step size sequence $\alpha_0, \alpha_1, \alpha_2, \ldots > 0$ \\
			$\bullet$
			Positive and differentiable
			policy parametrization $(s, \theta, a) \mapsto \pi(a | s, \theta)$
		\end{minipage}
		\vspace{0.25em}
		
		\State \textbf{Initialization:}
		\begin{minipage}[t]{.7\textwidth}
			Policy parameter $\Theta_0 \in \Omega$
			and initial state $S_0 \in \cS$
		\end{minipage}
		\vspace{0.25em}
		
		\State \textbf{Main loop:}
		\For{$m = 0, 1, 2, \ldots$}
		\For{$t = t_m, \ldots, t_{m+1} - 1$} \label{step:batch-start}
		\State Sample $A_t \sim \pi(\cdot | S_t, \Theta_m)$
		\State Take action $A_t$ and observe $R_{t+1}, S_{t+1}$
		\EndFor \label{step:batch-end}
		\State Update $\Theta_{m+1} \gets \Theta_m + \alpha_m \textsc{Gradient}(m)$ \label{step:sga}
		\EndFor
	\end{algorithmic}
\end{algorithm}

As discussed at the end of \Cref{sec:stationary},
finding an estimator~$H_m$ for $\nabla J(\Theta_m)$
directly from \eqref{eq:J} is difficult in general.
A common way to obtain $H_m$
follows from the policy-gradient theorem
\cite[Chapter~13]{SB18},
which instead writes the gradient $\nabla J(\theta)$
using the action-value function~$q$:
\begin{align*}
	\nabla J(\theta)
	&= \esp{
		q(S, A)
		\,
		\nabla \log \pi(A | S, \theta)
	},
\end{align*}
where $(S, A, R) \sim$ \ref{eq:stat},
for each $\theta \in \Omega$.
Consistently, in a model-free setting,
policy-gradient methods like the actor--critic algorithm (recalled in \Cref{app:actor--critic}) estimate $\nabla J(\Theta_m)$ by first estimating a value function. 
However, this approach can suffer from high-variance of the estimator, which slows down convergence, as described in \Cref{sec:Intro}.
Some of these problems can be circumvented by exploiting the problem structure,
as we will see now.

\section{Score-Aware Gradient Estimator (SAGE)}
\label{sec:sage}

We now define the key structural assumption in our paper.
Namely, that we have information on the impact of the policy parameter~$\theta$
on the stationary distribution~$p$.
In \Cref{sec:sage-estimator}, we will use this assumption
to build \glspl{SAGE}, a new family of estimators for the gradient~$\nabla J$
that do not involve the state-value function,
contrary to actor--critic.
In \Cref{sec:sage-algo},
we will further explain how to use this insight
to design a \gls{SAGE}-based policy-gradient method.

\subsection{Product-Form and Exponential Family} \label{sec:sage-ass}

As announced in the introduction,
our end goal is to design a gradient estimator
capable of exploiting information
on the stationary distribution $p(\cdot | \theta)$ of the \gls{MDP}
when such information is available.
\Cref{ass:stat} below formalizes this idea
by assuming that the stationary distribution forms an exponential family
parametrized by the policy parameter~$\theta$.

\begin{assumption}[Stationary Distribution] \label{ass:stat}
	There exist
	a scalar function
	$\Phi: \cS \to \bR_{>0}$,
	an integer $d \in \bN_+$,
	a differentiable vector function
	$\rho : \Omega \to \bR_{>0}^d$,
	and a vector function
	$x : \cS \to \bR^d$
	such that
	the following two equivalent equations are satisfied:
	\begin{subequations}
		\label{eq:p}
		\begin{align}
			\tag{\ref{eq:p}--PF}
			\label{eq:p-pf}
			p(s | \theta)
			& = \frac1{Z(\theta)} \Phi(s) \prod_{i = 1}^d {\rho_i(\theta)}^{x_i(s)},
			&& s \in \cS,
			\quad \theta \in \Omega, \\
			\tag{\ref{eq:p}--EF}
			\label{eq:p-ef}
			\log p(s | \theta)
			& = \log \Phi(s) + \log \rho(\theta)^\intercal x(s) - \log Z(\theta),
			&& s \in \cS,
			\quad \theta \in \Omega,
		\end{align}
	\end{subequations}
	where the partition function
	$Z: \Omega \to \bR_{> 0}$
	follows by normalization:
	\begin{align}
		\label{eq:Z}
		Z(\theta)
		&= \sum_{s \in \cS} \Phi(s)
		\prod_{i = 1}^d {\rho_i(\theta)}^{x_i(s)}
		= \sum_{s \in \cS} \re^{
			\log \Phi(s)
			+ \log \rho(\theta)^\intercal x(s)
		},
		\quad \theta \in \Omega.
	\end{align}
	We will call
	$\Phi$ the \emph{balance} function,
	$\rho$ the \emph{load} function,
	and $x$ the \emph{sufficient statistics}.
\end{assumption}

\eqref{eq:p-pf}
is the product-form variant of the stationary distribution,
classical in queueing theory.
\eqref{eq:p-ef}
is the exponential-family description of the distribution.
This latter representation is more classical in machine learning~\citep{WJ08}
and will simplify our derivations.
Let us briefly discuss the implications of this assumption
as well as examples where this assumption is satisfied.

\Cref{ass:stat} implies that the stationary distribution~$p$
depends on the policy parameter~$\theta$ only via the load function~$\rho$.
Yet, this assumption may not seem very restrictive \textit{a priori}.
Assuming for instance that the state space~$\cS$ is finite,
with $\cS = \{s_1, s_2, \ldots, s_N\}$,
we can write the stationary distribution in the form~\eqref{eq:p}
with $d = N$,
$\rho_i(\theta) = p(s_i | \theta)$,
$x_i(s) = \indicator{s = s_i}$,
and $\Phi(s) = Z(\theta) = 1$,
for each $\theta \in \bR^n$, $s \in \cS$,
and $i \in \{1, 2, \ldots, N\}$.
However, writing the stationary distribution in this form
is not helpful,
in the sense that in general the function~$\rho$ will be prohibitively intricate.
As we will see in \Cref{sec:sage-estimator},
what will prove important in \Cref{ass:stat} is that
the load function~$\rho$ is simple enough
so that we can evaluate its Jacobian matrix function $\rD \log \rho$ numerically.

There is much literature
on stochastic networks and queueing systems
with a stationary distribution of the form~\eqref{eq:p-pf}.
Most works focus on performance evaluation,
that is, evaluating~$J(\theta)$ for some parameter~$\theta \in \Omega$,
assuming that the \gls{MDP}'s transition probability kernel is known.
In this context, the product-form~\eqref{eq:p-pf} arises in
Jackson and Whittle networks~\cite[Chapter~1]{S99},
BCMP networks~\citep{BCMP75},
as well as more recent models arising in datacenter scheduling
and online matching~\citep{GR20}\footnote{\label{footnote}%
	Although the distributions recalled in
	\cite[Theorems~3.9, 3.10, 3.13]{GR20}
	do not seem to fit the framework of \eqref{eq:p} \emph{a priori}
	because the number of factors in the product can be arbitrarily large,
	some of these distributions can be rewritten in the form~\eqref{eq:p}
	by using an expanded state descriptor,
	as in \cite[Equation~4, Corollary~2, and Theorem~6]{ABMW17}
	and \cite[Equation~7 and Proposition~3.1]{MBM21}.
}.
Building on this literature, in \Cref{sec:examples},
we will consider policy parametrizations
for control problems
that also lead to a stationary distribution of the form~\eqref{eq:p}.

In the next section,
we exploit Assumption~\ref{ass:stat} to construct a gradient estimator
that requires knowing the functions
$\rD \log \rho$ and $x$
but not the functions~$\rho$, $\Phi$, and~$Z$.

\subsection{Score-Aware Gradient Estimator (SAGE)}
\label{sec:sage-estimator}

As our first contribution,
\Cref{theo:sage} below
gives simple expressions for
$\nabla \log p(s | \theta)$
and $\nabla J(\theta)$
under \Cref{ass:markov,ass:reward,ass:stat}.
Gradient estimators that will be formed using~\eqref{eq:sage}
will be called \acrfullpl{SAGE},
to emphasize that the estimators rely
on the simple expression~\eqref{eq:dlogp}
for the score $\nabla \log p(s | \theta)$.
Particular cases of this result
have been obtained by~\cite{SSM88,SSG91,LN91}
for specific stochastic networks;
our proof is shorter and more general
thanks to the exponential form~\eqref{eq:p-ef}.

\begin{theorem} \label{theo:sage}
	Suppose that \Cref{ass:markov,ass:reward,ass:stat} hold.
	For each $\theta \in \Omega$, we have
	\begin{align}
		\label{eq:dlogp}
		\nabla \log p(s | \theta)
		&= \rD \log \rho(\theta)^\intercal (x(s) - \esp{x(S)}), \\
		\label{eq:sage}
		\nabla J(\theta)
		&= \rD \log \rho(\theta)^\intercal \cov{R, x(S)}
		+ \esp{R \nabla \log \pi(A | S, \theta)},
	\end{align}
	where $(S, A, R) \sim$~\ref{eq:stat},
	$\cov{R, x(S)}
	= (\cov{R, x_1(S)}, \ldots, \cov{R, x_d(S)})^\intercal$,
	and the gradient and Jacobian operators,~$\nabla$ and $\rD$ respectively,
	are taken with respect to~$\theta$.
\end{theorem}

\begin{proof}
	Applying the gradient operator
	to the logarithm of \eqref{eq:Z}
	and simplifying yields
	\begin{align} \label{eq:dlogZ}
		\nabla \log Z(\theta)
		&= \rD \log \rho(\theta)^\intercal \esp{x(S)}.
	\end{align}
	This equation is well-known and was already discussed in \Cref{sec:rw-gradient}.
	Equation~\eqref{eq:dlogp} follows by
	applying the gradient operator to~\eqref{eq:p-ef}
	and injecting~\eqref{eq:dlogZ}.
	Equation~\eqref{eq:sage} follows by
	injecting~\eqref{eq:dlogp} into~\eqref{eq:dJ}
	and simplifying.
\end{proof}

Assuming that the functions $\rD \log \rho$ and $x$
are known in closed-form,
\Cref{theo:sage} allows us
to construct an estimator of~$\nabla J(\theta)$
from a state--action--reward sequence
$((S_t, A_t, R_{t+1}), t \in \{0, 1, \ldots, T\})$
obtained by applying policy~$\pi(\theta)$ at every time step
as follows:
\begin{align}
	\label{eq:estimator}
	H
	&= \rD \log \rho(\theta)^\intercal
	\overline{C}
	+ \overline{E},
\end{align}
where $\overline{C}$ and $\overline{E}$
are estimators of $\cov{R, x(S)}$
and $\esp{R \, \nabla \log \pi(A | S, \theta)}$, respectively,
obtained for instance by taking
the sample mean and sample covariance.
An estimator of the form~\eqref{eq:estimator}
will be called a \acrfull{SAGE}.
This idea will form the basis of
the \gls{SAGE}-based policy-gradient method
that will be introduced in \Cref{sec:sage-algo}.
Observe that such an estimator will typically be biased
since the initial state~$S_0$ is not stationary. Nonetheless, we will show in the proof of
the convergence result in \Cref{sec:convergence} that this bias does not prevent convergence.

The advantage of using a \gls{SAGE} is twofold.
First, the challenging task of estimating $\nabla J(\theta)$
is reduced to the simpler task of estimating
the $d$-dimensional covariance $\cov{R, x(S)}$
and the $n$-dimensional expectation
$\esp{R \, \nabla \log \pi(A | S, \theta)}$,
for which leveraging estimation techniques in the literature is possible.
Also recall that the gradient estimator used in the
actor--critic algorithm (\Cref{app:actor--critic})
relies on the state-value function,
so that it requires estimating $|\cS|$ values;
we therefore anticipate \glspl{SAGE}
to yield better performance when $\max(n, d) \ll |\cS|$;
see examples from \Cref{sec:Load-balancing-system,sec:ising}.
Second, as we will also observe in \Cref{sec:examples},
\glspl{SAGE} can exploit information on the structure of the policy and stationary distribution ``by design''.
Actor--critic exploits this information only indirectly
due to its dependency on the state-value function.

\subsection{\texorpdfstring{\gls{SAGE}}{SAGE}-Based Policy-Gradient Algorithm}
\label{sec:sage-algo}

\Cref{algo:sage} introduces
a \gls{SAGE}-based policy-gradient method
based on \Cref{theo:sage}.
For each $m \in \bN$, the \textsc{Gradient}($m$) procedure
is called in the gradient-update step (\Cref{step:sga})
of \Cref{algo:policy-gradient},
at the end of epoch~$m$,
and returns an estimate of $\nabla J(\Theta_m)$
based on batch~$\cD_m$, defined in~\eqref{eq:batch}.
\Cref{algo:sage} can be understood as follows.
According to \Cref{theo:sage}, we have
$\nabla J(\Theta_m)
= \rD \log \rho(\Theta_m)^\intercal \cov{R, x(S)}
+ \esp{R \nabla \log \pi(A | S, \Theta_m)}$
with $(S, A, R) \sim$ \hyperref[eq:stat]{\textsc{stat}($\Theta_m$)}.
\Cref{row:X,row:R,row:C} estimate $\cov{R, x(S)}$
using the usual sample covariance estimator.
\Cref{row:E} estimates $\esp{R \nabla \log \pi(A | S, \theta)}$
using the usual sample mean estimator.
To simplify the signature of \textsc{Gradient}($m$),
we assume that all variables from \Cref{algo:policy-gradient},
in particular batch~$\cD_m$,
are accessible within \Cref{algo:sage}.
The variable $N_m$ computed on \Cref{row:N}
is the batch size, i.e., the number of samples
used to estimate the gradient $\nabla J(\Theta_m)$,
and we assume that it is greater than or equal to~2.
An alternate implementation of the \gls{SAGE}-based policy-gradient method
that allows for batch sizes equal to~1
is given in \Cref{app:sage}.

Recall that our initial goal was to exploit information
on the stationary distribution,
when such information is available.
Consistently, compared to actor--critic (\Cref{app:actor--critic}),
the \gls{SAGE}-based method of \Cref{algo:sage}
requires as input the Jacobian matrix function
\begin{algorithm}[h!]
	\caption{\gls{SAGE}-based policy-gradient method,
		to be called on \Cref{step:sga} of \Cref{algo:policy-gradient}.}
	\label{algo:sage}
	
	\begin{algorithmic}[1]
		\State \textbf{Input:}
		\begin{minipage}[t]{.8\textwidth}
			$\bullet$
			Positive and differentiable policy parametrization
			$(s, \theta, a) \mapsto \pi(a | s, \theta)$\\
			$\bullet$
			Jacobian matrix function $\theta \mapsto \rD \log \rho(\theta)$ \\
			$\bullet$
			Feature function $s \mapsto x(s)$
		\end{minipage}
		\vspace{0.25em}
		
		\Procedure{Gradient}{$m$}
		\State \label{row:N}
		$N_m \gets t_{m+1} - t_m$
		\State \label{row:X}
		$\overline{X}_m \gets \frac1{N_m}
		\sum_{t = t_m}^{t_{m+1} - 1} x(S_t)$
		\State \label{row:R}
		$\overline{R}_m \gets \frac1{N_m}
		\sum_{t = t_m}^{t_{m+1} - 1} R_{t+1}$
		\State \label{row:C}
		$\overline{C}_m \gets \frac1{N_m - 1}
		\sum_{t = t_m}^{t_{m+1} - 1}
		(x(S_t) - \overline{X}_m) (R_{t+1} - \overline{R}_m)$
		\State \label{row:E}
		$\overline{E}_m \gets \frac1{N_m}
		\sum_{t = t_m}^{t_{m+1} - 1}
		R_{t+1} \nabla \log \pi(A_t | S_t, \Theta_m)$
		\State \label{row:return}
		\textbf{return} $\rD \log \rho(\Theta_m) \overline{C}_m + \overline{E}_m$
		\EndProcedure
	\end{algorithmic}
\end{algorithm}%
$\rD \log \rho$
and the sufficient statistics~$x$.
In return,
as we will see in \Cref{sec:convergence,sec:examples},
the \glspl{SAGE}-based method
relies on a lower-dimensional estimator
whenever $\max(n, d) \ll |\cS|$,
which can lead to improved convergence properties.

\section{A Local Convergence Result} \label{sec:convergence}

Our goal in this section is to study the limiting behavior of \Cref{algo:sage}.
To do so, we will consider this algorithm
as an \gls{SGA} algorithm
that uses biased gradient estimates.
The gradient estimates are biased
because they arise from the \gls{MCMC} estimations
from Lines~\ref{row:X}--\ref{row:E}
in \Cref{algo:sage}.
Throughout the proof, we assume that
the reward is a deterministic function $r:\mathcal{S} \times \mathcal{A} \to \R$ for simplicity.
Under this assumption, for each $m \in \bN$,
\Cref{algo:sage} follows the gradient ascent step~\eqref{eqn:update_SGA_intro}, with
\begin{align} \label{eqn:SGA-recursion-for-the-parameter-vector}
	H_m
	&= \rD \log \rho(\Theta_m)^\intercal \overline{C}_m + \overline{E}_m,
	~ \text{where} ~
	\left\{
	\begin{aligned}
		\overline{X}_m
		&= \displaystyle
		\frac{ \sum_{t = t_m}^{t_{m+1} - 1} x(S_t) }{ t_{m+1} - t_m }, \quad
		\overline{R}_m
		= \frac{\sum_{t = t_m}^{t_{m+1} - 1} r(S_t, A_t)}{t_{m+1} - t_m}, \\[.25cm]
		\overline{C}_m
		&= \displaystyle
		\frac{
			\sum_{t = t_{m}}^{t_{m+1} - 1}
			\left( x(S_t) - \overline{X}_m \right)
			\left( r(S_t, A_t) - \overline{R}_m \right)
		}{t_{m+1} - t_m - 1}, \\[.25cm]
		\overline{E}_m
		&= \displaystyle
		\frac
		{\sum_{t = t_m}^{t_{m+1} - 1}
			r(S_t, A_t) \nabla \log \pi(A_t | S_t, \Theta_m)}
		{t_{m+1} - t_m}.
	\end{aligned}
	\right.
\end{align}
The estimates
$\overline{X}_m$, $\overline{R}_m$, and $\overline{C}_m$
are functions of $\cD_m$,
while $H_m$ and $\overline{E}_m$ are functions of~$\cD_m$ and $\Theta_m$.
We will additionally apply decreasing step sizes and increasing batch sizes of the form
\begin{align}
	\alpha_m = \frac{\alpha}{(m+1)^\sigma}
	\quad \text{and} \quad
	t_{m+1} = t_m + \ell m^{\frac\sigma2 + \kappa},
	\quad \text{for each } m \in \bN,
	\label{eqn:step_and_batchsizes}
\end{align}
for some parameters $\alpha \in (0, \infty)$, $\ell \in (1, \infty)$,
$\sigma \in (2/3, 1)$, and $\kappa \in [0, \infty)$.

Our goal---studying the limiting algorithmic behavior of \Cref{algo:sage}---is equivalent to studying the limiting algorithmic behavior
of the stochastic recursion~\eqref{eqn:update_SGA_intro}.
In particular, we will focus on the local convergence of the iterates of \eqref{eqn:update_SGA_intro} and \eqref{eqn:SGA-recursion-for-the-parameter-vector}
to the following set of global maximizers:
\begin{align}
	\label{eq:M}
	\cM = \left\{
	\theta \in \Omega:
	J(\theta) = J^\star
	\right\},
	\quad \text{where }
	J^\star = \sup_{\theta \in \Omega} J(\theta).
\end{align}
We will assume that $\mathcal{M}$ is nonempty, that is, $\mathcal{M} \neq \emptyset$ and at least one $\theta \in \Omega$ satisfies $J(\theta) = J^{\star}$.
Note that Assumption~\ref{ass:nondegenerate-maxima} below
allows $\mathcal{M}$ to be a manifold just locally.
Consequently, $J$ can be nonconvex with noncompact level-subsets, and $J$ is even allowed not to exist outside the local neighborhood, for instance if the policy is unstable. 
If the policy $\pi(\theta)$ is unstable, then necessarily we have $\theta \notin \Omega$,
and we adopt the convention that $J(\theta) = \inf_{s \in S, a \in A} r(s,a) \geq -\infty$
While the previous assumptions allow for general objective functions, the convergence will be guaranteed close to the set of maxima $\mathcal{M}$, or to a set of local maxima that satisfy equivalent assumptions.

\subsection{Assumptions Pertaining to Algorithmic Convergence}
\label{sec:assumptions_pertaining_convergence}

We use the Markov chain of state-action pairs.
Specifically, consider the pairs $\{(S_t, A_t)\}_{t \geq 0} \subset \cS \times \cA$, where $A_t$ is generated according to policy $\pi(\,\cdot\,|S_t, \theta)$.
For a given $\theta \in \Omega$, the one-step transition probability
and the stationary distribution of this Markov chain are
\begin{align}
	P((s^{\prime}, a^{\prime})| (s, a), \theta)
	&= \pi(a^{\prime}|s^{\prime}, \theta)P(s^{\prime}|s,a),
	\quad \text{  for  } (s, a), (s', a') \in \cS \times \cA,
	\label{eqn:mm1_transition_matrix_state_action_chain}
	\\
	\tilde{p}((s,a)|\theta) &= p(s|\theta)\pi(a|s, \theta)
	\quad \text{ for } (s,a) \in \cS\times \cA.
\end{align}
The following are assumed:

\begin{assumption} \label{ass:geometric-ergodicity}
	There exists a function $\cL: \cS \times \mathcal{A} \to [1, \infty)$
	such that, for any $\theta^\star \in \cM$,
	there exist a neighborhood $U$ of $\theta^\star$ in $\Omega$
	and four constants $\lambda \in (0, 1)$, $C > 0$, $b \in \bR_{> 0}$,
	and $v \ge 16$ such that,
	for each $\theta \in U$, the policy $\pi( \theta)$ is such that
	\begin{align*}
		\sum_{(s', a') \in \cS \times \mathcal{A}} P((s', a') | (s,a), \theta) (\cL(s',a^{\prime}))^{v}
		&\le \lambda (\cL (s,a))^v + b,
		&& \text{for each } (s,a) \in \cS \times \mathcal{A},
	\end{align*}
	and, for each $\ell \in \bN_+$ and $(s,a), (s',a') \in \cS\times \cA$,
	\begin{align*}
		\left| P^\ell ((s',a') | (s,a), \theta) - \tilde{p}((s',a') | \theta) \right|
		\le C \lambda^\ell \cL(s,a),
	\end{align*}
	where $P^\ell(\theta)$ is the $\ell$-step transition probability kernel
	of the Markov chain with transition probability kernel \eqref{eqn:mm1_transition_matrix_state_action_chain}.
\end{assumption}

\begin{assumption}\label{ass:regularity-score-function}
	There exists a constant $C > 0$
	such that $| \rD \log \rho(\theta) |_{\textrm{op}} < C$
	for each $\theta \in \Omega$.
\end{assumption}

\begin{assumption} \label{ass:growth_condition}
	Let $\mathcal{L}$ be the Lyapunov function from Assumption~\ref{ass:geometric-ergodicity}. For any $\theta^\star \in \cM$, if $U$ is a local neighborhood satisfying the conditions of Assumption~\ref{ass:geometric-ergodicity}, then there exists a constant $C>0$ such that for any $\theta \in U$ and $(s,a) \in \mathcal{S} \times \mathcal{A}$,
	\begin{equation}
		|x(s)| < C \mathcal{L}(s,a), \quad
		|r(s, a)| < C \mathcal{L}(s,a), \quad
		|r(s, a) \nabla \log \pi(a|s, \theta)| < C\mathcal{L}(s,a).
	\end{equation}
\end{assumption}

\begin{assumption}
	\label{ass:nondegenerate-maxima}
	There exist an integer $\mathfrak{n} \in \{0, 1, \ldots, n-1\}$
	and an open subset $U \subseteq \Omega$
	such that
	(i) $\cM \cap U$ is a nonempty $\fn$-dimensional $C^2$-submanifold of $\bR^n$,
	and (ii) the Hessian of $J$ at $\theta^\star$ has rank $n - \fn$,
	for each $\theta^\star \in \cM \cap U$.
\end{assumption}

These assumptions have the following interpretation.
Assumption~\ref{ass:geometric-ergodicity} formalizes that the Markov chain is geometrically ergodic in a neighborhood around the optimizer, which implies in particular that policies in this neighborhood are stable (i.e., positive recurrent).
Remarkably, it does not require that the chain is geometrically ergodic for all policies, only for those close to an optimal policy when $\theta \in \Omega$.
This stability is guaranteed by a local Lyapunov function $\mathcal{L}$ \emph{uniformly} over some neighborhood close to a maximizer.
In the notation for $b$, and $\lambda \in (0,1)$ of Assumption~\ref{ass:geometric-ergodicity}, if $S_0$ is the initial state, a term that will later bound the size of an `effective' state space in the regret of the algorithm is
\begin{equation}
	\mathcal{L}^{\star} = \max\left(\frac{b}{(1-\lambda)^2},  \max_{a \in \mathcal{A}} \mathcal{L}(S_0, a)^{v} \right).
	\label{eqn:L_star}
\end{equation}

\Cref{ass:regularity-score-function,ass:growth_condition} together guarantee that the estimator $H_m$ concentrates around $\nabla J(\Theta_m)$ at an appropriate rate. Assumption~\ref{ass:regularity-score-function} is easy to verify in our examples since $\rho$ is always positive and bounded.
Assumption~\ref{ass:growth_condition} guarantees that the empirical estimators in~\eqref{eqn:SGA-recursion-for-the-parameter-vector} converge fast enough to their expected values.
In many applications from queueing, Assumption~\ref{ass:growth_condition} holds. Namely, $S$ is usually a normed space and the order of the Lyapunov function $\mathcal{L}(s,a)$ is exponential in the norm of the state $s \in S$, compared to the sufficient statistic $x$ which has an order linear in the norm of $s$.
We remark that, in a setting with a bounded reward function~$r$ and a bounded map~$x$ or with a finite state space, Assumption~\ref{ass:growth_condition} becomes trivial.

Assumption~\ref{ass:nondegenerate-maxima} is a geometric condition. 
It guarantees that, locally around the set of maxima $\mathcal{M}$ or set of local maxima satisfying the same assumptions,
in directions perpendicular to $\cM$, $J$ behaves approximately in a convex manner.
Concretely, this means that $\mathrm{Hess}_{\theta} J$ has strictly negative eigenvalues in the directions normal to $\mathcal{M}$---also referred to as the Hessian being \emph{nondegenerate}. Thus, there is a one--to--one correspondence between local directions around $\theta \in \mathcal{M}$ that decrease $J$ and directions that do not belong to the tangent space of $\mathcal{M}$.
Strictly concave functions satisfy that $\mathfrak{n}= 0$ and Assumption~\ref{ass:nondegenerate-maxima} is thus automatically satisfied in such cases. If $\mathcal{M} \cap U = \{\theta^\star\}$ is a singleton, Assumption~\ref{ass:nondegenerate-maxima} reduces to assuming that $\mathrm{Hess}_{\theta^\star} J$ is negative definite.
Assumption~\ref{ass:nondegenerate-maxima} in a general setting can be difficult to verify, but by adding a regularization term, it can be guaranteed to hold in a broad sense (see \Cref{sec:log_regularization}).

\subsection{Local Convergence Results}

This is our main convergence result for the case that the set of maxima is not necessarily bounded.

\begin{theorem}[Noncompact Case]
	\label{prop:main_prop_convergence in probability_noncompact_case}
	Suppose that Assumptions~\ref{ass:markov} to \ref{ass:nondegenerate-maxima} hold.
	For every maximizer $\theta^\star \in \cM \cap U$,
	there exist constants $c > 0$ and $\alpha_0 > 0$
	such that,
	for each $\alpha \in (0, \alpha_0]$,
	there exists a nonempty neighborhood $V$ of $\theta^\star$ and $\ell_0 \geq 1$
	such that,
	for each $\ell \in [\ell_0, \infty)$,
	$\sigma \in (2/3, 1)$, $\kappa \in [0, \infty)$
	with $\sigma + \kappa > 1$,
	we have, for each $m \in \bN_+$,
	\begin{align}
		\probability{ J(\Theta_m) < J^\star - \epsilon | \Theta_0 \in V }
		&
		\leq
		c
		\Bigl(
		\epsilon^{-2} \mathcal{L}^{\star} m^{-\sigma - \kappa}
		+ \frac{m^{1 - \sigma - \kappa}}{\ell} + \frac{\alpha^2}{\ell}
		+ \alpha m^{-\kappa/2}
		+ \frac{\alpha m^{1 - (\sigma + \kappa)/2}}{\sqrt{\ell}}
		\Bigr)
		,
		\label{eqn:theorem_2_bound}
	\end{align}
	where $(\Theta_m, m \in \bN)$ is a random sequence with $\probability{ \Theta_0 \in V } > 0$,
	and built by recursively applying
	the gradient ascent step~\eqref{eqn:update_SGA_intro}
	with the gradient update~\eqref{eqn:SGA-recursion-for-the-parameter-vector}
	and the step and batch sizes~\eqref{eqn:step_and_batchsizes}
	parameterized by these values of $\alpha$, $\ell$, $\sigma$, and $\kappa$.
\end{theorem}

In Theorem~\ref{prop:main_prop_convergence in probability_noncompact_case}, by setting the parameters $\alpha$, $\ell$, $\sigma$, and $\kappa$ in \eqref{eqn:step_and_batchsizes} appropriately, we can make the probability of $\Theta_m$ being $\epsilon$-suboptimal arbitrarily small. Specifically, the step and batch sizes for each epoch allow us to control the variance of the estimators in \eqref{eqn:SGA-recursion-for-the-parameter-vector}.
This shows that the \gls{SAGE}-based policy-gradient method converges with large probability.
The bound can be understood as follows. The term in \eqref{eqn:theorem_2_bound} on the bound depending on~$\epsilon$ characterizes the convergence rate assuming that all iterates up to time $m$ remain in $V$.
The remaining terms in \eqref{eqn:theorem_2_bound} estimate the probability that the iterates escape the set $V$, which can be made small by tuning parameters that diminish the variance of the estimator $H_m$, such as setting $\kappa$ or $\ell$ large---the batch size becomes larger.

Theorem~\ref{prop:main_prop_convergence in probability_noncompact_case} extends the result of \cite[Theorem~25]{fehrman2020convergence} to a Markovian setting with inability to restart.
In our case, the bias can be controlled by using a longer batch size with exponent at least $\sigma/2$.
Furthermore, we also use the Lyapunov function to keep track of the state of the MDP as we update the parameter in $V$ and ensure stability. 
The proof sketch of Theorem~\ref{prop:main_prop_convergence in probability_noncompact_case} can be found in Section~\ref{sec:Proof-outlines}
and the full proof in Appendix~\ref{sec:appendix_proof_of_theorems}.
In \Cref{sec:compact_case}, we also consider the case that $\mathcal{M} \cap U$ is compact, which can be used to improve Theorem~\ref{prop:main_prop_convergence in probability_noncompact_case}.
Note that the sequence $(\Theta_m, m \in \bN)$ from Theorem~\ref{prop:main_prop_convergence in probability_noncompact_case} is well defined even if unstable policies occur, since the update $H_{m}$ from \eqref{eqn:SGA-recursion-for-the-parameter-vector} is finite. 
In this case, recall that we have the convention that if $\theta$ yields an unstable policy then $J(\theta) = \inf_{s \in S, a\in A} r(s,a) \geq -\infty$.
In Theorem~\ref{prop:main_prop_convergence in probability_noncompact_case}, we can thus assume instead of initializing $\Theta_0 \in V$ that we restrict to trajectories of \gls{SGA} that end up in the neighborhood $V$---the first iterate satisfying this being $\Theta_0$.
In this alternative description, we can assume that the trajectory is $\{ \tilde{\Theta}_{t} \}_{t \in [0, T + t_{0}] }$ for some $t_{0} \in \N$, and $\Theta_0 = \tilde{\Theta}_{t_0}$ reaches $V$.

Theorem~\ref{prop:main_prop_convergence in probability_noncompact_case} also holds for any estimator $\tilde{H}_m$ of the gradient $J(\Theta_m)$ provided that this estimator satisfies appropiate bias and variance bounds typical for estimators using Markov chains (see Lemma~\ref{lem:Bounded_norms_on_the_gradient_estimator} and Proposition~\ref{prop:generalization_policy_gradient} in Section~\ref{sec:Proof-outlines} below). 
Thus, Theorem~\ref{prop:generalization_policy_gradient} and its consequences in the following sections hold for a wide range of policy-gradient methods.
Similarly, Theorem~\ref{prop:main_prop_convergence in probability_noncompact_case} also holds when $\mathcal{M}$ is a manifold of local maxima instead of global maxima.
Indeed, the assumptions are all local and the proof is equivalent.

From Theorem~\ref{prop:main_prop_convergence in probability_noncompact_case}, we immediately obtain a typical sample complexity bound.

\begin{corollary}[Sample Complexity]
	Under the same assumptions and notation as in Theorem~\ref{prop:main_prop_convergence in probability_noncompact_case}, there exists a constant $c>0$ such that for any $1> \epsilon > 0$ and $\delta>0$, if we fix $\ell \geq \alpha^2/(5 \delta c)$ and $\sigma +\kappa> 2$ then, for any $m \in \N$ satisfying 
	\begin{equation}
		m \geq  m(\epsilon, \delta) = c\max \left((\epsilon^2 \delta)^{-\frac{1}{\sigma + \kappa}}, \delta^{-\frac{1}{\sigma +\kappa-1}}, \delta^{-\frac{2}{\kappa}}, \delta^{-\frac{1}{(\sigma +\kappa)/2-1}} \right),
	\end{equation}
	we have
	\begin{equation}
		\probability{ J(\Theta_m) < J^\star - \epsilon | \Theta_0 \in V } < \delta.
	\end{equation}
\end{corollary}

\subsection{Lower Bound}
\label{sec:lower_bound}

As noted in Theorem~\ref{prop:main_prop_convergence in probability_noncompact_case},
the rate in~\eqref{eqn:theorem_2_bound} includes the probability that the iterates escape~$V$, outside which convergence cannot be guaranteed.
Indeed, there is a term $O(\alpha^2/\ell)$ that characterizes the probability that the iterates escape the basin of attraction. 
For general settings, this term cannot be avoided, even in the unbiased case.
In fact, the proposition below shows that for any $\beta > 0$ there are cases where there is a positive lower bound depending on $\alpha^{2+\beta}/\ell$.
In \Cref{prop:lower_bound} below, we consider an \gls{SGA} setting with i.i.d.\ data, where the target is to maximize a function $f$ using estimators $H_m$ for the gradient $\nabla f(\Theta_m)$ at epoch $m$.
In a non-\gls{RL} setting, we usually have $H_m = H_m(\Theta_m, Z_m)$, where $Z_m$ is a collection of i.i.d.\ random variables and $\mathcal{F}_m$ denotes the sigma algebra of the random variables $\Theta_0, \ldots, \Theta_m$ as well as $Z_0, \ldots, Z_{m-1}$.
For this result, we consider an \gls{RL} setting where
the iterates $\Theta_m$ satisfy~\eqref{eqn:update_SGA_intro}, and $\eta_m = H_m - \nabla f(\Theta_m)$ satisfies the following unbiased conditional concentration bounds for some $C>0$:
\begin{equation}
	\E[\eta_m | \mathcal{F}_m] = 0 \ \quad \text{and} \quad
	\E[|\eta_m|^2 | \mathcal{F}_m]| \leq \frac{C}{t_{m+1} - t_m}.
	\label{eqn:lower_bound_lemma4}
\end{equation}
Proposition~\ref{prop:lower_bound} below shows that \Cref{prop:main_prop_convergence in probability_noncompact_case} is almost sharp and characterizes the limitations of using an statistical estimator for the gradient, which can lead to instability. 
As we will see in \Cref{sec:Admission-control-in-an-MM1}, however, there are examples where only local convergence can be expected.  
The proof of Proposition~\ref{prop:lower_bound} can be found in Appendix~\ref{sec:Proof_lower_bound}.

\begin{proposition}
	\label{prop:lower_bound}
	For any $\beta > 0$, there are functions $f \in C^\infty(\R^{n})$ with a maximum $f^{\star} = f(\theta^{\star})$ satisfying \Cref{ass:nondegenerate-maxima},
	such that if the iterates $\Theta_m$ satisfy~\eqref{eqn:update_SGA_intro} and the gradient estimator $H_m = \nabla f(\Theta_m) + \eta_m$ satisfies \eqref{eqn:lower_bound_lemma4},
	there exists a constant $c> 0$ depending on~$f$ and independent of $m$ such that for any $\epsilon \in (0,1)$, $1 > \alpha > 0$, $\delta > 0$, $\ell \geq 1$ and any $\sigma \geq 0, \kappa \geq 0$, in \eqref{eqn:step_and_batchsizes} we have that
	\begin{equation}
		\prb{f(\Theta_m) < f^\star - \epsilon | \Theta_0 \in V}
		\geq
		c \frac{\alpha^{2+\beta}}{\ell}
		\text{ for any } m \geq 1.
	\end{equation}
\end{proposition}

\subsection{Performance Gap and Regret}
\label{sec:regret}

A performance gap bound that we can obtain from Theorem~\ref{prop:main_prop_convergence in probability_noncompact_case} is not fully satisfactory for the epoch number $m$.
Indeed, we can set the batch size very large ($\kappa$ large) since the cost of exploration is not factored in. 
We will therefore obtain a performance gap depending on the number of samples drawn. 
In particular, for a time-step $T \geq 1$, we will define the parameter at this time-step as $\Theta_{m(T)}$ where
\begin{equation}
	m(T) = \min\{n \in \N: \sum_{i=1}^{n} \ell i^{\sigma/2 + \kappa} \geq T\},
	\label{eqn:def_effective_epoch}
\end{equation}
that is, the corresponding epoch of the sample drawn at time $T$. We show the bounds on the performance gap in terms of the total number of samples $T$. The proof of Proposition~\ref{prop:regret} can be found in \Cref{sec:proof_proposition_performance_gap}.
\begin{proposition}[Performance Gap]
	Under the same assumptions and notation as in Theorem~\ref{prop:main_prop_convergence in probability_noncompact_case}, we fix $\alpha$ and $\delta$. 
	Then for any $1 > \zeta > 0$ there is $\kappa(\zeta) \geq 0$, $c>0$ and $\ell_0>0$ such that for any $\ell \geq \ell_0$ and $T \geq 1$ we have the following.\\
	(i) If $\sup_{(s,a)}|r(s,a)| < \infty$, then
	\begin{equation}
		\expectationBig{J^{\star} - J(\Theta_{m(T)}) \big| \Theta_0 \in V} \leq c\Bigl( (\mathcal{L}^{\star})^{\frac{1}{3}} \ell^{\frac{1}{3} + \zeta}T^{-\frac{1}{3}+\zeta} + \ell^{1/2 + \zeta}T^{-\frac{1}{2}+\zeta} + \ell^{2/3 + \zeta}T^{-1+ \zeta}  + \frac{\alpha^2}{\ell} \Bigr).
	\end{equation}
	(ii) If $\sup_{(s,a)}|r(s,a)|$ is unbounded, let $\mathcal{B}_{m(T)} = \{ \Theta_n \in V , n \in [m(T)]\}$ be the event that all iterates up to the epoch of sample $T$ stay in $V$. Then, we have that
	\begin{equation}
		\probability{\mathcal{B}_{m(T)}} \geq 1 - c\Bigl(\frac{\alpha^2}{\ell} + \ell^{1/2 + \zeta}T^{-\frac{1}{2}-\zeta} + \ell^{2/3 + \zeta}T^{-1- \zeta}\Bigr),
	\end{equation}
	and
	\begin{equation}
		\expectationBig{ J^{\star} - J(\Theta_{m(T)}) \big| \mathcal{B}_{m(T)}} \leq c (\mathcal{L}^{\star})^{\frac{1}{3}} \ell^{\frac{1}{3} + \zeta}T^{-\frac{1}{3}+\zeta}.
	\end{equation}
	\label{prop:regret}
\end{proposition}

From Proposition~\ref{prop:regret} we obtain the regret of \gls{SAGE} and, in general, of other policy-gradient algorithms that satisfy typical bounds on bias and variance bounds (see Lemma~\ref{lem:Bounded_norms_on_the_gradient_estimator} and Proposition~\ref{prop:generalization_policy_gradient} in \Cref{sec:Proof-outlines}). The proof of Corollary~\ref{cor:regret} can be found in \Cref{sec:proof_regret_corollary}.
\begin{corollary}[Regret]
	Suppose the assumptions and notation as in Theorem~\ref{prop:main_prop_convergence in probability_noncompact_case} hold. Then for any $1 > \zeta > 0$ there exist $\kappa(\zeta) \geq 0$, $c>0$, $\ell_0$ such that if $\ell \geq \ell_0$, when $\Theta_0$ is the first iterate of \eqref{eqn:update_SGA_intro} in $V$ \\
	(i) If $\sup_{(s,a)}|r(s,a)| < \infty$, then for any $T>1$,
	\begin{equation}
		\expectationBig{T J^{\star} - \sum_{t=1}^{T} r(S_t,A_t) \Big| \Theta_0 \in V} \leq c \Bigl( (\mathcal{L}^{\star})^{\frac{1}{3}} \ell^{\frac{1}{3}+\zeta} T^{\frac{2}{3}+\zeta} + \frac{\alpha^2}{\ell}T \Bigr).
		\label{eqn:regret_1}
	\end{equation}
	(ii)  If $\sup_{(s,a)}|r(s,a)|$ is unbounded, then for any $T \geq 1$
	\begin{equation}
		\expectationBig{ T J^{\star} - \sum_{t=1}^{T} r(S_t,A_t) \Big| \mathcal{B}_{m(T)}} \leq c (\mathcal{L}^{\star})^{\frac{1}{3}} T^{\frac{2}{3}+\zeta}.
		\label{eqn:regret_2}
	\end{equation}
	\label{cor:regret}
\end{corollary}

Besides the term $\alpha^2/\ell$ in Corollary~\ref{cor:regret} that captures the instability due to the estimation error of the gradient, the term $T^{2/3 + \zeta}$ in \eqref{eqn:regret_1} cannot be easily compared with other common regret bounds that assume global features for $J(\theta)$ or the gradient estimator $H_m$. 
However, the coefficient $(\mathcal{L}^{\star})^{\frac{1}{3}}$ plays an analogous role to the size of the state space in regret bounds for finite state space \glspl{MDP}, and directly depends on the Lyapunov function.
In Appendix~\ref{app:mm1}, we find explicitly the value of $\mathcal{L}^\star$ for the single-server queue example of Section~\ref{sec:examples} and see that it behaves as $\mathcal{L}^{\star} \sim O(\mathrm{Vol}(V))$, that is, it encodes the volume of the parameter space where the iterates are confined.

Remarkably, in the expectation of Corollary~\ref{cor:regret}(i) policies that are unstable are not avoided. 
Indeed, note that we only condition on initializing in a stable policy in $V$ but afterwards the trajectory may escape the set $V$ and encounter unstable policies.

In \eqref{eqn:regret_1}, by setting $\ell = T^{1/4}$ we would obtain a horizon-dependent regret of $O(T^{3/4 +\zeta})$, which is sublinear but far from the optimum $T^{1/2}$.
This is most likely due to the decreasing step and increasing batch-sizes. 
While they allow for asymptotic convergence, they are slower to reach a fixed suboptimality gap compared to, e.g., using a constant step and batchsize algorithm.
It may then be possible to use horizon-dependent step, and batch sizes together with a `doubling trick' \citep{besson2018doubling} argument to achieve an anytime optimal suboptimality in the sublinear term.
In this case, however, it is unclear if we would still obtain an equivalent factor $\alpha^2/\ell$ in \eqref{eqn:regret_1} that is optimal as shown with Proposition~\ref{prop:lower_bound}.

\subsection{Local Convergence with Entropy Regularization}
\label{sec:log_regularization}

A well-known phenomenon that can occur when using the softmax policy~\eqref{eqn:defintion_softmax_policy} is that, if the optimal policy is deterministic, the iterates converge to this optimal policy only when $\Theta_m \to \infty$.
Problems where this occurs will thus not satisfy Assumption~\ref{ass:nondegenerate-maxima}: the set of maxima will be empty. 
This phenomenon is illustrated in the example of Section~\ref{sec:Admission-control-in-an-MM1}.
One prevalent method to mitigate the occurrence of maxima at the boundary involves incorporating a regularization term, often linked to relative entropy
$\mathrm{KL}[\tilde{\pi} \parallel \pi]$ of the policy $\pi$ compared to a given $\tilde{\pi}$, defined below in~\eqref{eqn:definition_regularization_term}.

Let $\tilde{\pi}$ be a policy of the same type as those defined in \eqref{eqn:defintion_softmax_policy} and let $\zeta$ be a distribution on $\mathcal{S}$ such that $\zeta(h^{-1}(i)) > 0$ for any $i \in \mathcal{I}$, where $h$ is the index map defined for the class of policies that we use in \eqref{eqn:defintion_softmax_policy}. 
We define the regularization term as
\begin{equation}
	\mathcal{R}_{\tilde{\pi}}(\theta) = \E_{S \sim \zeta}[\mathrm{KL}[\tilde{\pi}(\,\cdot\,| S) \parallel \pi(\,\cdot\,| S, \theta)]] = \sum_{s \in \mathcal{S}} \zeta(s) \E_{A \sim \tilde{\pi}(\,\cdot\,| s)}\left[ \log \left( \frac{\tilde{\pi}(A| s)}{\pi(A|s,\theta)}\right) \right].
	\label{eqn:definition_regularization_term}
\end{equation}
For some $b > 0$ we define
\begin{equation}
	J_{\tilde{\pi}}(\theta) = J(\theta) - b \mathcal{R}_{\tilde{\pi}}(\theta).
	\label{eqn:definition_J_b}
\end{equation}
We can show that adding \eqref{eqn:definition_regularization_term} to $J(\theta)$ defined in \eqref{eq:J}  not only prevents maxima from being at the boundary, but also allows us to avoid using Assumption~\ref{ass:nondegenerate-maxima} altogether. 
The next proposition is proved in Appendix~\ref{sec:Proof_regularization}.

\begin{proposition}
	\label{prop:regularization}
	Assume that we use the softmax policy from \eqref{eqn:defintion_softmax_policy} and let $J(\theta)$ be defined as in \eqref{eq:J}.
	Then for almost every policy $\tilde{\pi}$ in the class of \eqref{eqn:defintion_softmax_policy} with respect to its Lebesgue measure,
	\begin{enumerate}
		\item the function $J_{\tilde{\pi}}(\theta)$ in \eqref{eqn:definition_J_b} satisfies Assumption~\ref{ass:nondegenerate-maxima} and the set of maximizers is bounded, and
		\item Theorem~\ref{prop:main_prop_convergence in probability_noncompact_case} for $J_{\tilde{\pi}}(\theta)$ holds without Assumption~\ref{ass:nondegenerate-maxima}.
	\end{enumerate}
\end{proposition}

By using the regularization in \eqref{eqn:definition_regularization_term} we are changing the original objective.
Nevertheless, we can explicitly bound the difference between $J$ and $J_{\tilde{\pi}}$ at their respective optima.
Let $J^{\star}$ be the optimal value of $J$, and $\theta^{\star}_{r}$ be an optimum of $J_{\tilde{\pi}}$. 
For any $\theta$ we have the inequalities
\begin{equation}
	J_{\tilde{\pi}}(\theta) = J(\theta) - b \mathcal{R}_{\tilde{\pi}}(\theta) \leq J_{\tilde{\pi}}(\theta^{\star}_{r}) \leq J(\theta^{\star}_{r}),
\end{equation}
so that rearranging and letting $\theta$ tend to $\theta^{\star}$ with $J(\theta^{\star}) = J^{\star}$ we have
\begin{equation}
	J^{\star} - J(\theta^{\star}_{r})  \leq b \E_{S \sim \zeta}[\mathrm{KL}[\tilde{\pi}(\,\cdot\,| S) \parallel \pi(\,\cdot\,| S, \theta^{\star})]].
	\label{eqn:bound_change_objective}
\end{equation}

\subsection{Proof Outline for \texorpdfstring{\Cref{prop:main_prop_convergence in probability_noncompact_case}}{Theorem 2}}
\label{sec:Proof-outlines}

We extend the local approach presented in \cite[Section~5]{fehrman2020convergence},
that deals with convergence of SGD where the samples used to estimate the gradient are i.i.d.
We consider instead an \gls{RL} setting where data is Markovian and thus presents a bias.
Fortunately, we can overcome its presence by adding an increasing batch size while tracking the states of the Markov chain via the local Lyapunov function from Assumption~\ref{ass:geometric-ergodicity}, which guarantees a stable \gls{MDP} trajectory as long as the parameter is in a neighborhood close to the maximum. Below we give an outline of the technique employed. For the full proof we refer to \Cref{sec:appendix_proof_of_theorems}.

\subsubsection{Structure of the Proof}

The proof of Theorems~\ref{prop:main_prop_convergence in probability_noncompact_case} consists of several parts.
To show a bound on the probability that $\Theta_m$ is $\epsilon$-suboptimal, we consider the event $\mathcal{B}_m$ that all previous iterates $\Theta_0, \ldots, \Theta_m$ belong to a local neighborhood $V$, and the complementary event $\overline{\mathcal{B}}_m$. We bound these separately.
Firstly, on the event $\mathcal{B}_m$, we show in Lemma~\ref{lem:conditional_convergence_in_basin} that the iterates converge to~$\mathcal{M}$, and we obtain a bound on the $\epsilon$-suboptimal probability for this case.
Secondly, we bound the probability of the complement $\overline{\mathcal{B}}_m$ by using a recursive identity relating the probability of $\overline{\mathcal{B}}_{m-1}$ and the sum of the probabilities of two disjoint events, namely, (i) $\Theta_{m} \notin V$ and the distance of $\Theta_{m}$ to $\mathcal{M}$ is larger than $\delta$, and (ii) $\Theta_{m} \notin V$ and the distance of $\Theta_{m}$ to $\mathcal{M}$ is less than $\delta$.
Intuitively, these events group the cases when $\Theta_{m}$ escapes $V$ in `normal directions' to $\mathcal{M}$ and in `tangent directions' to $\mathcal{M}$, respectively.
We can bound the former by using concentration inequalities, but for the latter we need a maximal excursion bound (Lemma~\ref{lem:maximal_excursion_event_probability} below).
Combining all bounds results in an upper bound on $\mathbb{P}[\mathcal{B}_m]$ (Lemma~\ref{lem:probability_not_leaving_basin}).
The local properties of~$J$ are then used to complete the proof.
Crucially, we use throughout the proof that the local Lyapunov function guarantees stability of the Markov chain and the gradient estimator within $V$, as well as keeps track of the initial state for each epoch. 
Geometric ergodicity also allows us to precisely quantify the bias and variance of the estimators.
Note that if we assumed only stability instead of geometric ergodicity, only asymptotic control of the empirical estimators would be possible.

\subsubsection{Preliminary Step: Definition of the Local Neighborhood and Bound Strategy}

For $\theta^{\star} \in \mathcal{M} \cap U$ and two positive numbers $\mathfrak{r} > 0$, and $\delta > 0$, we now define a neighborhood $V_{\mathfrak{r}, \delta}(\theta^{\star})$ of $\theta^{\star}$ where the algorithm will eventually operate by choosing $\delta$ and $\mathfrak{r}$ appropriately.
Let $
\bar{B}_{\mathfrak{r}}(\theta^{\star})
:= \{ \theta \in \Omega: |\theta - \theta^{\star}| \leq \mathfrak{r} \}
$
denote a closed ball around $\theta^{\star}$ with radius~$\mathfrak{r}$ and $\mathrm{dist}(\theta, L) = \sup_{\theta^\prime \in L} |\theta - \theta^\prime|$ for an open set~$L$. Let $U$ be the neighborhood of $\theta^{\star}$ described in Assumptions~\ref{ass:geometric-ergodicity} and \ref{ass:nondegenerate-maxima}.
We define a tubular neighborhood of $\theta^{\star}$ as follows:
\begin{equation}
	V_{\mathfrak{r}, \delta}(\theta^{\star})
	:=
	\bigl\{
	\theta \in \Omega \cap U
	:
	\mathrm{dist}(\theta, M \cap U)
	=
	\mathrm{dist}(\theta, \bar{B}_{\mathfrak{r}}(\theta^{\star}) \cap M \cap U)
	<
	\delta
	\bigr\}
	.
	\label{eqn:definition_V_set}
\end{equation}
Crucially, Assumption~\ref{ass:nondegenerate-maxima} implies that there exist $\delta_0, \mathfrak{r}_0 >$ such that for any $\delta \in (0,\delta_0]$ and $\mathfrak{r} \in (0, \mathfrak{r}_0]$ an equivalent definition of the set is then
\begin{equation}
	V_{\mathfrak{r},\delta}(\theta^{\star})
	=
	\bigl\{
	y + v
	:
	y \in \bigl( \bar{B}_{\mathfrak{r}}(\theta^{\star}) \cap \mathcal{M} \cap U \bigr)
	\textnormal{ and }
	v \in \bigl( \mathrm{T}_y( \mathcal{M} \cap U )\bigr)^\perp
	\textnormal{ with }
	| v | < \delta, \mathfrak{p}(y + v) = y
	\bigr\}
	.
	\label{eqn:definition_V_set_alternative}
\end{equation}
Here, $\mathfrak{p}$ is the unique local projection onto $\mathcal{M} \cap U$, and $\mathrm{T}_y( \mathcal{M} \cap U )^{\perp}$ denotes the cotangent space of $\mathcal{M} \cap U$ at $y$.
For further details on this geometric statement, we refer to \cite[Proposition~13]{fehrman2020convergence} or \cite[Theorem~6.24]{lee2012smooth}.

In the following, we let $U$ denote the intersection of the neighborhoods from Assumptions~\ref{ass:geometric-ergodicity} and \ref{ass:nondegenerate-maxima},
and $\mathcal{L}$ the Lyapunov function from Assumption~\ref{ass:geometric-ergodicity}.
For any $m \in \naturalNumbersPlus$ define the event and filtration
\begin{align}
	\label{eqn:Definition_of_Bm}
	\mathcal{B}_m
	&:=
	\bigcap_{l=1}^m
	\bigl\{
	\Theta_l
	\in
	V_{\mathfrak{r},\delta}(\theta^{\star})
	\bigr\}
	, \\
	\label{eqn:sigma-algebra}
	\mathcal{F}_m
	&:=
	\sigma
	\Bigl(
	\cD_1 \cup \ldots \cup \cD_{m-1}
	\cup
	\bigl\{
	\Theta_0, \dots, \Theta_m
	\bigr\}
	\Bigr)
	.
\end{align}

Due to the local properties of $J$,
Theorem~\ref{prop:main_prop_convergence in probability_noncompact_case} can be shown by bounding $\probability{ \mathrm{dist}(
	\Theta_m
	,\allowbreak
	\mathcal{M} \cap U
	) \geq \epsilon | \mathcal{B}_0} $.
By separating into the event $\mathcal{B}_{m}$ and its complement, we can show that
\begin{equation}
	\probability{ \mathrm{dist}(
		\Theta_m
		,
		\mathcal{M} \cap U
		) \geq \epsilon | \mathcal{B}_0} \leq  \probability{ \mathrm{dist}(
		\Theta_m
		,
		\mathcal{M} \cap U
		)\indicator{\mathcal{B}_{m-1}} \geq \epsilon} + \probability{\overline{\mathcal{B}_{m}}}.
	\label{eqn:sketch_1}
\end{equation}
The remaining steps of the proof consist of bounding both terms in the right-hand side of \eqref{eqn:sketch_1}.

\subsubsection{Step 1: The Variance of the Gradient Estimator Decreases, in Spite of the Bias}

For each $m \in \naturalNumbersPlus$, let
\begin{equation}
	\eta_{m}
	:=
	H_m
	-
	\nabla J(\Theta_m),
	\label{eqn:plugin_estimator_gradient_difference}
\end{equation}
denote the difference between the gradient estimator $H_m$ in \eqref{eqn:SGA-recursion-for-the-parameter-vector} and the true gradient $\nabla J(\Theta_m)$.
Lemma~\ref{lem:Bounded_norms_on_the_gradient_estimator} below implies that the difference in \eqref{eqn:plugin_estimator_gradient_difference} is, ultimately, small.
From Assumption~\ref{ass:geometric-ergodicity}, since the state-action chain $\{(S_t, A_t)\}_{t \geq 0}$ has a Lyapunov function $\mathcal{L}$, so does the chain $\process{S_t}{t>0}$ with
\begin{equation}
	\mathcal{L}_{v}(s) = \sum_{a \in \mathcal{A}} \mathcal{L}(s,a)^{v} \pi(a|s, \theta),
	\label{eqn:Lyapunov_function_state_chain}
\end{equation}
where $v \geq 16$ is the exponent from Assumption~\ref{ass:geometric-ergodicity}. We can define $\mathcal{L}_{4}(s)$ similarly. The following lemma bounds the variance of $\eta_m$ on the event $\mathcal{B}_m$, which can be controlled with the local Lyapunov function.
The proof of Lemma~\ref{lem:Bounded_norms_on_the_gradient_estimator} is deferred to \Cref{sec:Proof_of_bounded_norms_gradient_estimator}.

\begin{lemma}
	\label{lem:Bounded_norms_on_the_gradient_estimator}
	
	Suppose that Assumptions~\ref{ass:markov}--\ref{ass:nondegenerate-maxima} hold.
	There exists
	a constant $C > 0$ that depends on $\theta^{\star}$, $U$, and $J$
	such that for every $m \in \bN_+$,
	\begin{align}
		\label{eqn:lemma4_i}
		|\expectation{\eta_m \indicator{ \mathcal{B}_{m} } | \mathcal{F}_m
		}|
		&\leq
		\frac{C}{{t_{m+1} - t_m}}\mathcal{L}_{4}(S_{t_m})^{1/2}, \\
		\label{eqn:lemma4_ii}
		\expectation{ |\eta_m|^l \indicator{ \mathcal{B}_{m} }
			|\mathcal{F}_m}
		&	\leq
		\frac{C}{({t_{m+1} - t_m})^{l/2}}\mathcal{L}_{4}(S_{t_m})^{l/2}
		,
		\quad \text{for every } l \in \{1, 2\}.
	\end{align}
\end{lemma}

Lemma~\ref{lem:Bounded_norms_on_the_gradient_estimator} helps to determine the bias incurred when starting at a different state than that of stationarity, and is used to bound the term $\mathrm{dist}(
\Theta_m
,
\mathcal{M} \cap U
)$ from \eqref{eqn:sketch_1} in Lemma~\ref{lem:conditional_convergence_in_basin} below. Note that the definition of \gls{SAGE} and Assumptions~\ref{ass:regularity-score-function} and \ref{ass:growth_condition} are used. 

As a matter of fact, any other estimator~$\tilde{H}_m$ of $\nabla J$ satisfying \eqref{eqn:lemma4_i} and \eqref{eqn:lemma4_ii} from Lemma~\ref{lem:Bounded_norms_on_the_gradient_estimator} will yield similar guarantees. In particular, instead of using the aforementioned assumptions for the proof of Theorem~\ref{prop:main_prop_convergence in probability_noncompact_case} that involve the structure of SAGE or the stationary distribution, we may repeat the proof using instead Lemma~\ref{lem:Bounded_norms_on_the_gradient_estimator}.

\begin{proposition}
	Suppose Assumptions~\ref{ass:markov}, \ref{ass:reward}, \ref{ass:geometric-ergodicity} and \ref{ass:nondegenerate-maxima} hold. Suppose moreover that the estimator of the gradient $H_m$ satisfies Lemma~\ref{lem:Bounded_norms_on_the_gradient_estimator} for each $m \geq 1$. Then, the results of Theorem~\ref{prop:main_prop_convergence in probability_noncompact_case}, \Cref{sec:regret} and \ref{sec:log_regularization} also hold for the step-size from \eqref{eqn:step_and_batchsizes} and policy-gradient update in~\eqref{eqn:update_SGA_intro}.
	\label{prop:generalization_policy_gradient}
\end{proposition}

\subsubsection{Step 2: Convergence on the Event $\mathcal{B}_{m-1}$}

We turn to the first term on the right-hand side of \eqref{eqn:sketch_1} and examine, on the event $\mathcal{B}_{m-1}$, if the iterates converge. 
Using a similar proof strategy as that of \cite[Proposition 20]{fehrman2020convergence} for the unbiased non-Markovian case, we prove in Lemma~\ref{lem:conditional_convergence_in_basin} that the variance of the distance to the set of minima  decreases under the appropriate step and batch sizes. 
Moreover, the rate depends on an effective size of the state space, characterized by the Lyapunov function, and \eqref{eqn:L_star}.

The proof of Lemma~\ref{lem:conditional_convergence_in_basin} is in Appendix~\ref{sec:Proof_of_conditional_convergence_basin}.

\begin{lemma}
	\label{lem:conditional_convergence_in_basin}
	
	Suppose that Assumptions~\ref{ass:markov}--\ref{ass:nondegenerate-maxima} hold.
	There then exist
	$\mathfrak{r}_0$,
	$\alpha_0$,
	$\ell_0 > 0$ and $c > 0$
	such that
	for any
	$\mathfrak{r} \in (0, \mathfrak{r}_0]$,
	$\alpha \in (0, \alpha_0]$ and
	$\ell \in [\ell_0, \infty)$
	there also exists $\delta_0 > 0$ such that for any $\delta \in (0, \delta_0]$ and $m \in \naturalNumbersPlus$,
	\begin{equation}
		\expectationBig{
			(\mathrm{dist}(
			\Theta_m
			,
			\mathcal{M} \cap U
			) \wedge \delta )^2
			\indicator{ \mathcal{B}_{m-1} }
		}
		\leq
		c \mathcal{L}^{*} m^{-\sigma-\kappa}
		.
	\end{equation}
\end{lemma}

Compared to the unbiased case of \cite{fehrman2020convergence}, Lemma~\ref{lem:conditional_convergence_in_basin} needs to use a larger batch size to deal with the bias of Lemma~\ref{lem:Bounded_norms_on_the_gradient_estimator}.
A key result required is that on the event $\mathcal{B}_{m-1}$, the Lyapunov function is bounded in expectation by $\mathcal{L}^*$, which captures the size of the `effective' state space for the policies around an optimum.
With Lemma~\ref{lem:conditional_convergence_in_basin} together with Markov's inequality, a bound of order $\epsilon^{-2} m^{-\sigma - \kappa}$ for the first term in \eqref{eqn:sketch_1} follows.

\subsubsection{Step 3: Excursion and the Probability of Staying in $V_{\mathfrak{r}, \delta}(\theta^{\star})$}

We next focus on $\probability{\overline{\mathcal{B}}_{m}}$.
Since
\begin{equation}
	\probability{\mathcal{B}_{m}} \geq \probability{\mathcal{B}_{m-1}}  - \probability{\Theta_{m} \notin V_{\mathfrak{r}, \delta}(\theta^{\star}), \mathcal{B}_{m-1}},
\end{equation}
we can use a recursive argument to obtain a lower bound, \emph{if} we can bound first the probability
\begin{align}
	\probability{\Theta_{m} \notin V_{\mathfrak{r}, \delta}(\theta^{\star}), \mathcal{B}_{m-1}} &= \probability{\mathrm{dist}(\Theta_{m}, \mathcal{M} \cap U) > \delta, \mathcal{B}_{m-1}} \nonumber \\
	& \phantom{=}+ \probability{\mathrm{dist}(\Theta_{m}, \mathcal{M} \cap U) \leq \delta, \Theta_{m} \notin V_{\mathfrak{r}, \delta}(\theta^{\star}), \mathcal{B}_{m-1}}.
	\label{eqn:sketch_2}
\end{align}
The first term in \eqref{eqn:sketch_2} represents the event that the iterand $\Theta_m$ escapes the set $V_{\mathfrak{r}, \delta}(\theta^{\star})$ in directions `normal' to $\mathcal{M}$, while the second term represents the escape in directions `tangent' to $\mathcal{M}$---intuition derived from the fact that, in that latter event, we still have $\mathrm{dist}(\Theta_{m}, \mathcal{M} \cap U) \leq \delta$.

The first term in \eqref{eqn:sketch_2} can be bounded by using the local geometric properties around minima in the set $U$ and associating the escape probability with the probability that on the event $\mathcal{B}_{m-1}$ escape can only occur if $|\eta_m|$ is large enough. The probability of this last event happening can then be controlled with the variance estimates from Lemma~\ref{lem:Bounded_norms_on_the_gradient_estimator}.

After a recursive argument, we have to consider the second term in \eqref{eqn:sketch_2} for all $l \leq m$. Fortunately, this term can be bounded by first looking at the maximal excursion event for the iterates $\{ \Theta_l \}_{l=1}^m$. The proof can be found in Appendix~\ref{sec:Proof_maximal_excursion_event_probability}. Here, the Lyapunov function again plays a crucial role to control the variance of the gradient estimator on the events $\mathcal{B}_{l}$ for $l \leq m$, compared to an unbiased and non-Markovian case.

\begin{lemma}
	\label{lem:maximal_excursion_event_probability}
	
	Suppose that Assumptions~\ref{ass:markov}--\ref{ass:nondegenerate-maxima} hold.
	Then there exist
	$\mathfrak{r}_0$,
	$\alpha_0$,
	$\ell_0 > 0$, and $c>0$
	such that
	for any
	$\mathfrak{r} \in (0, \mathfrak{r}_0]$,
	$\alpha \in (0, \alpha_0]$ and $\ell \in [\ell_0, \infty)$,
	there exist $\delta_0>0$ such that for any $\delta \in (0, \delta_0]$ and $m \geq 1$,
	\begin{equation}
		\expectationBig{
			\max_{1 \leq l \leq m}
			\bigl|
			\Theta_{l} - \Theta_0
			\bigr|
			\indicator{ \mathcal{B}_{l-1} }
		}
		<
		c \alpha
		\bigl(
		m^{1-3\sigma/2 - \kappa/2}
		+
		\sqrt{\frac{1}{\ell}}m^{1-5\sigma/8 - \kappa/2}
		\bigr)
		.
		\label{eqn:lemma_basin_of_atraction}
	\end{equation}
\end{lemma}

Finally, with the previous steps we obtain a bound on $\probability{\mathcal{B}_m}$ in Lemma~\ref{lem:probability_not_leaving_basin} below\footnote{
	In \eqref{eqn:lemma_basin_of_atraction}, if the Lyapunov function has only smaller moments than order $\nu$, then condition on $\kappa \geq 0$ will become stricter.
	In particular, $\kappa$ tunes the batch size required to sample from the tails of the stationary distribution and may be required to be positive depending the moments of the Lyapunov function.
	The terms $\sigma$ and $\kappa$ can be tuned to control the bias coming from variance and nonstationarity, and finite batch size, respectively.
}. The proof of Lemma~\ref{lem:probability_not_leaving_basin} can be found in Appendix~\ref{sec:Proof_probability_not_leaving_basing}.

\begin{lemma}
	\label{lem:probability_not_leaving_basin}
	Suppose that Assumptions~\ref{ass:markov}--\ref{ass:nondegenerate-maxima} hold and $\sigma + \kappa > 1$.
	There exist
	$\mathfrak{r}_0$,
	$\alpha_0$,
	such that
	for any
	$\mathfrak{r} \in (0, \mathfrak{r}_0]$,
	$\alpha \in (0, \alpha_0]$,
	there also exists a constant $c > 0$, $\delta_0 > 0$ such that for any $\delta \in (0, \delta_0]$, if $\Theta_0 \in V_{\mathfrak{r}/2, \delta}(\theta^{\star})$, there exists $\ell_0 > 0$ such that for any $\ell \in [\ell_0, \infty)$ and $m \in \bN_+$,
	\begin{equation}
		\probability{\mathcal{B}_{m}}
		\geq
		\exp \Bigl( -\frac{c\alpha^2}{\delta^2 \ell} \Bigr) - \frac{c}{\delta^4 \ell}m^{1-\sigma-\kappa} - c\alpha \frac{(m^{1-3\sigma/2 - \kappa/2} + \ell^{-1/2}m^{1-5\sigma/8 - \kappa/2})}{(\mathfrak{r}/2 - 2\delta)_{+}}
		.
	\end{equation}
\end{lemma}

\subsubsection{Step 4: Combining the Bounds in \eqref{eqn:sketch_1}}

The proof of Theorem~\ref{prop:main_prop_convergence in probability_noncompact_case} follows the same steps as are used to prove \cite[Theorem 25]{fehrman2020convergence} by substituting the modified bounds that we have obtained from Lemmas~\ref{lem:conditional_convergence_in_basin} and \ref{lem:probability_not_leaving_basin} in \eqref{eqn:sketch_1}.
The details can be found in Appendix~\ref{sec:Proof_Theorem2}.

\section{Examples and Numerical Results}
\label{sec:examples}

In \Cref{sec:convergence}, we have shown convergence of a \gls{SAGE}-based policy-gradient method
under the assumptions in \Cref{sec:assumptions_pertaining_convergence}.
We now numerically assess\footnote{The code is available at \url{https://gitlab.laas.fr/ccomte/policy-gradient-reinforcement-learning}.} its performance in examples
from stochastic networks and statistical physics
that go beyond these assumptions.
Specifically, we examine a single-server queue with admission control in \Cref{sec:Admission-control-in-an-MM1}, a load-balancing system in \Cref{sec:Load-balancing-system}, and the Ising model with Glauber dynamics in \Cref{sec:ising}.
These examples satisfy the assumptions of \Cref{sec:sage},
required to implement the \gls{SAGE}-based policy-gradient method,
but not necessarily those of \Cref{sec:assumptions_pertaining_convergence}
used to prove convergence;
more discussion about these assumptions
is provided in \Cref{app:examples}.
\Cref{sec:Admission-control-in-an-MM1}
can be seen as a toy warm-up example
whose simple structure allows us
to gain insight into the compared behavior
of \gls{SAGE} and actor-critic.
The larger-scale examples of
\Cref{sec:Load-balancing-system,sec:ising}
show the superiority of \gls{SAGE},
in the sense that the dimension of these examples
makes them out of reach for existing approaches
without the use of function approximation.

The simulation setup is as follows.
Plots are obtained by averaging $10$ independent simulation runs,
each lasting $T_{\max} = 10^6$ time steps.
The initial parameter vector $\Theta_0$
is taken to be the zero vector,
yielding in each example a uniform policy over the action space.
The \gls{SAGE}-based algorithm (\Cref{algo:sage})
is run with batch size~$100$ and step size $\alpha_m = 10^{-1}$.
The actor--critic algorithm (\Cref{app:actor--critic}) is run with batch size~$1$
and step sizes $\alpha_m = 10^{-3}$ and $\alpha_v = \alpha_{\overline{R}} = 10^{-2}$.
It uses a tabular value function
which we initially treat as containing all-zeros,
and which in practice is expanded as states are visited.

\subsection{Admission Control in a Single-Server Queue}
\label{sec:Admission-control-in-an-MM1}

Consider a queueing system where
jobs arrive according to
a Poisson process with rate $\lambda > 0$,
service times are independent
and exponentially distributed
with rate $\mu > 0$,
and the server applies an arbitrary
nonidling nonanticipating scheduling policy
such as first-come-first-served or processor-sharing.
This model is also commonly known as an M/M/1 queue in the literature.
When a job arrives,
the agent decides to either admit or reject it:
in the former case, the job is added to the queue,
otherwise it is lost permanently.
The agent receives a one-time reward $\gamma > 0$ for each admitted job
(incentive to accept jobs)
and incurs a holding cost $\eta >0$ per job per time unit
(incentive to reject jobs).
The goal is to find an admission-control policy
that achieves a trade-off between these two conflicting objectives.

The problem can be related to the framework of \Cref{sec:Problem} as follows.
For $t \in \bN$, let
$S_t$ denote the number of jobs in the system
right before the arrival of the $(t+1)$th job, and let
$A_t$ denote the decision of either admitting or rejecting this job.
We have $\cS = \bN$ and $\cA = \{\accept, \reject\}$.
Also, let $(\Sigma_\tau, \tau \in \bR_{\ge 0})$ denote
the continuous-time process that describes
the evolution of the number of jobs over time
and $(T_t, t \in \bN)$ the sequence of job arrival times,
so that $S_0 = \Sigma_0$ and $S_t = \lim_{\sigma \uparrow T_t} \Sigma_\tau$
for $t \in \bN_+$.
Rewards are given by
\begin{align*}
	R_{t+1}
	&
	=
	\rdisc(S_t, A_t)
	+
	\int_{T_t}^{T_{t+1}}
	r_{\textrm{cont}}(\Sigma_\tau)
	\rd\tau,
\end{align*}
where $\rdisc(s, a) = \gamma \indicator{a = \accept}$
represents the one-time admission reward
and $\rcont(s) = - \eta s$
the holding cost incurred continuously over time. We use this common reward structure in this example, but we remark that arbitrary reward functions $\rdisc$ and $\rcont$ are possible.

For each $k \in \bN$,
we define a random policy parametrization $\pi_k$ with threshold~$k$
and parameter vector\footnote{In this example, vectors and matrices are indexed starting at~0 (instead of~1) for notational convenience.} $\theta = (\theta_0, \theta_1, \ldots, \theta_k) \in \bR^{k+1}$
as follows.
Under policy~$\pi_k$,
an incoming job finding~$s$ jobs in the system
is accepted with probability
\begin{align} \label{eq:mm1-policy}
	\pi_k(\accept | s, \theta)
	&= \frac1{1 + e^{- \theta_{\min(s, k)}}},
	\quad s \in \bN.
\end{align}
Taking $k = 0$ yields a static (i.e., state-independent) random policy,
while letting $k$ tend to infinity yields
a fully state-dependent random policy.
We believe this parametrization makes intuitive sense
because, in a stable queueing system,
small states tend to be visited more frequently than large states.

Under policy parametrization~$\pi_k$,
\Cref{ass:markov,ass:reward,ass:stat} are satisfied
with
$n = d = k + 1$,
$\Omega = \{\theta \in \bR^{k+1}: \pi_k(\accept | s, \theta) < \frac\mu\lambda\}$,
$\Phi(s) = (\frac\lambda\mu)^s$ for each $s \in \cS$,
$x_i(s) = \indicator{s \ge i + 1}$ for each $i \in \{0, 1, \ldots, k-1\}$
and $x_k(s) = \max(s - k, 0)$,
and
$\rho_i(\theta) = \pi_k(\accept | i, \theta)$
for each
$
i
\in
\{0, 1, \ldots, k\}
$.
It follows that
$\nabla \log \rho_i = \nabla \log \pi_k(\accept | i, \cdot)$
for each $i \in \{0, 1, \ldots, k\}$.
We refer to \Cref{app:mm1} for further details.

\subsubsection{Numerical Results in a Stable Queue}

We study the impact of the policy threshold~$k \in \bN$
on the performance of \gls{SAGE} and actor--critic.
The parameters are
$\lambda = 0.7$, $\mu = 1$, $\gamma = 5$, and $\eta = 1$,
and we consider random policies $\pi_k$ with various thresholds.
We have $\Omega = \bR^{k+1}$ because $\lambda < \mu$,
i.e., the queue is always stable.
As we can verify using \Cref{app:mm1},
if $k \le 2$ the best policy is random,
while if $k \ge 3$,
the best policy (deterministically) admits incoming jobs
if and only if there are at most 2 jobs in the system.
Thus, if $k \ge 3$, the best policy is approximated when
$\theta_i \to +\infty$ if $i \in \{0, 1, 2\}$
and $\theta_i \to -\infty$ if $i \in \{3, 4, \ldots, k\}$.
This deterministic policy is optimal among all Markovian policies.
The initial policy is $\pi_k(\accept | s, \Theta_0) = \frac12$ for each $s \in \bN$,
and the system is initially empty, i.e., $S_0 = 0$ with probability~1.

\begin{figure}[ht]
	\centering
	\begin{subfigure}{.49\linewidth}
		\includegraphics[width=\textwidth]{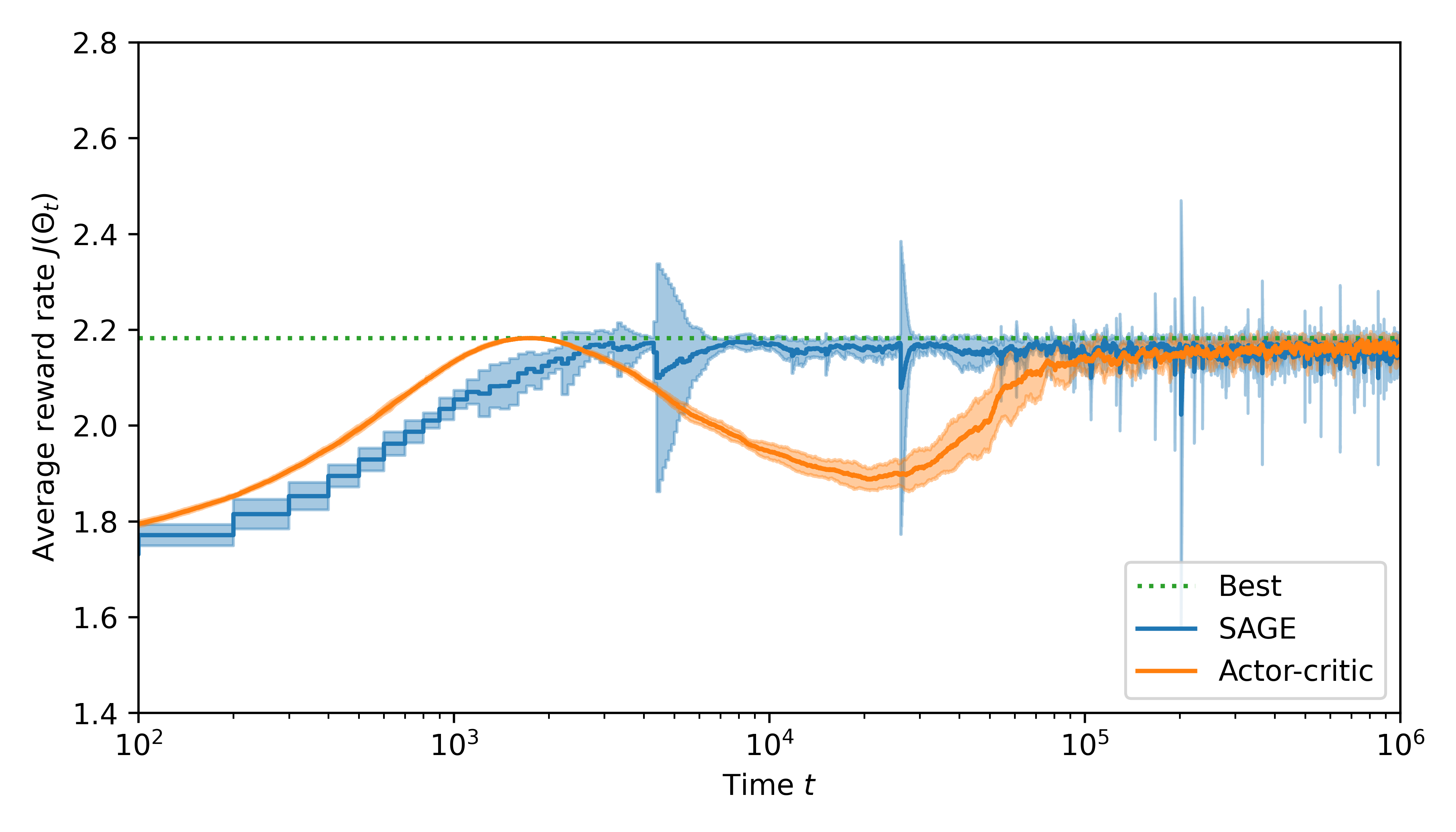}
		\caption{Performance under $\pi_0$.}
		\label{fig:mm1-stable-rwd-0}
	\end{subfigure}
	\begin{subfigure}{.49\linewidth}
		\includegraphics[width=\textwidth]{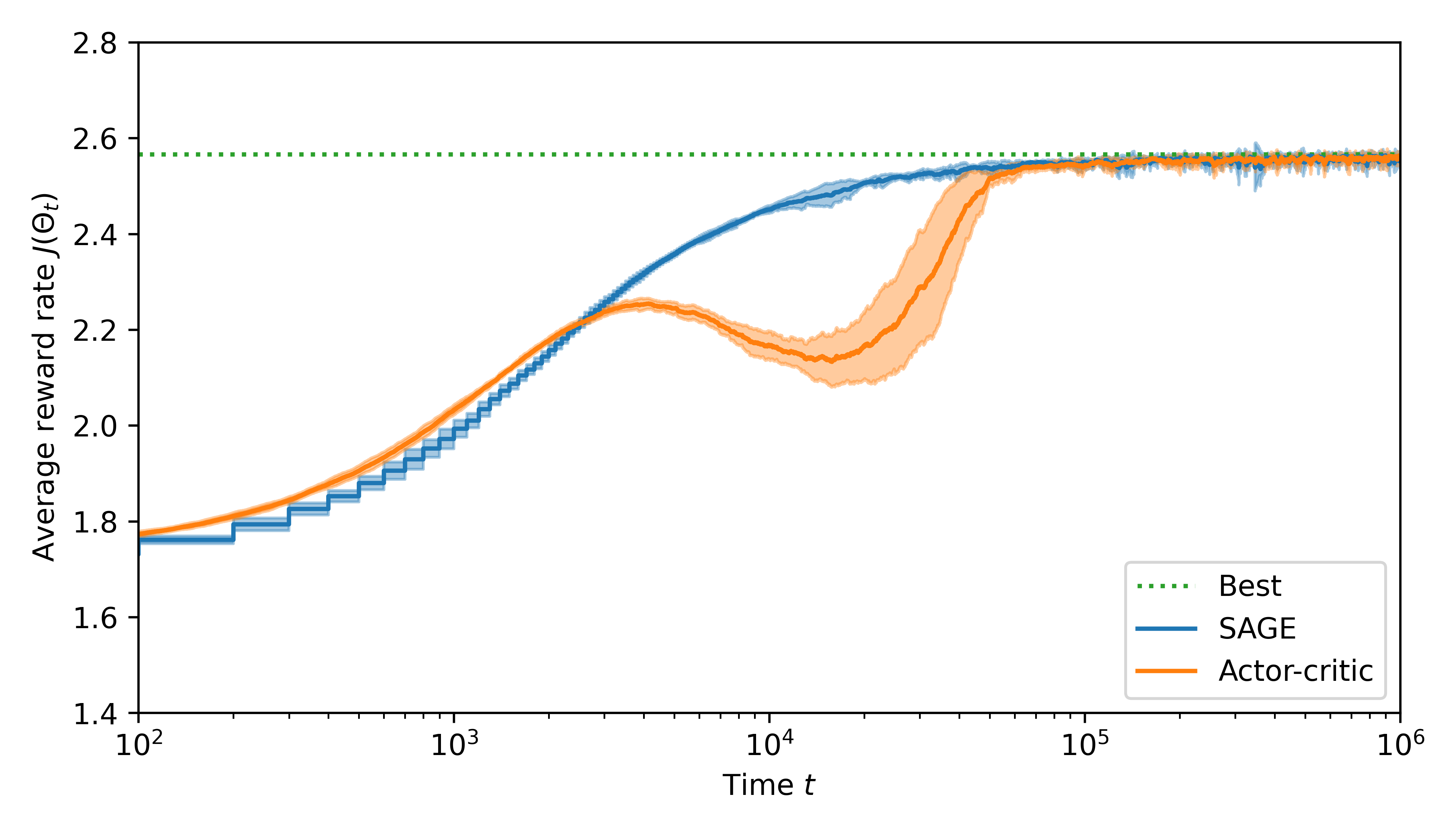}
		\caption{Performance under $\pi_1$.}
		\label{fig:mm1-stable-rwd-1}
	\end{subfigure}
	\begin{subfigure}{.49\linewidth}
		\includegraphics[width=\textwidth]{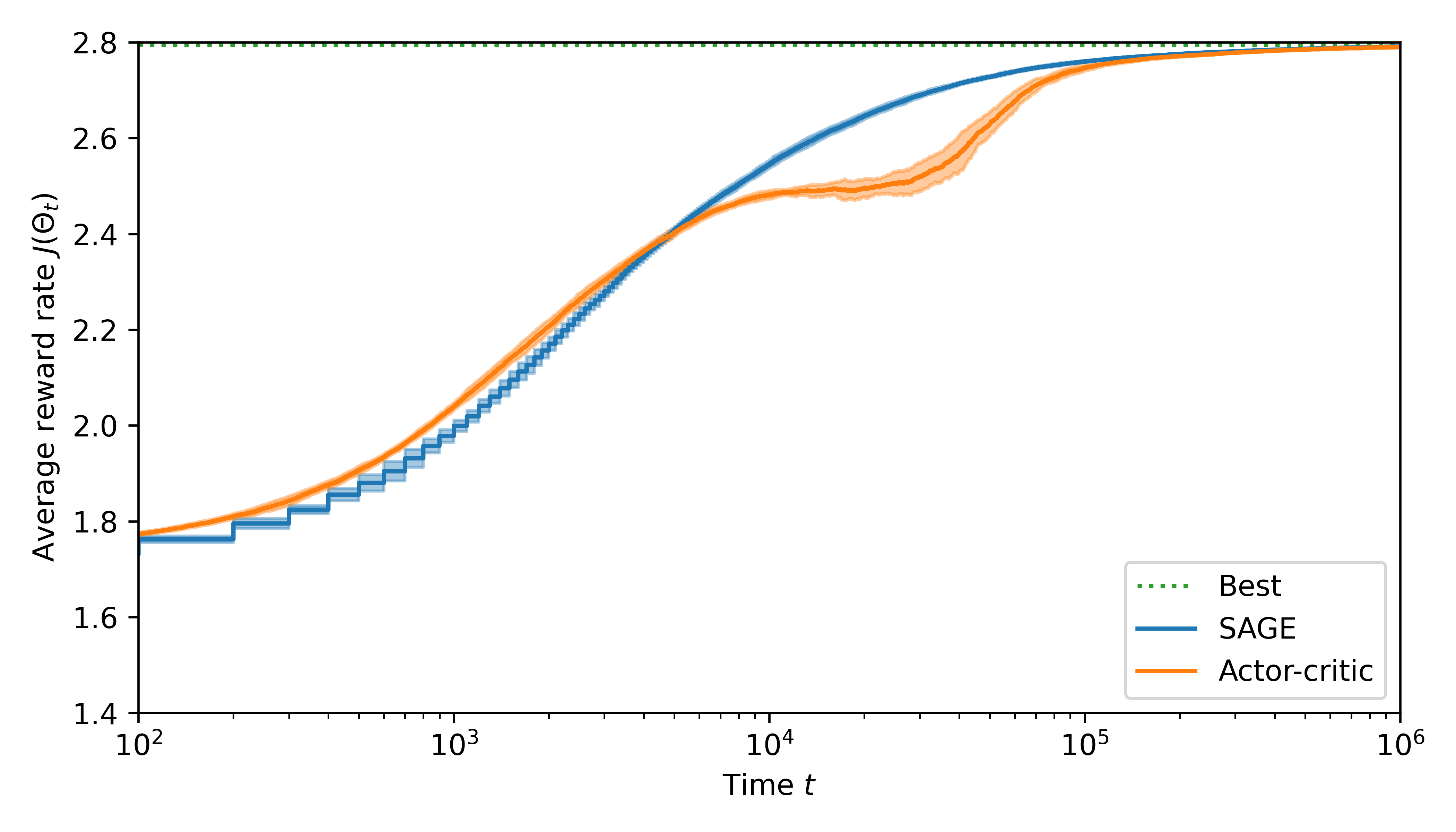}
		\caption{Performance under $\pi_3$.}
		\label{fig:mm1-stable-rwd-3}
	\end{subfigure}
	\begin{subfigure}{.49\linewidth}
		\includegraphics[width=\textwidth]{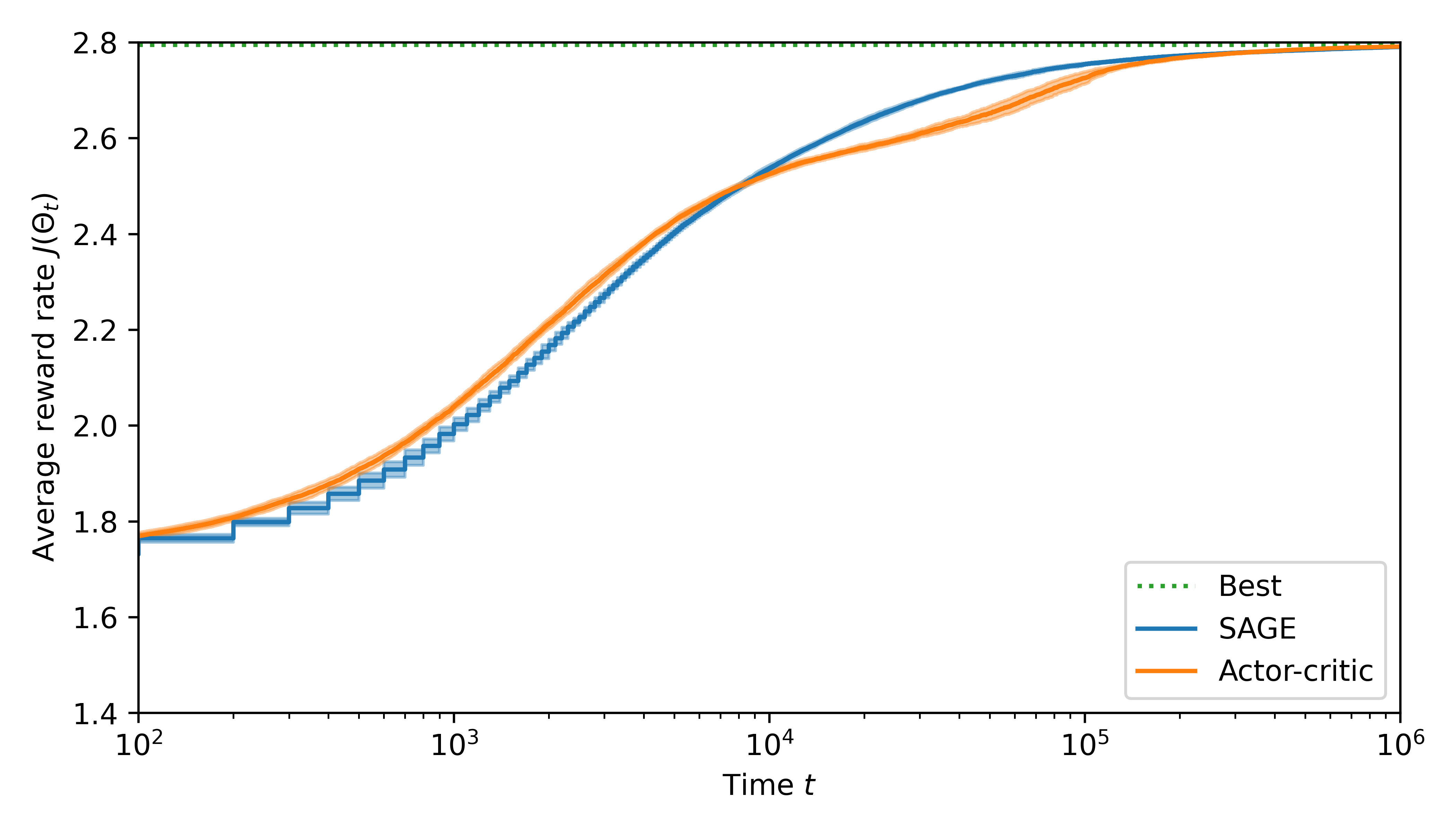}
		\caption{Performance under $\pi_{100}$.}
		\label{fig:mm1-stable-rwd-100}
	\end{subfigure}
	\begin{subfigure}{.49\linewidth}
		\includegraphics[width=\textwidth]{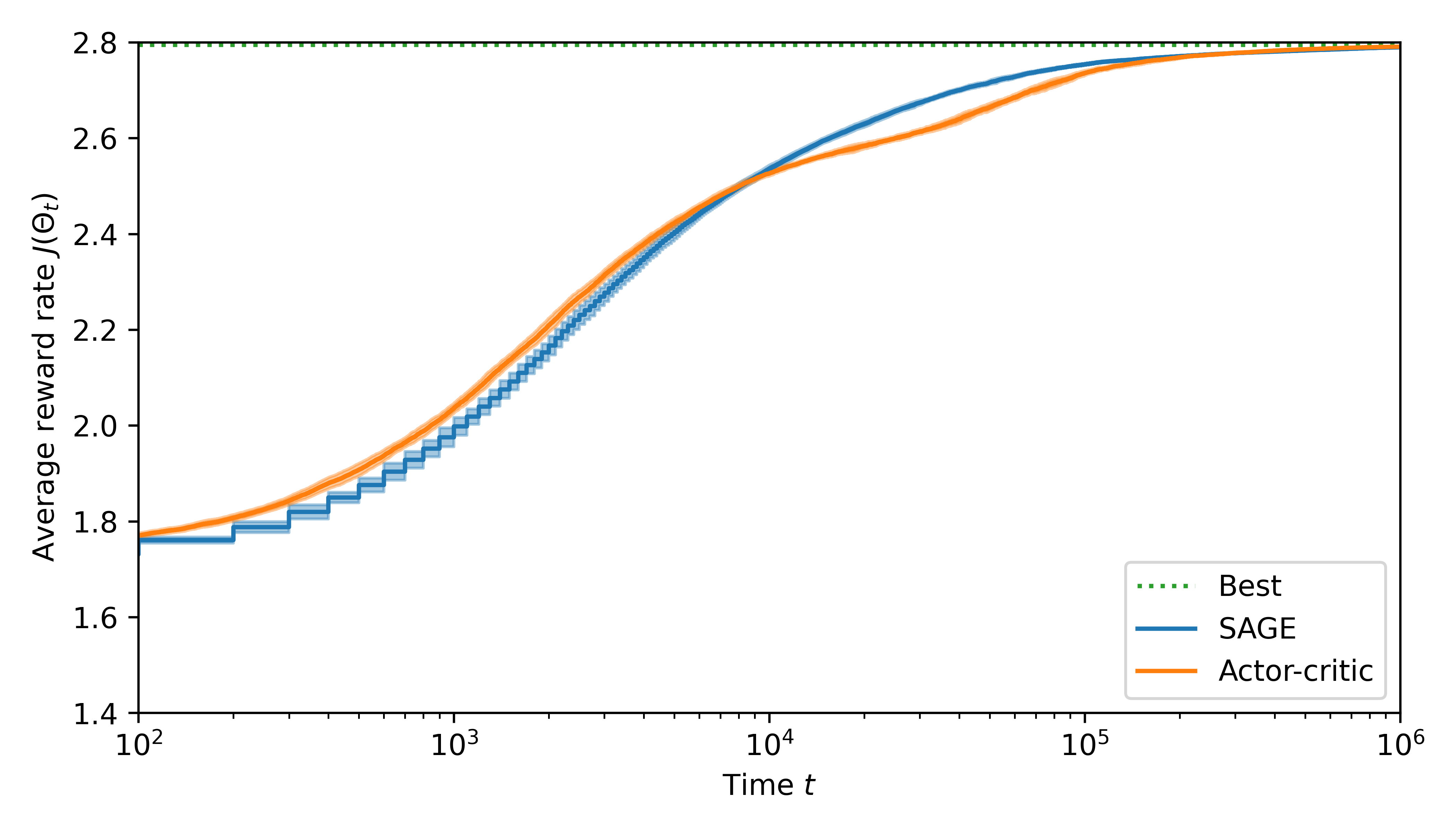}
		\caption{Performance under $\pi_{1000}$.}
		\label{fig:mm1-stable-rwd-1000}
	\end{subfigure}
	\caption{%
		Long-run average reward $J(\Theta_t)$
		in the admission-control problem
		with $\lambda = 0.7$, $\mu = 1$, $\gamma = 5$, and $\eta = 1$.
		Using \Cref{app:mm1}, we can verify that
		the long-run average reward under the best policy
		is approximately
		2.183 if $k = 0$,
		2.566 if $k = 1$,
		and 2.795 if $k \ge 3$.
	}
	\label{fig:mm1-stable-rwd}
\end{figure}

\begin{figure}[ht]
	\centering
	\begin{subfigure}{.49\linewidth}
		\includegraphics[width=\textwidth]{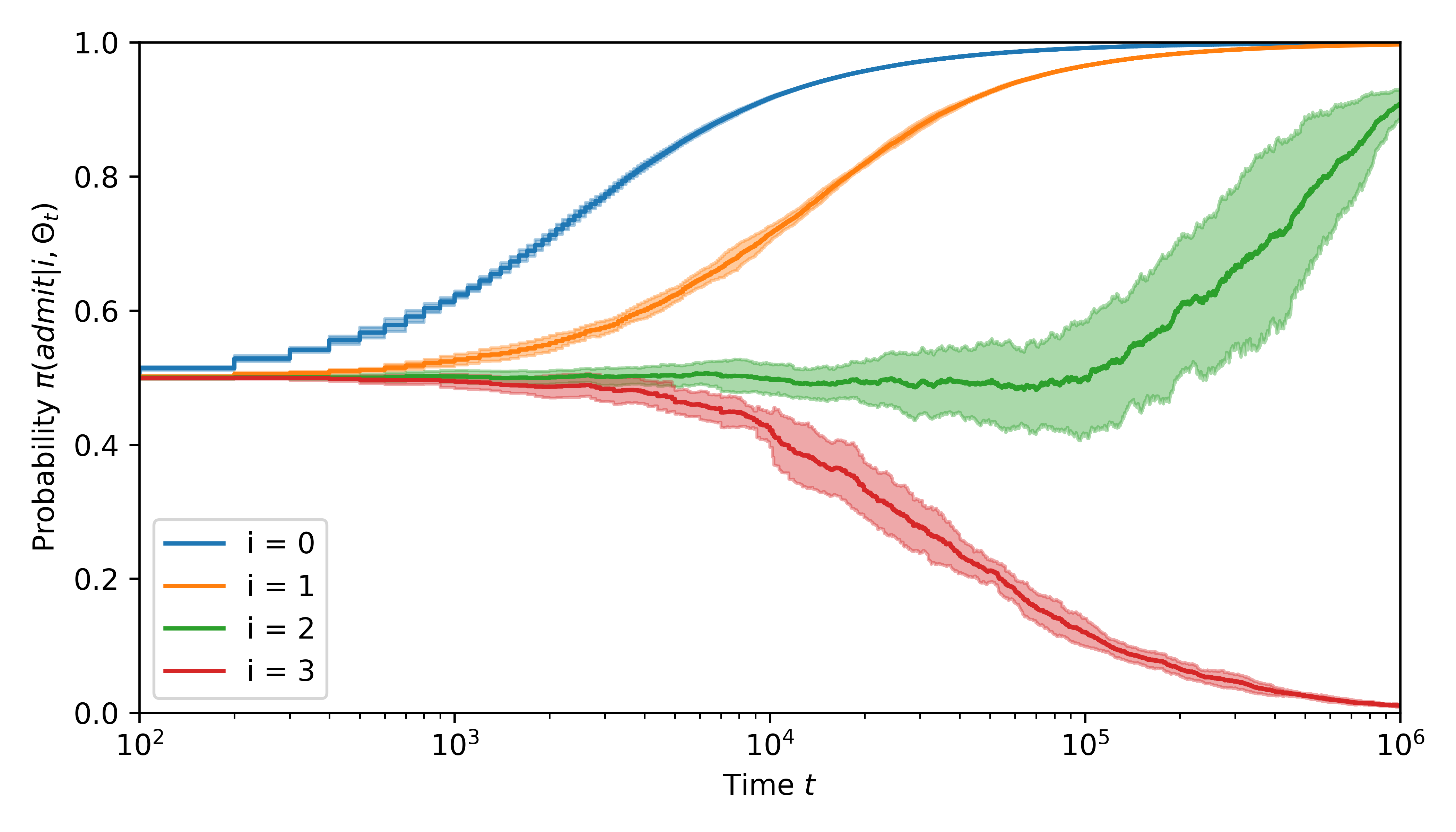}
		\caption{\gls{SAGE}}
		\label{fig:mm1-stable-prb-sage}
	\end{subfigure}
	\begin{subfigure}{.49\linewidth}
		\includegraphics[width=\textwidth]{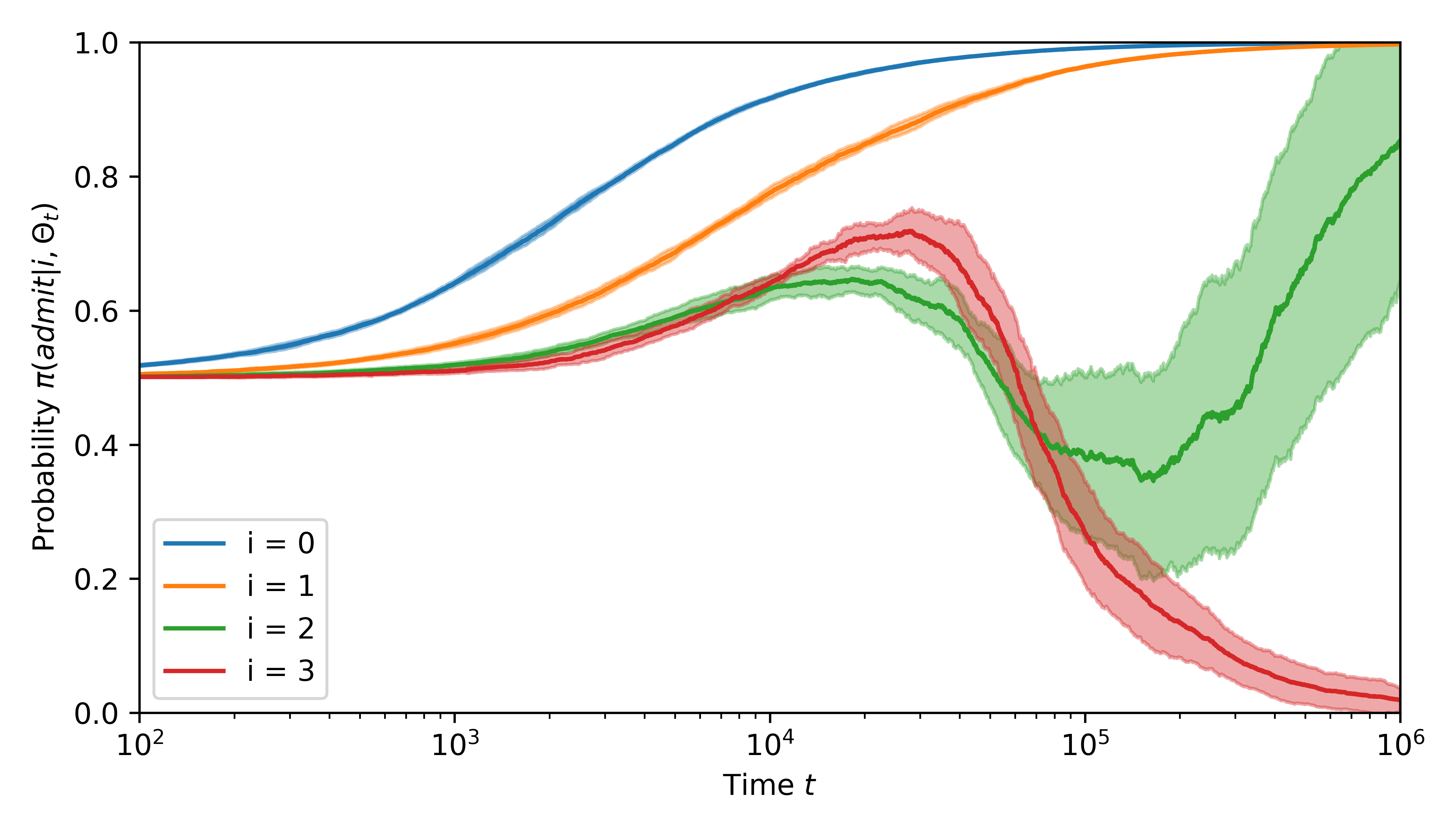}
		\caption{Actor--critic}
		\label{fig:mm1-stable-prb-ac}
	\end{subfigure}
	\caption{Admission probabilities under policy parametrization $\pi_3$.}
	\label{fig:mm1-stable-prb}
\end{figure}

\Cref{fig:mm1-stable-rwd} depicts the impact of the threshold~$k$
on the evolution of the long-run average reward $J(\Theta_t)$
(defined in~\eqref{eq:J} and computed using the formulas of \Cref{app:mm1})
under \gls{SAGE} and actor--critic.
\Cref{fig:mm1-stable-prb}
shows the admission probabilities $\pi_3(\accept | i, \Theta_t)$
for each $i \in \{0, 1, 2, 3\}$
(i.e., the admission probabilities
under the policy with threshold $k = 3$).
In both plots,
the x-axis has a logarithmic scale starting at time $t = 10^2$,
lines are obtained by averaging the results over 10 independent simulations,
and transparent areas show the standard deviation.
Both \gls{SAGE} and actor--critic eventually converge
to the maximal attainable long-run average reward,
and under both algorithms
the convergence is initially faster under policy~$\pi_0$
than under $\pi_1$, $\pi_3$, $\pi_{100}$, and $\pi_{1000}$.
For a particular threshold~$k$,
the convergence is initially faster under actor--critic than under \gls{SAGE}.
However, the long-run average reward under \gls{SAGE} increases monotonically
from its initial value to its maximal value while,
under actor--critic,
there is a time period (comprised between $10^3$ and $10^5$ time steps)
where the long-run average reward stagnates or even decreases.
Similar qualitative remarks can be made
when looking at the running average reward
$\frac1t \sum_{t' = 1}^t R_{t'}$
instead of the long-run average reward $J(\Theta_t)$.
\Cref{fig:mm1-stable-prb-ac} suggests that,
under~$\pi_3$,
this is because actor--critic first ``overshoots''
by increasing $\pi_3(\accept | 3, \Theta_t)$ too much
and then decreasing $\pi_3(\accept | 2, \Theta_t)$ too much
before eventually converging to the best admission probabilities.
This overshooting is more pronounced with a small threshold~$k$,
but it is still visible with~$k = 100$ and $k = 1000$.

\Cref{fig:mm1-stable-rwd,fig:mm1-stable-prb} suggest
actor--critic has more difficulty to correctly estimate the policy update compared to \gls{SAGE},
especially under parametrizations~$\pi_k$
with small thresholds~$k$.
We conjecture this is due to the combination of two phenomena
which reaches a peak when $k$ is small.
First, a close examination of the evolution of the value function
under~$\pi_3$ and $\pi_{10}$ (not shown here) reveals that
there is a transitory bias in the estimate of the value function.
For instance, right after increasing
the admission probability in state~$0$,
the estimate of the value function at states~2 and~3 becomes negative,
even if the optimal value function at these states is positive.
Second, due to the policy parametrization,
parameter $\theta_k$ is updated
whenever a state $s \in \{k, k+1, k+2, \ldots\}$ is visited
(while, for each $i \in \{0, 1, \ldots, k-1\}$,
parameter $\theta_i$ is updated only when state~$i$ is visited).
As a result, the correlated biases in the estimates of
the value function at states $k, k+1, k+2, \ldots$
add up and lead actor--critic to overshoot the update of $\theta_k$,
which has a knock-on effect on other states.

\subsubsection{Numerical Results in a Possibly-Unstable Queue}

\Cref{fig:mm1-unstable-rwd} is the counterpart of \Cref{fig:mm1-stable-rwd}
when the arrival rate is $\lambda = 1.4 > 1 = \mu$.
Now the set of policy parameters
for which the system is stable is
$\Omega = \{\theta \in \bR^{k+1} : \pi_k(\accept | k, \theta) < \frac\mu\lambda\}
\subsetneq \bR^{k+1}$,
with $\frac\mu\lambda \simeq 0.714$.
For simplicity, we will say that
a policy is \emph{stable}
if the Markov chain defined by the system state
under this policy is positive recurrent
(i.e., if $\pi_k(\accept | k, \theta) < \frac\mu\lambda$),
and \emph{unstable} otherwise.
This is an example where convergence can only be guaranteed locally,
as not all policies are stable.
Again using \Cref{app:mm1},
we can verify that if $k \le 1$, the best policy is random,
while if $k \ge 2$,
the best policy (deterministically)
admits incoming jobs
if and only if
there are fewer than 2 jobs in the system.
This deterministic policy is optimal among all Markovian policies.
The initial policy is again the (stable) uniform policy,
and the system is initially empty, i.e., $S_0 = 0$ with probability~1.

\begin{figure}[ht]
	\centering
	\begin{subfigure}{.49\linewidth}
		\includegraphics[width=\textwidth]{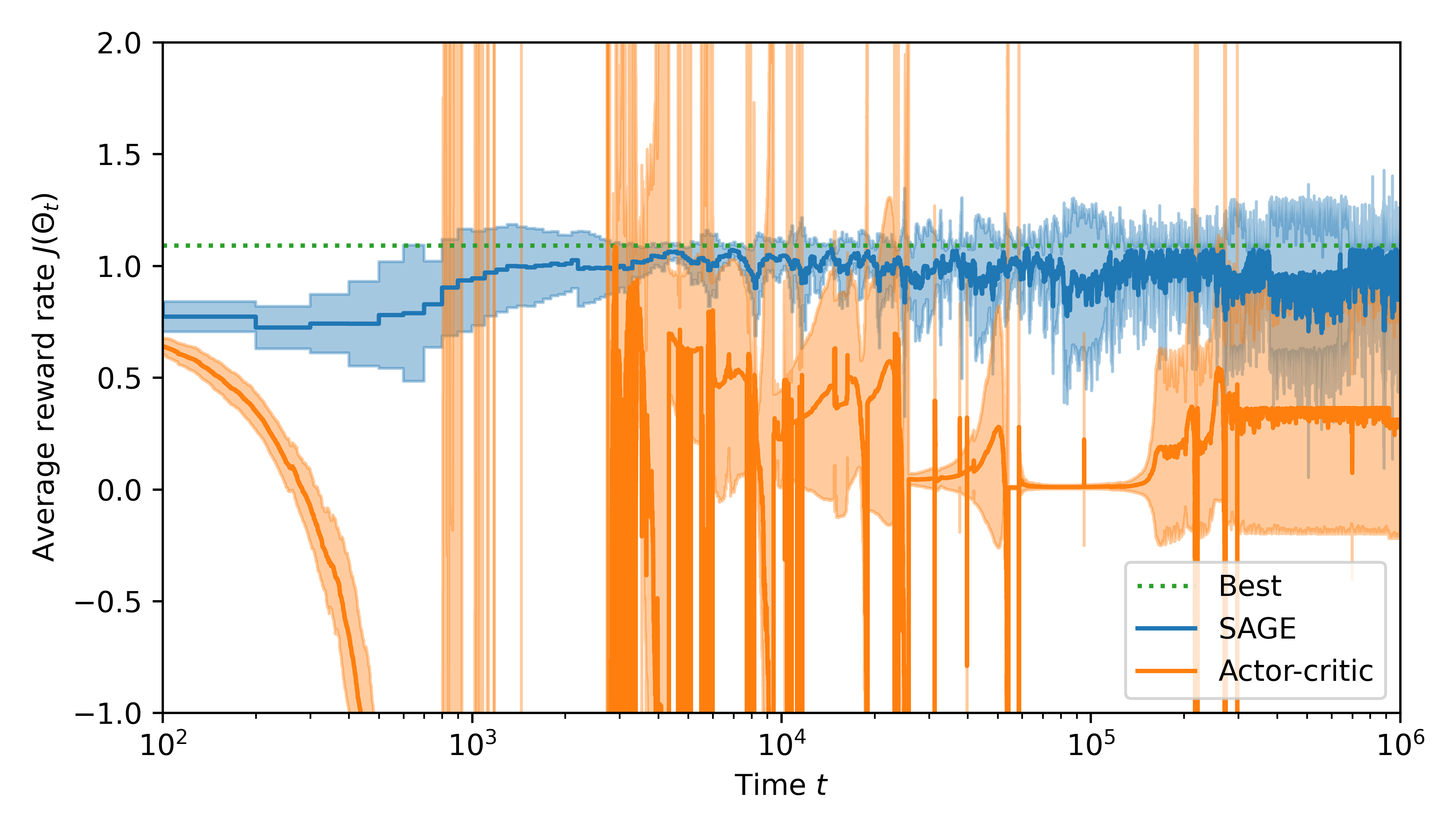}
		\caption{Performance under~$\pi_0$.}
		\label{fig:mm1-unstable-rwd-0}
	\end{subfigure}
	\begin{subfigure}{.49\linewidth}
		\includegraphics[width=\textwidth]{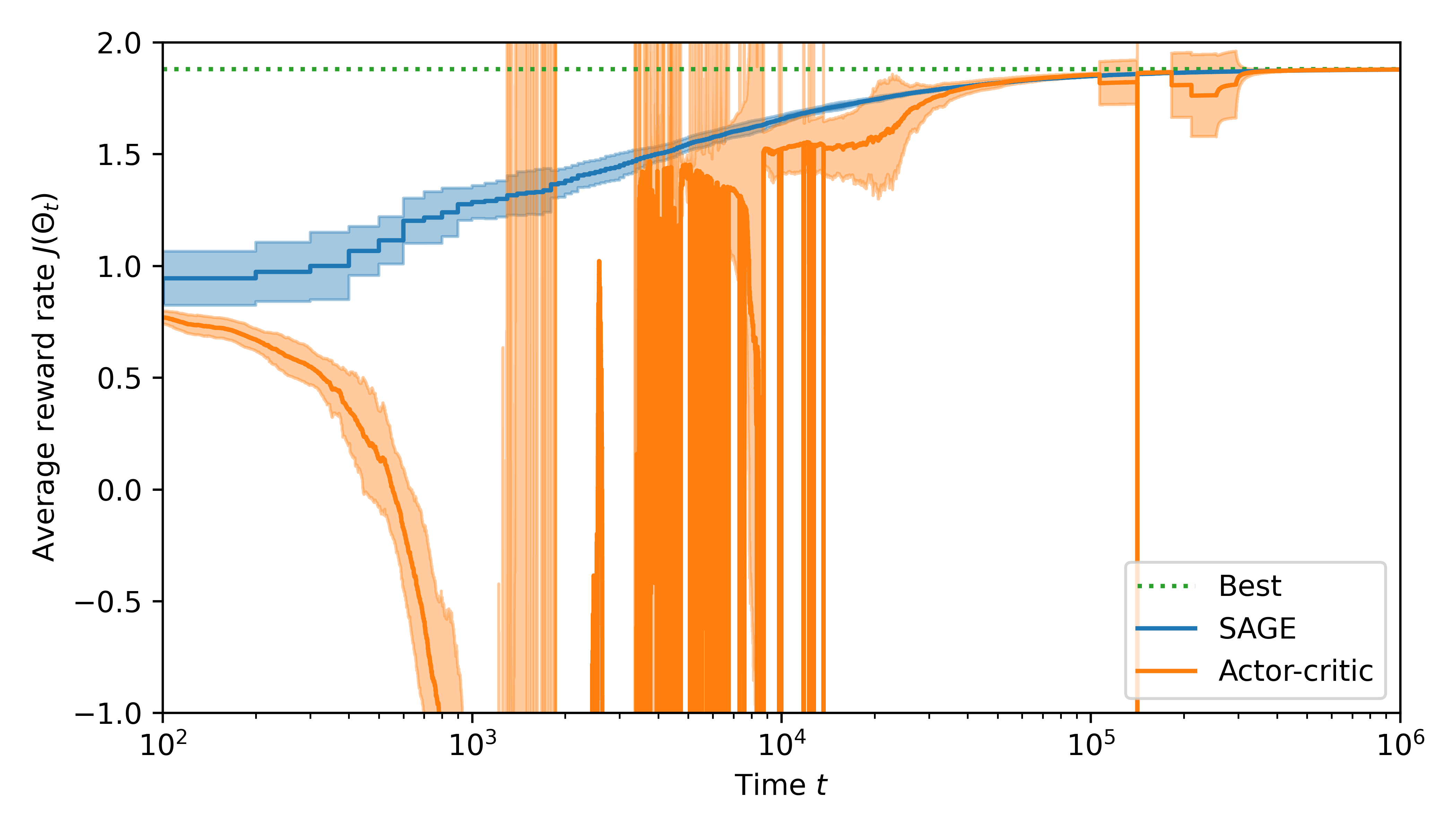}
		\caption{Performance under~$\pi_2$.}
		\label{fig:mm1-unstable-rwd-2}
	\end{subfigure}
	\begin{subfigure}{.49\linewidth}
		\includegraphics[width=\textwidth]{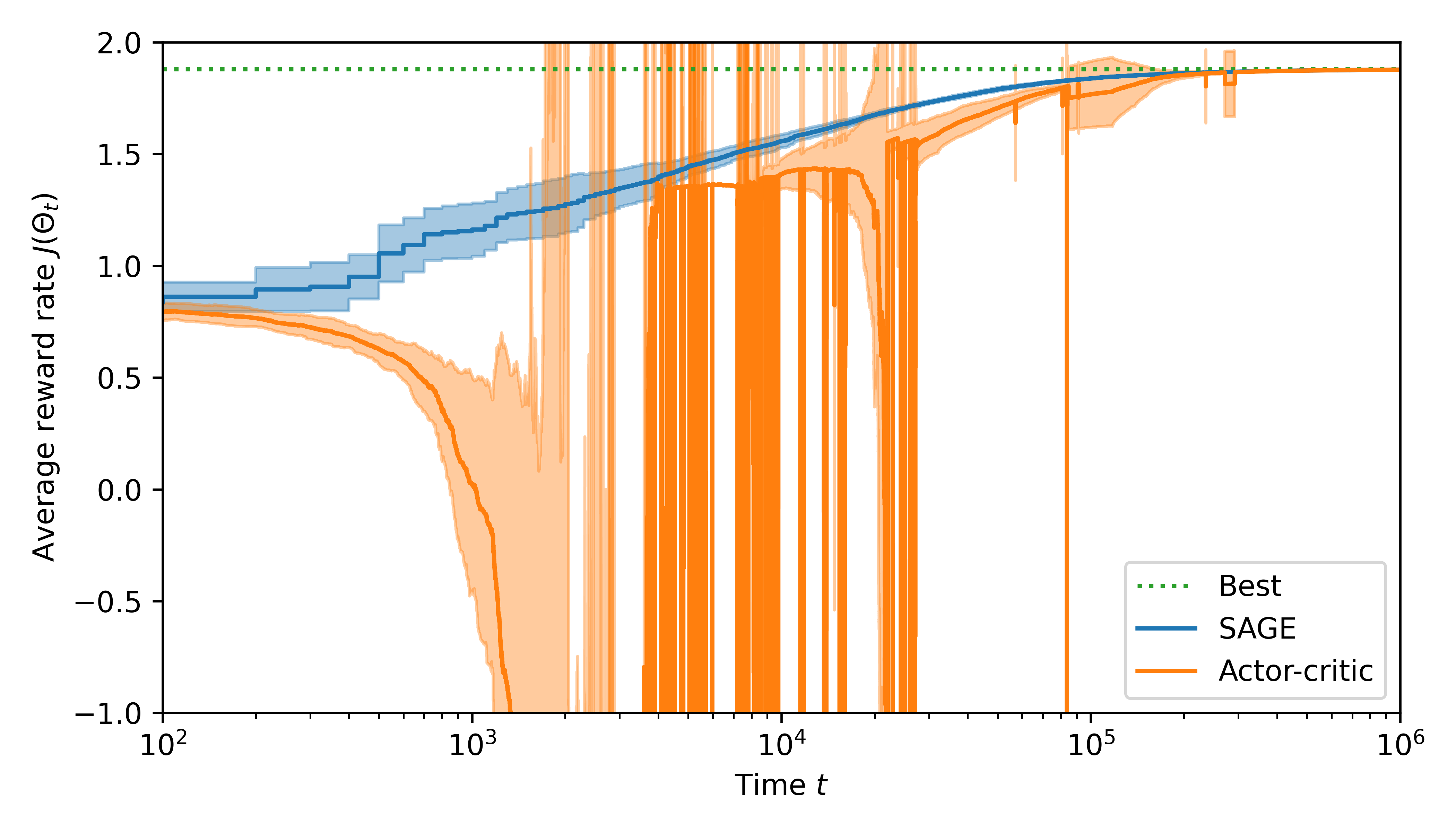}
		\caption{Performance under~$\pi_4$.}
		\label{fig:mm1-unstable-rwd-4}
	\end{subfigure}
	\begin{subfigure}{.49\linewidth}
		\includegraphics[width=\textwidth]{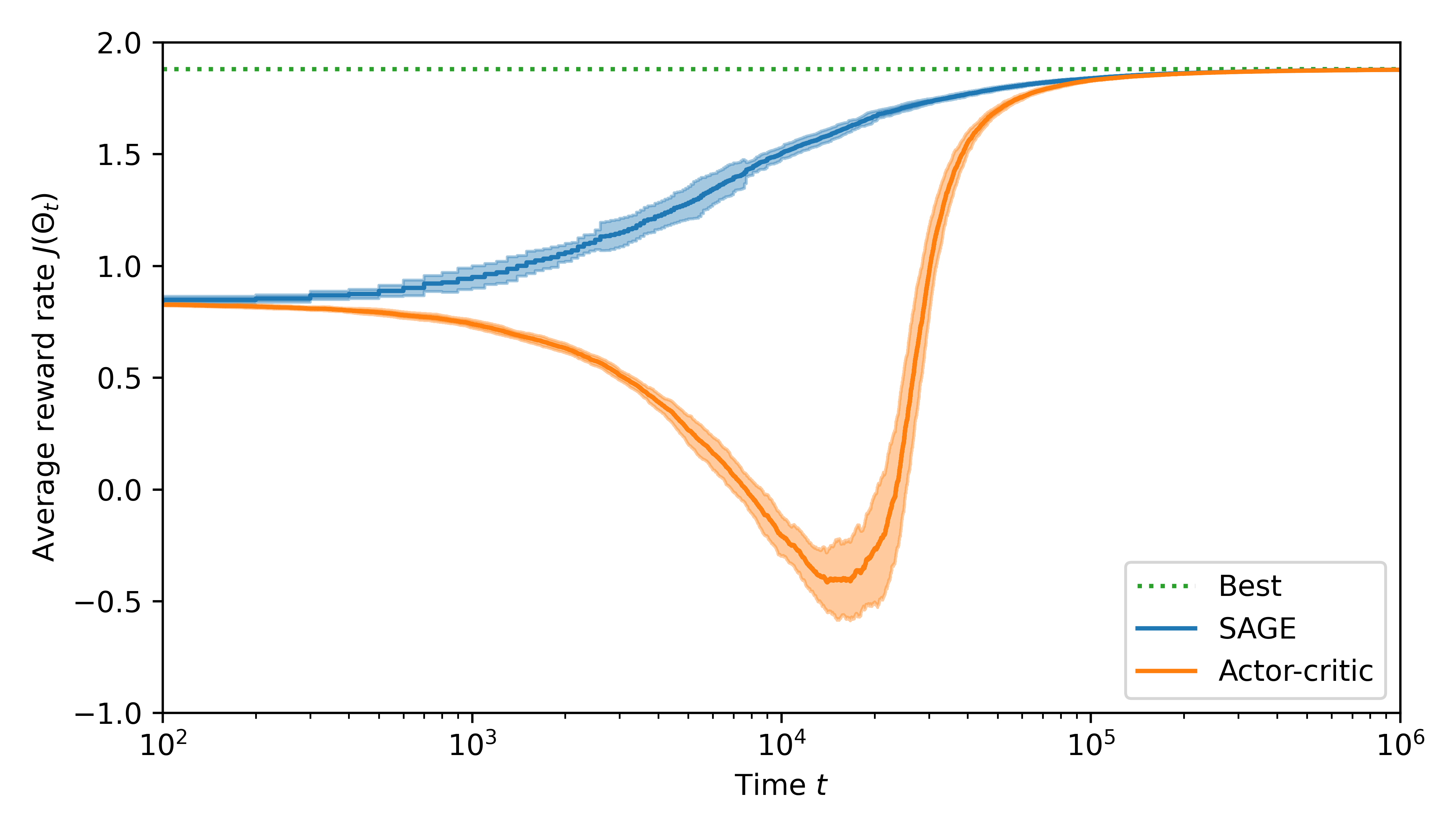}
		\caption{Performance under~$\pi_{100}$.}
		\label{fig:mm1-unstable-rwd-100}
	\end{subfigure}
	\begin{subfigure}{.49\linewidth}
		\includegraphics[width=\textwidth]{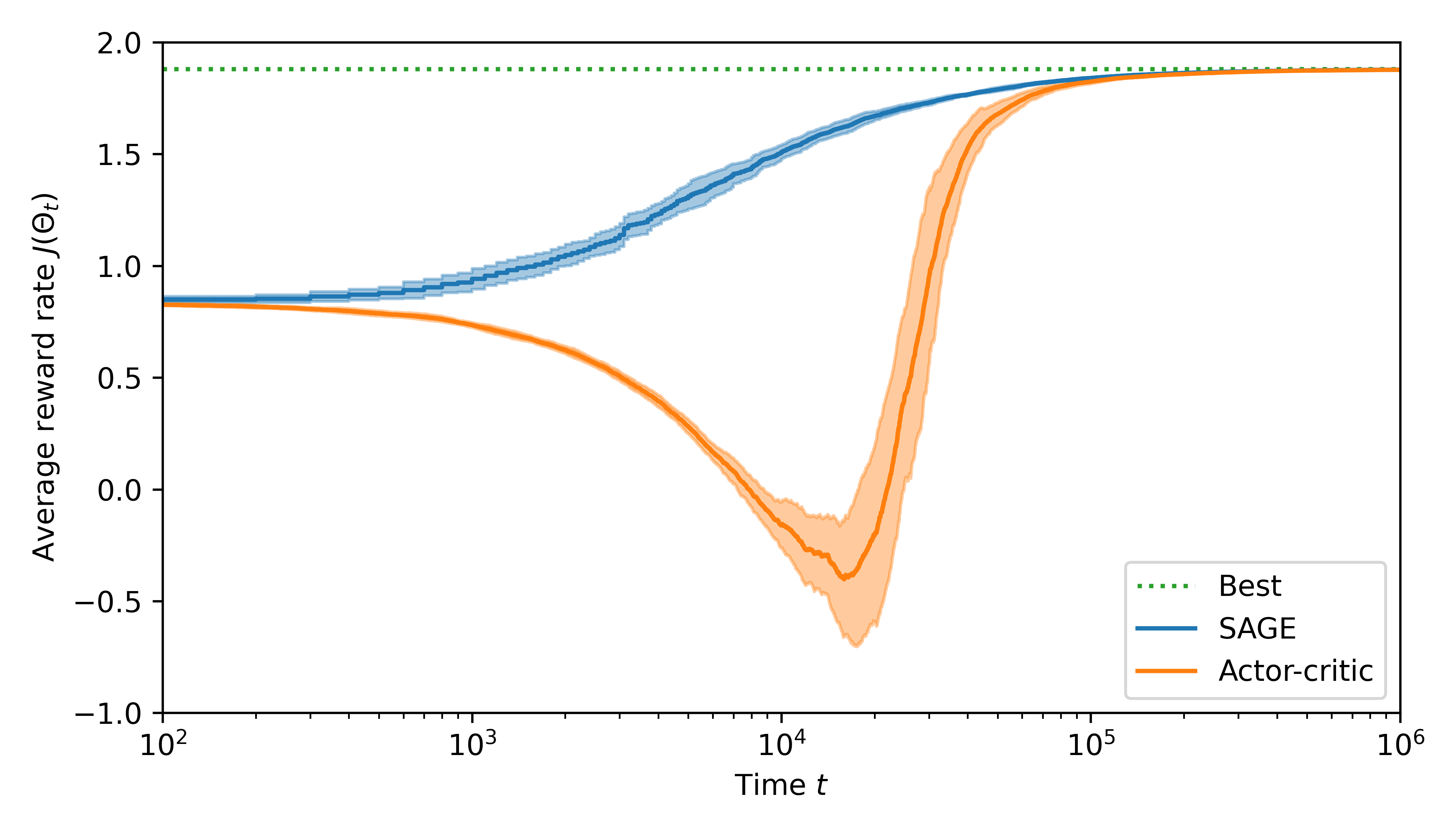}
		\caption{Performance under~$\pi_{1000}$.}
		\label{fig:mm1-unstable-rwd-1000}
	\end{subfigure}
	\caption{%
		Long-run average reward in the admission-control problem
		with parameters $\lambda = 1.4$, $\mu = 1$, $\gamma = 5$, and $\eta = 1$.
		Using \Cref{app:mm1}, we can verify that
		the maximal value of the long-run average reward is approximately
		1.091 if $k = 0$
		and 1.880 if $k \ge 2$.
	}
	\label{fig:mm1-unstable-rwd}
\end{figure}

The first take-away of \Cref{fig:mm1-unstable-rwd} is that
\gls{SAGE} converges to a close--to--optimal policy
despite the fact that some policies are unstable.
The convergence of \gls{SAGE} is actually faster
under $\lambda = 1.4$ compared to $\lambda = 0.7$ (\Cref{fig:mm1-stable-rwd}).
By looking at the evolution of the admission probability (not shown here),
we conjecture this is due to the fact that
the admission probability in states larger than or equal to~2
decreases much faster when $\lambda = 1.4$ compared to $\lambda = 0.7$,
and that this probability has a significant impact
on the long-run average reward.
In none of the simulations does \gls{SAGE} reach an unstable policy. This suggests that the updates of \gls{SAGE} have lower chance of reaching unstable regions of the policy space per observed sample.

The second take-away of \Cref{fig:mm1-unstable-rwd} is that,
on the contrary,
actor--critic has difficulties
coping with instability in this example.
In all simulation runs used to plot this figure,
the long-run average reward $J(\Theta_t)$ first decreases
before possibly increasing again
and converging to the best achievable long-run average reward.
Under parametrizations $\pi_0$, $\pi_2$, and $\pi_4$,
unstable policies are visited for thousands of steps
in all simulation runs,
and a stable policy is eventually reached
in only 7 out of 10 runs.
Under parametrization $\pi_0$,
the long-run average reward under the last policy is close to the best
only in 2 out of 10 runs.
Under $\pi_{100}$ and $\pi_{1000}$,
the policy remains stable throughout all runs,
but the long-run average reward
transitorily decreases before increasing again.

\subsection{Load-Balancing System}
\label{sec:Load-balancing-system}

Consider a cluster of $n$ servers.
Jobs arrive according to a Poisson process
with rate $\lambda > 0$,
and a new job is admitted
if and only if there are fewer
than $c \in \bN_+$ jobs in the system.
Each server~$i \in \{1, 2, \ldots, n\}$
processes jobs in its queue
according to a nonidling, nonanticipating policy.
The service time of each job at server~$i$
is exponentially distributed with rate $\mu_i > 0$,
independently of all other random variables.
The agent aims to maximize the admission probability
by adequately distributing load across servers.

For each $t \in \bN$,
let $S_t = (S_{t, 1}, S_{t, 2}, \ldots, S_{t, n})$
denote the vector containing
the number of jobs at each server
right before the arrival of the $(t+1)$th job,
and let $A_t \in \{1, 2, \ldots, n\}$ denote
the server to which this $(t+1)$th job is assigned.
(This decision is void if $S_{t, 1} + \ldots + S_{t, n} = c$
because the job is rejected anyway.)
We have $\cS = \{s \in \bN^n: s_1 + s_2 + \ldots + s_n \le c\}$
and $\cA = \{1, 2, \ldots, n\}$.
The agent obtains a reward of~$1$ if the job is accepted and $0$ otherwise,
that is, $R_{t+1} = \indicator{S_{t, 1} + \ldots + S_{t, n} \le c - 1}$ for each $t \in \bN$.

We consider the following static policy parametrization,
with parameter vector
$\theta \in \bR^n$:
irrespective of the system state $s \in \cS$,
an incoming job is assigned
to server~$i$ with probability
\begin{align}
	\pi(i | s, \theta)
	&= \pi(i | \theta)
	= \frac{e^{\theta_i}}{\sum_{j = 1}^n e^{\theta_j}},
	\quad i \in \{1, 2, \ldots, n\}
	.
	\label{eq:lb-policy}
\end{align}
\Cref{ass:markov,ass:reward,ass:stat} are satisfied
with
$n = d$,
$\Omega = \bR^n$,
$\Phi(s) = \prod_{i = 1}^n (\frac\lambda{\mu_i})^{s_i}$ for each $s \in \cS$,
$x_i(s) = s_i$ for $i \in \{1, 2, \ldots, n\}$ and $s \in \cS$,
and $\rho_i(\theta) = \pi(i | \theta)$
for $i \in \{1, 2, \ldots, n\}$ and $\theta \in \bR^n$.
Also note that $\nabla \log \rho_i(\theta) = \nabla \log \pi(i | \theta)$
for $i \in \{1, 2, \ldots, d\}$ and $\theta \in \bR^n$.
Except for \Cref{ass:nondegenerate-maxima}, the remaining assumptions outlined in \Cref{sec:convergence} are also satisfied.
We refer to \Cref{app:lb} for more details.
Lastly observe that,
in spite of the policy being static and the state space being finite,
the function~$J$ is still nonconvex for typical system parameters.
In fact, our numerical experiments are done in nonconvex scenarios.
Furthermore, note that this system can become challenging to optimize if $c$ and $n$ are large.

\subsubsection{Numerical Results}
We study the performance of \gls{SAGE} and actor--critic
under varying numbers of servers and service speed imbalance.
Given an integer $n \in \bN_{> 0}$ multiple of~4
and $\delta > 1$,
we consider the following cluster of $n$ servers
divided into 4 pools.
For each $k \in \{1, 2, 3, 4\}$,
pool~$k$ consists of the $\frac{n}4$ servers
indexed from $(k-1) \frac{n}4 + 1$ to $k \frac{n}4$,
and each server~$i$ in this pool has service rate~$\mu_i = \delta^{k-1}$.
The total arrival rate is $\lambda = 0.7 (\sum_{i = 1}^n \mu_i)$
and the upper bound on the number of jobs in the system is $c = 10\frac{n}4$.
Letting $\delta = 1$ gives a system where all servers have the same service speed,
while increasing $\delta$ makes the server speeds more and more imbalanced.
The initial policy is uniform,
i.e., $\pi(i | \Theta_0) = \frac1n$
for each $i \in \{1, 2, \ldots, n\}$,
and the initial state is empty,
i.e., $S_0 = 0$ with probability~1.

\begin{figure}[t]
	\centering
	\begin{subfigure}{.49\linewidth}
		\includegraphics[width=\textwidth]{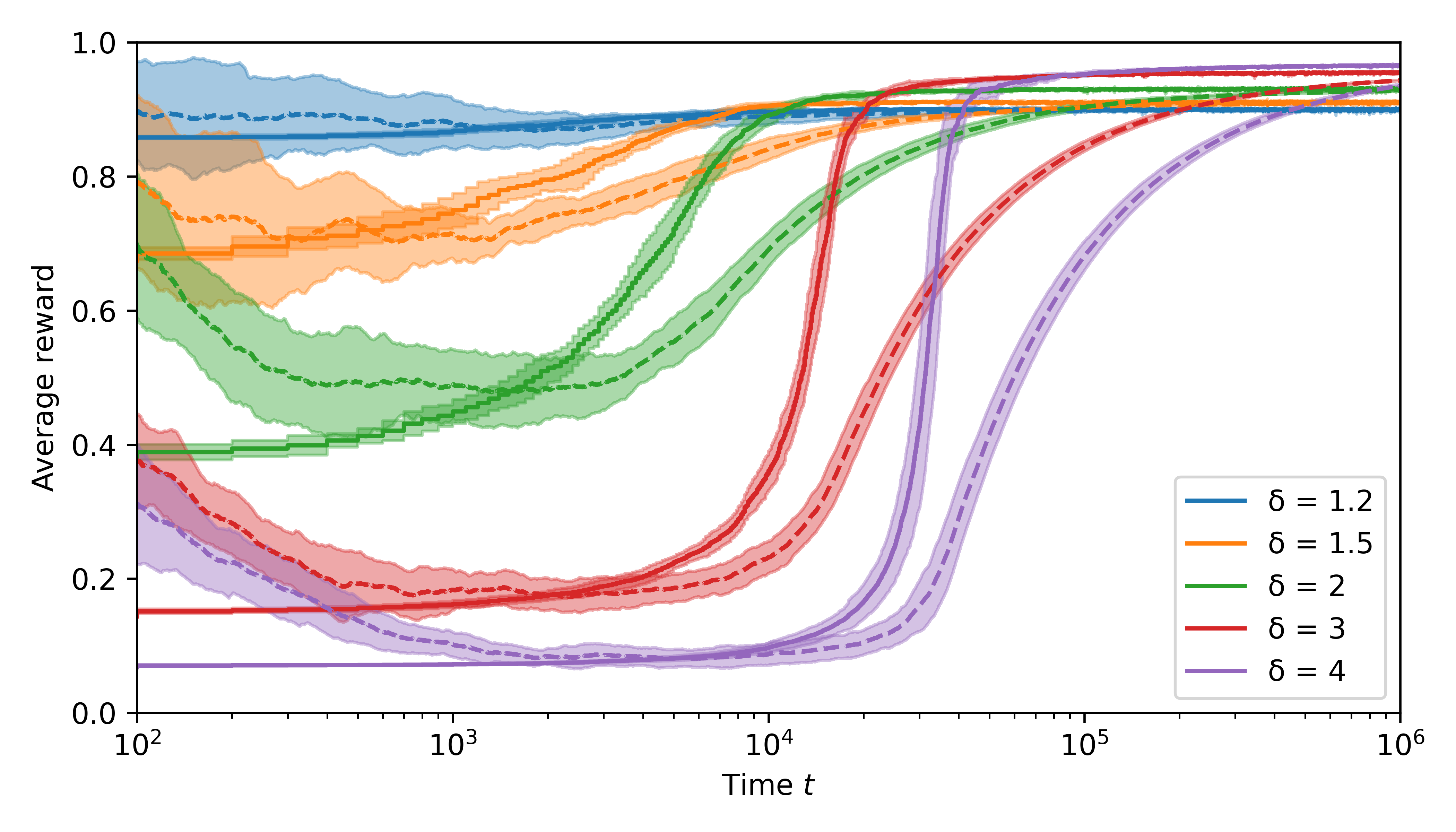}
		\caption{\gls{SAGE}, $n = 4$ servers.}
		\label{fig:lb-4-rwd-sage}
	\end{subfigure}
	\hfill
	\begin{subfigure}{.49\linewidth}
		\includegraphics[width=\textwidth]{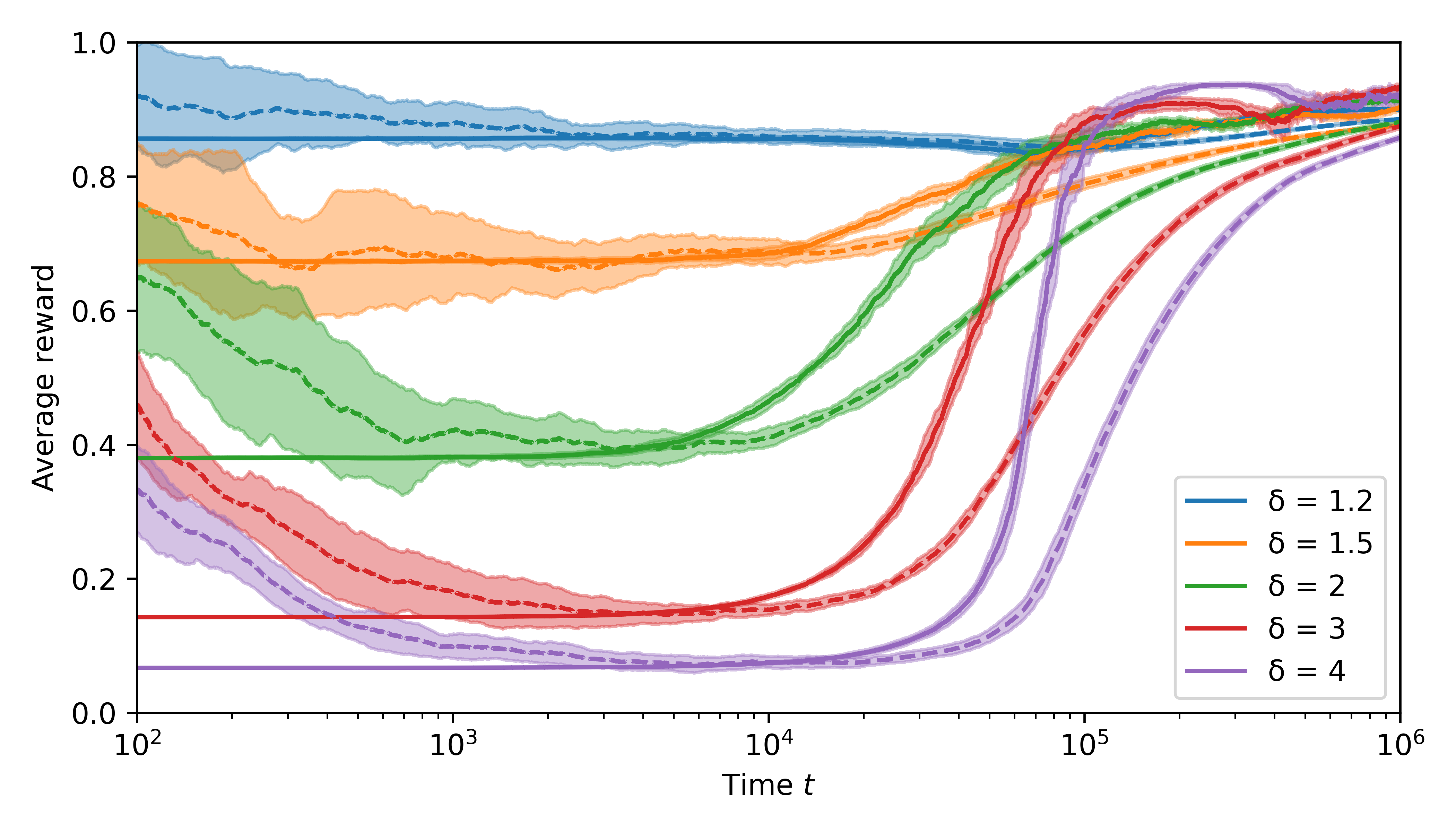}
		\caption{Actor--critic, $n = 4$ servers.}
		\label{fig:lb-4-rwd-ac}
	\end{subfigure}
	\\
	\begin{subfigure}{.49\linewidth}
		\includegraphics[width=\textwidth]{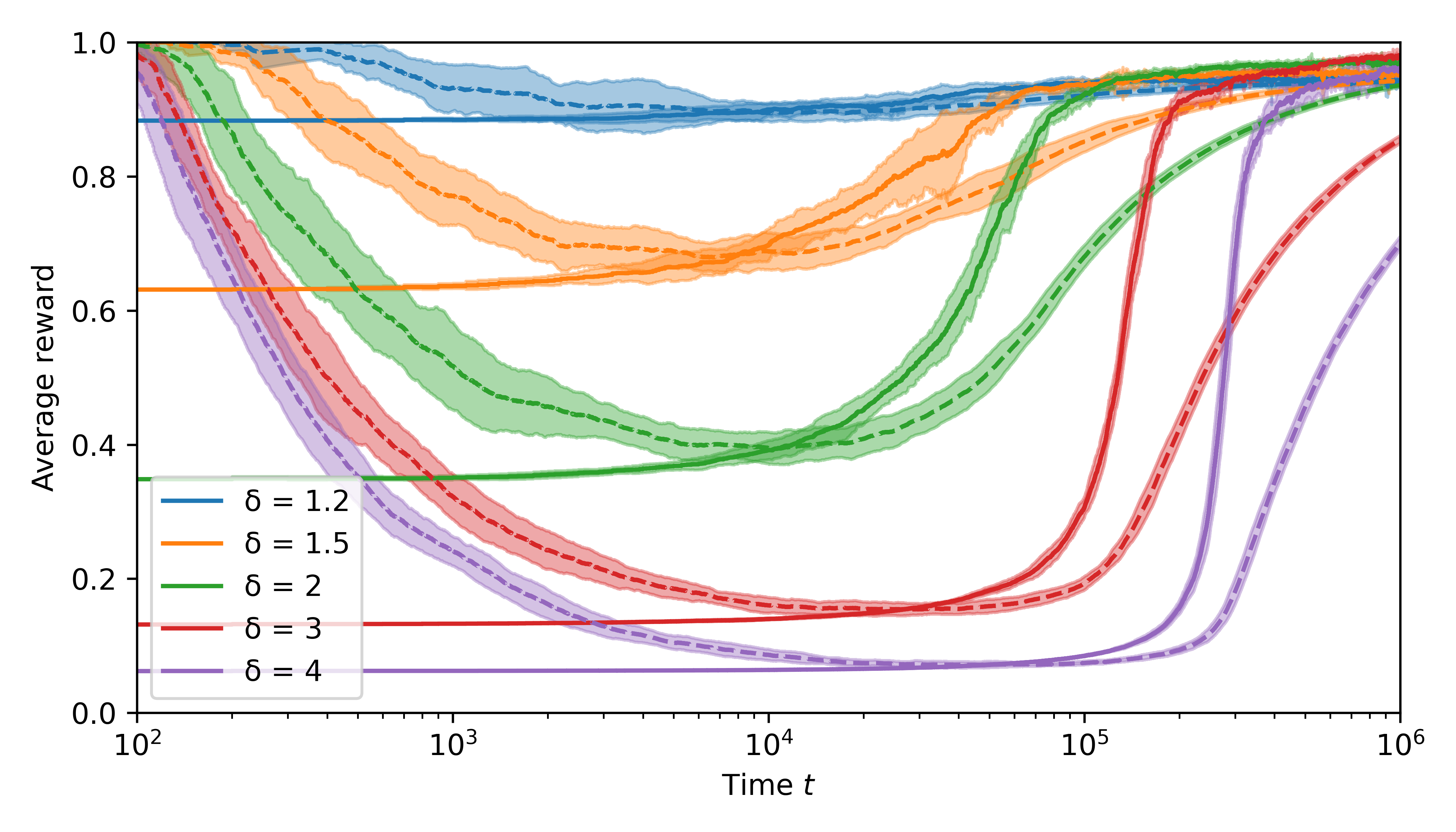}
		\caption{\gls{SAGE}, $n = 20$ servers.}
		\label{fig:lb-20-rwd-sage}
	\end{subfigure}
	\hfill
	\begin{subfigure}{.49\linewidth}
		\includegraphics[width=\textwidth]{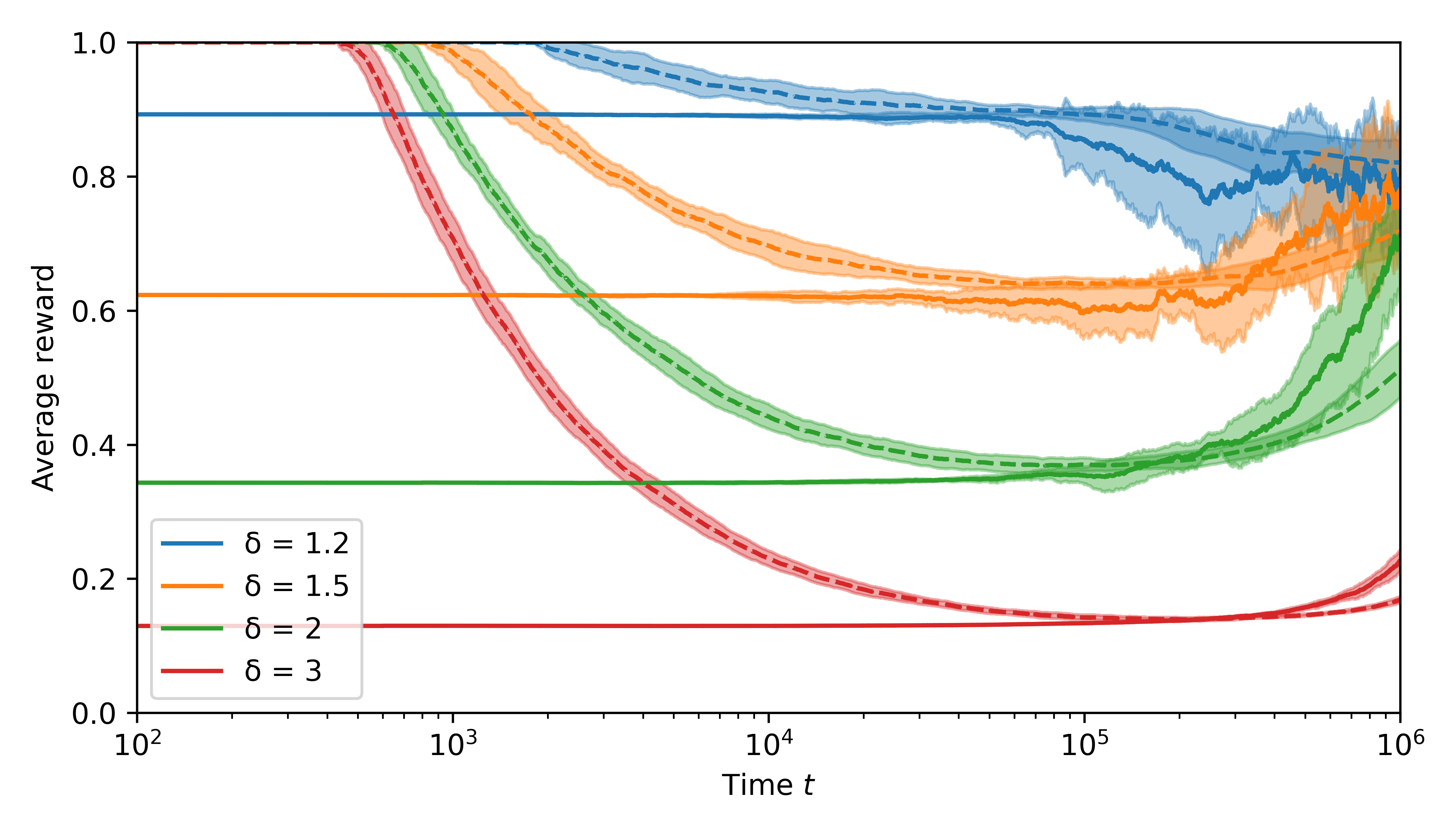}
		\caption{\gls{SAGE}, $n = 100$ servers.}
		\label{fig:lb-100-rwd-ac}
	\end{subfigure}
	\caption{Impact of the number of servers
		and service-rate imbalance
		on the performance of \gls{SAGE} and actor--critic
		in a load-balancing system.
		Solid lines show the long-run average reward $J(\Theta_t)$,
		while dashed lines show the running average reward,
		$\frac1t \sum_{t' = 1}^t R_{t'}$.
		Simulations for $n = 100$ and $\delta = 4$ are omitted
		because numerical instability of Buzen's algorithm (see \Cref{app:lb})
		prevents us from computing $J(\Theta_t)$ in this case.
	}
	\label{fig:lb-rwd}
\end{figure}
\Cref{fig:lb-rwd} shows performance
of \gls{SAGE} and actor--critic
in clusters of $n \in \{4, 20, 100\}$ servers.
Solid lines show the evolution of
the long-run average reward $J(\Theta_t)$,
and dashed lines show the running average reward $\frac1t \sum_{t' = 1}^t R_{t'}$.
(Recall $J(\Theta_t)$ is the limit of the running average we would see
if we ran the system under policy $\pi(\Theta_t)$.
It is defined in~\eqref{eq:J} and can be computed as shown in \Cref{app:lb}.)
As before, transparent areas show the standard deviation
around the average.
The results under actor--critic are reported only for $n = 4$ servers,
as this method already suffers from a combinatorial explosion in the state--action space for $n \in \{20, 100\}$.
Indeed, while the memory complexity
increases linearly with the number~$n$ of servers under \gls{SAGE},
it increases with the cardinality $\binom{n+c}{c}$ of the state space under actor--critic\footnote{As shown by applying the \emph{stars and bars} method in combinatorics.},
which is already prohibitively large
for $n \in \{20, 100\}$.

All four subfigures in \Cref{fig:lb-rwd}
show a consistent 2-phase pattern:
first the running average reward $\frac1t \sum_{t' = 1}^t R_{t'}$ converges
to the initial long-run average reward $J(\Theta_0)$,
and then the long-run average reward increases
to reach the best value,
with the running average reward
catching up at a slower pace.
This suggests that the gradient estimates under both algorithms
remain close to zero until the system reaches approximate stationarity.
A similar reasoning explains why the algorithms converge
at a slower pace
when we increase
the imbalance factor~$\delta$
(as the stationary distribution
under the initial uniform policy~$\pi(\Theta_0)$
puts mass on states that are further away from the initial empty state)
or the number~$n$ of servers (as the mixing time increases).

Focusing on the system with $n = 4$ servers,
\Cref{fig:lb-4-rwd-sage,fig:lb-4-rwd-ac} show
convergence occurs approximately ten times faster
under \gls{SAGE} than under actor--critic.
We conjecture that this is again due to the fact that
actor--critic relies on estimating the state-value function,
so that it needs to estimate $\binom{n + c}c$ values.
\gls{SAGE}, on the other hand,
exploits the structure of the stationary distribution and
only needs to estimate a number of values
that grows linearly with~$n$
(and is independent of~$c$).
We also note that actor--critic shows nonmonotonic convergence
(i.e., $J(\Theta_t)$ decreases before increasing again
between $10^5$ and $10^6$ time steps).
We conjecture this is due to a similar phenomenon
as described in \Cref{sec:Admission-control-in-an-MM1}.

\subsection{Ising Model and Glauber Dynamics}
\label{sec:ising}

Consider a system of spin particles spread over
a two-dimensional lattice of shape $d_1 \times d_2$,
for some $d_1, d_2 \in \{2, 3, 4, \ldots \}$.
Let $\cV = \{ 1, \allowbreak 2, \allowbreak \ldots, \allowbreak d_1\}
\times \{ 1, \allowbreak 2, \allowbreak \ldots, \allowbreak d_2 \}$
denote the set of lattice coordinates.
For any two coordinates $v = (v_1, v_2) \in \cV$ and $w = (w_1, w_2) \in \cV$,
we write $v \sim w$ if and only $v$ and $w$ are neighbors in the lattice,
that is, if and only if $|v_1 - w_1| + |v_2 - w_2| = 1$.

A map $\sigma: \cV \to \{-1, +1\}$ is called a \emph{spin configuration},
and the set of all $2^{d_1 d_2}$ configurations is denoted by~$\Sigma$.
Given a configuration $\sigma \in \Sigma$,
we refer to $\sigma(v) \in \{-1, +1\}$ as
the \emph{spin} (of the particle located) at~$v$.
If the system is in some configuration~$\sigma \in \Sigma$,
we say that the spin at~$v \in \cV$ is \emph{flipped}
if the system jumps to the configuration~$\sigma_{-v} \in \Sigma$
such that $\sigma_{-v}(v) = - \sigma(v)$
and $\sigma_{-v}(w) = \sigma(w)$ for each $w \in \cV \setminus \{v\}$.
As we will formalize below, the agent's goal
is to reach a configuration $\sigma \in \Sigma$
so that the \emph{magnetization}
on the left (resp.\ right) half of the lattice
is close to $\xi_{\tleft} \in (-1, +1)$ (resp.\ $\xi_{\tright} \in (-1, +1)$),
i.e.,
\begin{align*}
	\frac2{d_1 d_2} M_\tleft(\sigma) &\simeq \xi_\tleft,
	&
	\frac2{d_1 d_2} M_\tright(\sigma) &\simeq \xi_\tright,
\end{align*}
where
\begin{align*}
	M_\tleft(\sigma) &= \sum_{v \in \cV:\, v_2 \le d_2 / 2} \sigma(v),
	&
	M_\tright(\sigma) &= \sum_{v \in \cV: \, v_2 > d_2 / 2} \sigma(v).
\end{align*}

To each configuration $\sigma \in \Sigma$
is associated an \emph{energy}
$E(\sigma) \triangleq - J I(\sigma) - \mu F(\sigma)$,
where $I$ and $F$ are called the \emph{interaction} and \emph{external field} terms,
respectively, given by
\begin{align*}
	I(\sigma)
	&= \sum_{\substack{v, w \in \cV: \, v \sim w}} \sigma(v) \sigma(w),
	&
	F(\sigma)
	&= \sum_{v \in \cV} h(v) \sigma(v),
\end{align*}
where the first sum runs over all pairs of neighboring coordinates
(so that each pair appears once).
Here, $J \in \bR$ is the \emph{coupling constant},
$\mu \in \bR_{\ge 0}$ the \emph{magnetic moment},
and $h: \cV \to \bR$ the \emph{external magnetic field}.
Under the dynamics defined below,
the probability of a configuration~$\sigma \in \Sigma$
will be proportional to $e^{- \beta E(\sigma)}$,
where $\beta \in \bR_{> 0}$ is the \emph{inverse temperature}.
If $J > 0$ (resp.\ $J < 0$),
the interaction term~$I$ contributes to increasing the probability of configurations
where neighboring spins have the same (resp.\ opposite) sign.
Concurrently, due to the external-field term~$F$,
the spin at each~$v \in \cV$ is attracted in the direction pointed by the sign of~$h(v)$.
The coupling constant~$J$ and magnetic moment~$\mu$
are fixed and known by the agent
(as they depend on the particles),
and the agent will
fine-tune the inverse temperature~$\beta$
and coarse-tune the external magnetic field~$h$.

\subsubsection{Glauber Dynamics}

Given a starting configuration, at every time step,
the spin at a coordinate chosen uniformly at random
is flipped (or not)
with some probability that depends on
the current configuration and the parameters set by the agent.
This is cast as a Markov decision process as follows.
The state and action spaces are given by
$\cS = \Sigma \times \cV$
and $\cA = \{\flip, \notflip\}$,
respectively.
For each $s = (\sigma, v) \in \cS$
and $a \in \cA$,
the state reached by taking action~$a$ in state~$s$
is given by $S' = (\sigma', V')$, where
$\sigma' = \sigma_{-v}$ if $a = \flip$
and $\sigma' = \sigma$ if $a = \notflip$,
and $V'$ is chosen uniformly at random in~$\cV$,
independently of the past states, actions, and rewards.
The next reward~$r$ is the opposite of the sum of the absolute difference
between the next magnetizations and the desired magnetizations, that is,
\begin{align*}
	r
	&=
	- \Big| \xi_{\tleft} - \frac2{d_1 d_2} M_\tleft(\sigma') \Big|
	- \Big| \xi_{\tright} - \frac2{d_2 d_2} M_\tright(\sigma') \Big|.
\end{align*}
The agent controls a vector~$\theta \in \bR^3$
that determines the inverse temperature
and the left and right external magnetic fields as follows:
\begin{align*}
	\beta(\theta)
	&=
	1 + \tanh(\theta_1),
	&
	h_{\tleft}(\theta)
	&= \tanh(\theta_2),
	&
	h_{\tright}(\theta)
	&= \tanh(\theta_3),
\end{align*}
so that in particular $\beta(\theta) \in (0, 2)$,
$h_{\tleft}(\theta) \in (-1, 1)$,
and $h_{\tright}(\theta) \in (-1, 1)$.
The corresponding external magnetic field
and external field term are
\begin{align*}
	h(v | \theta)
	&= h_{\tleft}(\theta) \indicator{v_2 \le d_2 / 2}
	+ h_{\tright}(\theta) \indicator{v_2 > d_2 / 2},
	\quad v \in \cV,
	\\
	F(\sigma | \theta)
	&= \sum_{v \in \cV} h(v | \theta) \sigma(v)
	= h_{\tleft}(\theta) M_{\tleft}(\sigma)
	+ h_{\tright}(\theta) M_{\tright}(\sigma),
	\quad \sigma \in \Sigma.
\end{align*}
Given $\theta \in \bR^3$,
for each $s = (\sigma, v) \in \Sigma$,
the probability that the spin at the randomly-chosen coordinate~$v$
is flipped when the current configuration is~$\sigma$ is given by
\begin{align} \label{eq:ising-policy}
	\pi(\flip | s, \theta)
	&= \frac1{1 + e^{\delta(s | \theta)}},
	\quad \text{with} \quad
	\delta(s | \theta)
	= 2 \beta(\theta) \sigma(v)
	\left(
	J \sum_{\substack{w \in \cV: \, w \sim v}} \sigma(w)
	+ \mu h(v | \theta)
	\right).
\end{align}
When $\theta \in \bR^3$ is fixed,
the dynamics defined by this system are called the Glauber dynamics
\cite[Section~3.3]{LP17}.
Note that, although we use the word \emph{action}
to match the terminology of \glspl{MDP},
here an action should be seen as
a random event in the environment,
of which only the distribution $\pi$ can be controlled by the agent
via the parameter vector~$\theta$.

We verify in \Cref{app:ising} that
the stationary distribution of the system state
under a particular choice of $\theta \in \bR^3$ satisfies
\begin{align} \label{eq:ising-stationary-distribution}
	p(s | \theta)
	&\propto e^{\beta(\theta) (J I(\sigma) + \mu F(\sigma | \theta))},
	\quad s = (\sigma, v) \in \cS,
	\quad \theta \in \bR^3.
\end{align}
\Cref{ass:markov,ass:reward,ass:stat}
are satisfied with $n = d = 3$,
$\Omega = \bR^3$,
$\Phi(s) = 1$ for each $s \in \cS$,
$\log \rho_1(\theta) = \beta(\theta) J$,
$\log \rho_2(\theta) = \beta(\theta) \mu h_{\tleft}(\theta)$,
and $\log \rho_3(\theta) = \beta(\theta) \mu h_{\tright}(\theta)$
for each $\theta \in \bR^3$,
and $x_1(s) = I(\sigma)$,
$x_2(s) = M_\tleft(\sigma)$,
and $x_3(s) = M_\tright(\sigma)$
for each $s = (\sigma, v) \in \cS$.
All derivations are given in \Cref{app:ising}.

\subsubsection{Numerical Results}

\Cref{fig:ising} shows the performance of \gls{SAGE}
in a system with parameters
$d_1 = 10$, $d_2 = 20$, $J = \mu = 1$,
$\xi_\tleft = -1$, and $\xi_\tright = 1$.
We do not run simulations under actor--critic,
as again the state space has size $2^{d_1 d_2} = 2^{200}$,
which is out of reach for this method.
The initial parameter vector is $\Theta_0 = 0$,
yielding inverse temperature $\beta(\Theta_0) = 1$
and external fields
$h_\tleft(\Theta_0) = h_\tright(\Theta_0) = 0$.
The initial configuration
has spins~$1$ on the left-hand side
and $-1$ on the right-hand side,
so that reaching the target configuration
requires flipping every spin.
In \Cref{fig:ising-10},
the reward~$R_t$ seems to increase on average monotonically from $-4$ to $0$,
which is consistent with the observation that
the left (resp.\ right) magnetization
decreases from $1$ to $-1$
(resp.\ increases from $-1$ to $1$).
The increase of the reward is stepwise,
with stages where it remains roughly constant
for several thousand time steps.%
\begin{figure}[ht]
	\begin{subfigure}{.49\linewidth}
		\includegraphics[width=\textwidth]
		{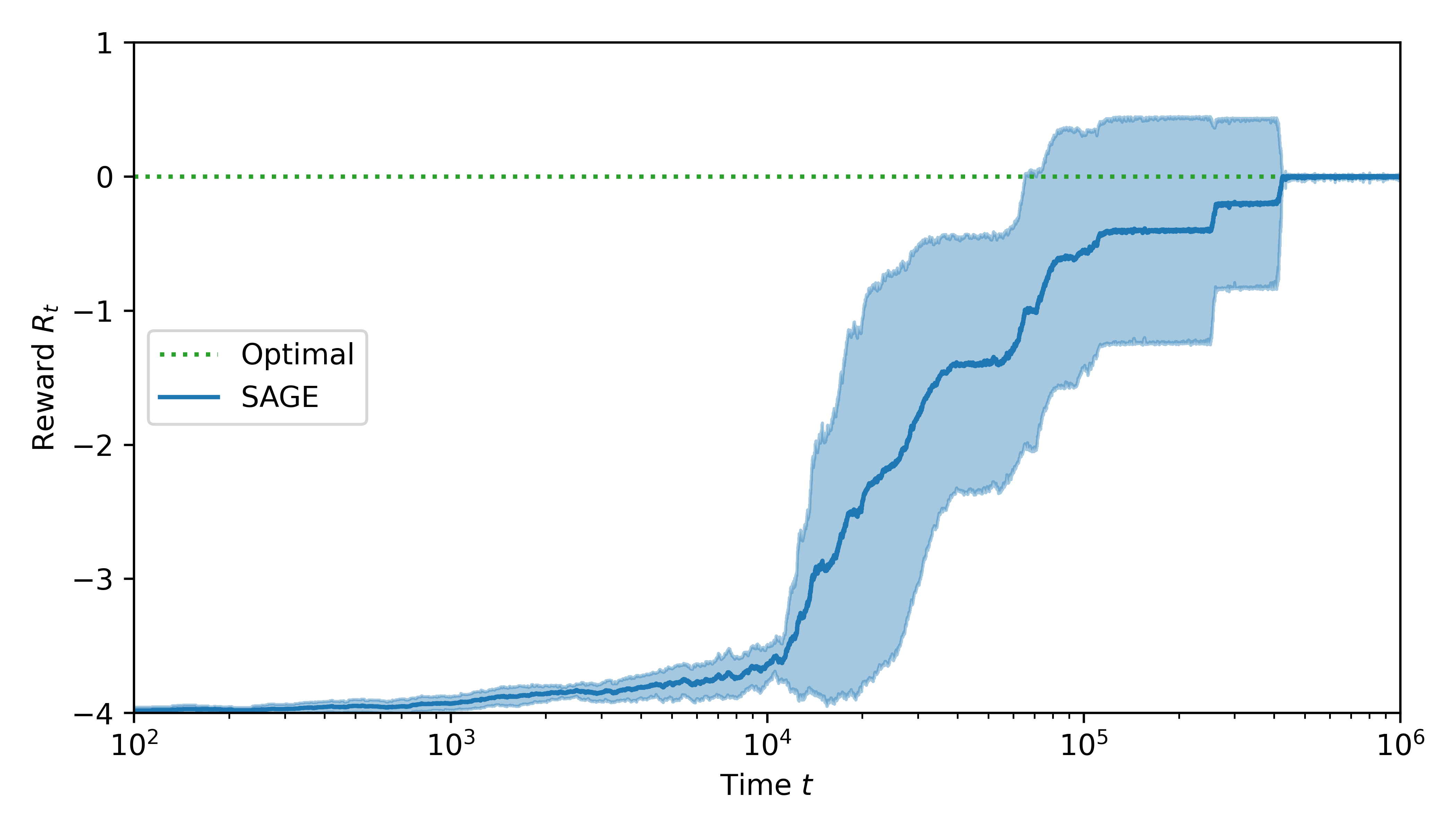} \\
		\includegraphics[width=\textwidth]
		{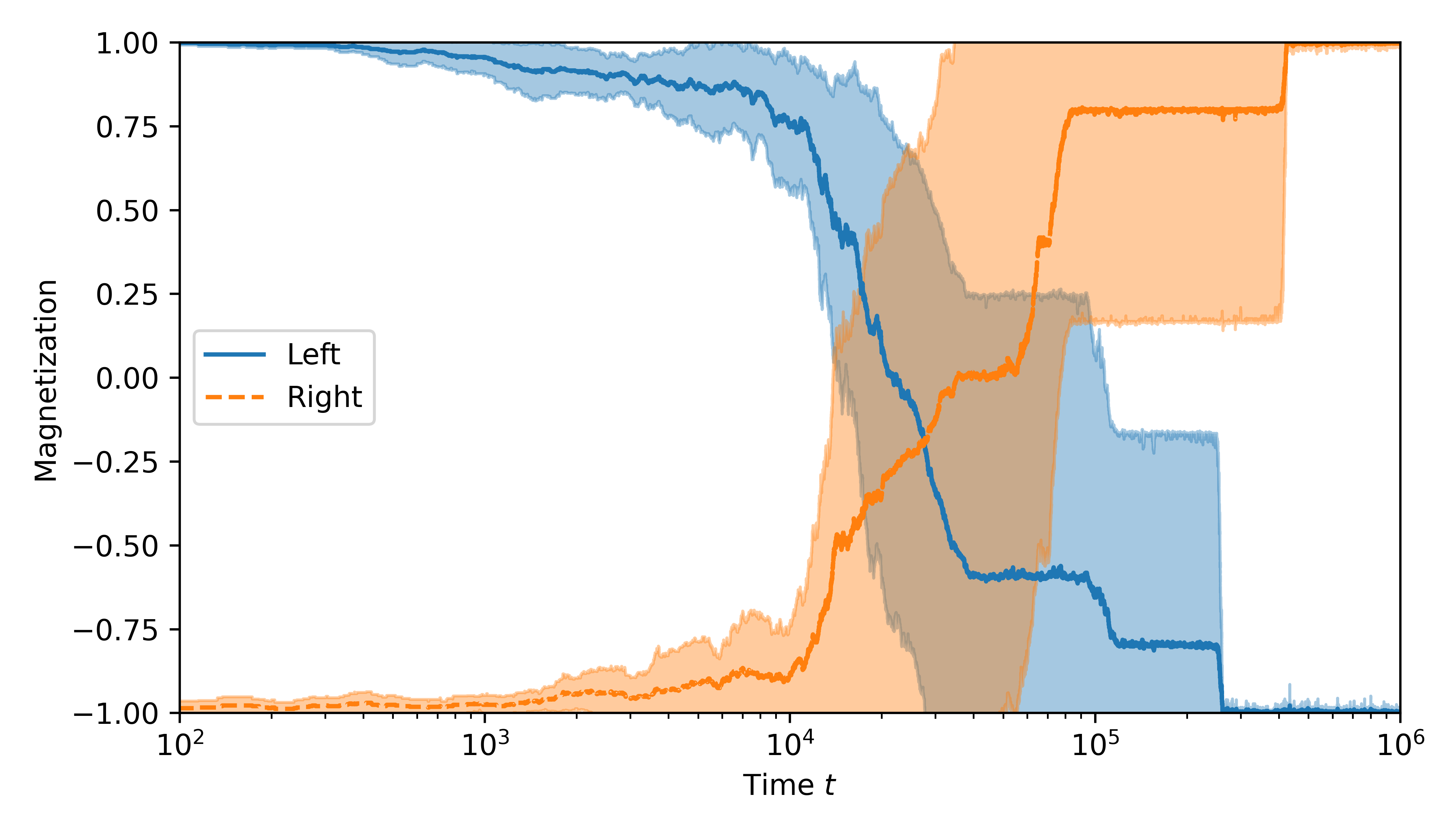}
		\caption{Average over 10 simulation runs.}
		\label{fig:ising-10}
	\end{subfigure}
	\hfill
	\begin{subfigure}{.49\linewidth}
		\includegraphics[width=\textwidth]
		{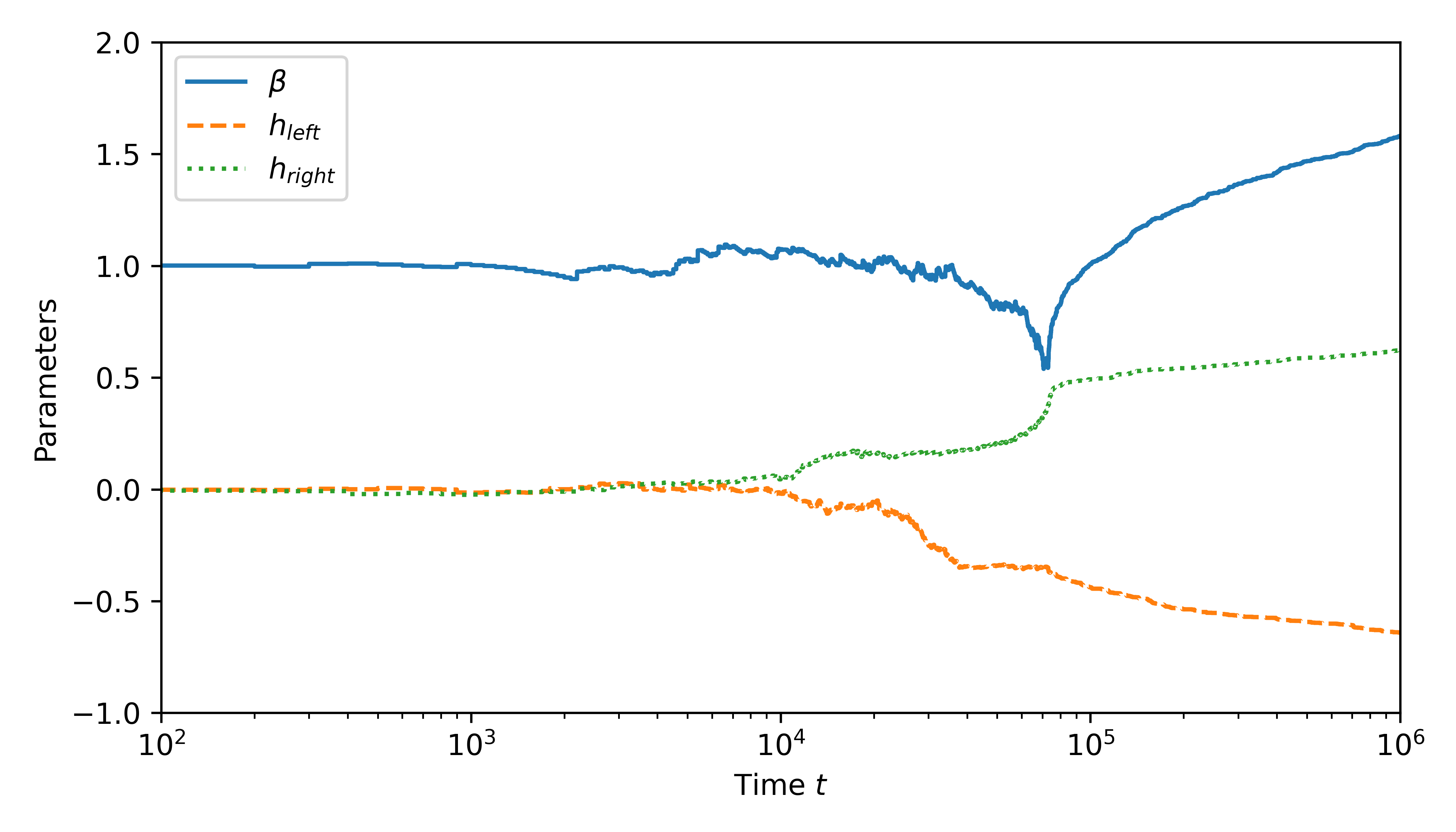}
		\includegraphics[width=\textwidth]
		{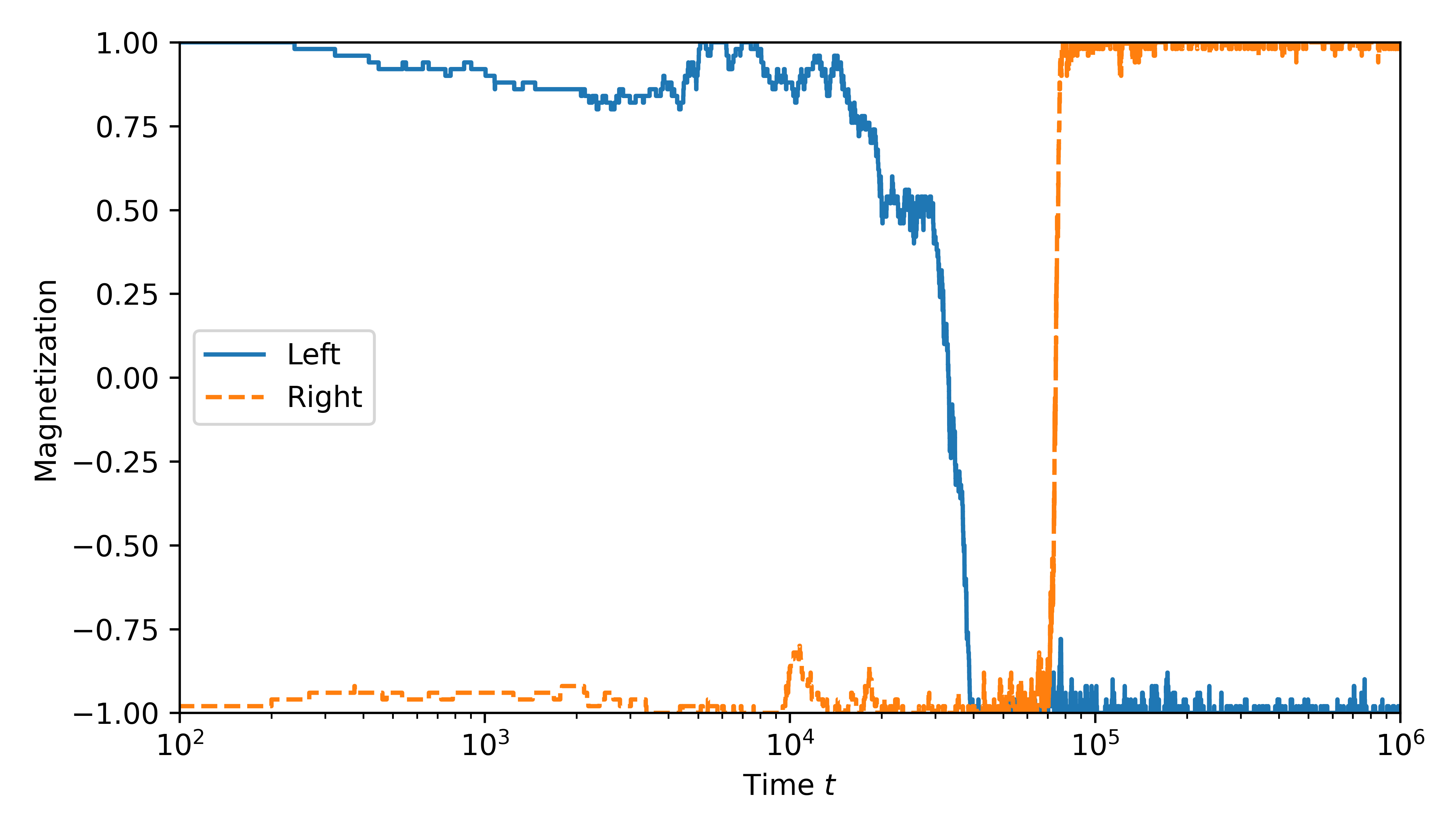}
		\caption{A particular simulation run.}
		\label{fig:ising-1}
	\end{subfigure}
	\caption{Performance of \gls{SAGE} in the Ising model.}
	\label{fig:ising}
\end{figure}
Lastly, the standard deviation increases significantly
from about $10^4$ to $3 \cdot 10^5$ time steps,
and it becomes negligible afterwards.

To help us understand these observations,
\Cref{fig:ising-1} shows
the evolution of the system parameters
and of the magnetizations
over a particular simulation run.
The left magnetization $M_{\textnormal{left}}$ starts decreasing
around $10^4$ time steps (bottom plot),
approximately when $h_\tleft(\Theta_t)$ and $h_\tright(\Theta_t)$
become nonzero (top plot),
to become roughly $-1$ around $3 \cdot 10^4$ time steps.
At that moment, the system configuration is close
to the all--spin--down configuration
$\sigma_{-1}$ such that $\sigma_{-1}(v) = -1$
for each $v \in \cV$.
The right magnetization starts increasing significantly only
when the inverse temperature $\beta(\Theta_t)$ has a sudden decrease (top plot).
To make sense of this observation,
consider $\pi(\flip | s, \theta)$ as given
by~\eqref{eq:ising-policy},
where $s = (\sigma_{-1}, v)$
for some $v \in \{2, 3, \ldots, d_1 - 1\}
\times \{2, 3, \ldots, d_2 - 1\}$.
In $\delta(s | \theta)$,
the first term $J \sum_{w \in \cV: w \sim v} \sigma_{-1}(v) \sigma_{-1}(w)$
is equal to 4, while the absolute value of
the second term $\mu \sigma_{-1}(v) h(v | \theta)$ is at most~$1$;
hence, if $\beta(\theta) \simeq 1$ as initially,
$\pi(\flip | s, \theta)$ is between
$\frac1{1+e^{2(4+1)}} \simeq 4.5 \cdot 10^{-5}$
and $\frac1{1+e^{2(4-1)}} \simeq 2.5 \cdot 10^{-3}$.
The brief decrease of $\beta(\theta)$
is an efficient way of increasing
the flipping probability in all states,
which allows the system to escape from~$\sigma_{-1}$.
Other simulation runs are qualitatively similar,
but the times at which the qualitative changes occur
and the side (left or right) that flips magnetization first vary,
which explains the large standard deviation observed earlier.

\section{Conclusion}

In this paper, we incorporated model-specific information about \glspl{MDP} into the gradient estimator in policy-gradient methods.
Specifically, assuming that the stationary distribution is an exponential family,
we derived \acrfullpl{SAGE} that do not require the computation of value functions (Theorem~\ref{theo:sage}).
As showcased in \Cref{sec:examples},
this assumption is satisfied by models from stochastic networks,
where the stationary distribution possesses a product-form structure,
and by models from statistical mechanics,
such as the Ising model with Glauber dynamics.

The numerical results in \Cref{sec:examples} show that in these systems, policy-gradient algorithms equipped with a \gls{SAGE} outperform actor--critic.
In these examples, the Jacobian of the load function $\mathrm{D} \log \rho(\theta)$ can be computed explicitly in terms of the policy parameter~$\theta$.
However, \gls{SAGE} estimators can be harder to compute in more complex cases,
for example when $\mathrm{D} \log \rho(\theta)$ depends on some model parameters.
Nevertheless, our examples showcase how it is possible to improve the current policy gradient methods by levering information on the \gls{MDP}, and we expect extensions of \glspl{SAGE} to cover more challenging cases,
for example by combining \gls{SAGE} with model selection by first estimating the model parameters appearing in $\mathrm{D} \log \rho(\theta)$.
We leave such extensions of \gls{SAGE} for future work.

We have also shown with Theorem~\ref{prop:main_prop_convergence in probability_noncompact_case} that
policy gradient with \gls{SAGE} converges to the optimal policy under light assumptions,
namely, the existence of a local Lyapunov function close to the optimum,
which allows for unstable policies to exist,
and a nondegeneracy property of the Hessian at maxima.
The convergence occurs with a probability arbitrarily close to one
provided that the iterates start close enough. Notably, our method proof also works with other policy-gradients with similar policy-gradient approximation properties.
In Corollary~\ref{cor:regret}, the regret of the algorithm is shown to be $O(T^{2/3 + \epsilon} + T\alpha^2/\ell )$, where $T$ is the number of samples drawn.
Unlike most common convergence results, we have gradient estimation on a countable state space and there is a nonzero probability that such error will drive the algorithm to an unstable policy.
This fact is namely captured by the term $\alpha^2/\ell$.
Remarkably, such instabilities are observed in one of the examples of Section~\ref{sec:examples}.
If we had made stronger assumptions such as the existence of a global Lyapunov function, then such phenomena would not have been captured by the analysis.

\newpage

\appendix

\section{Policy-Gradient Algorithms}
\label{app:policy-gradient}

In this section, we give additional details about the two policy-gradient algorithms
that are evaluated numerically in \Cref{sec:examples}:
actor--critic in \Cref{app:actor--critic},
and \gls{SAGE}-based methods in \Cref{app:sage}.

\subsection{Actor--Critic Algorithm} \label{app:actor--critic}

The actor--critic algorithm is first mentioned in \Cref{sec:policy-gradient}
and compared to our \gls{SAGE}-based policy-gradient algorithm
in \Cref{sec:examples}.
We focus on the version of actor--critic
described in \cite[Section~13.6]{SB18}
for the average-reward criterion in infinite horizon.
The algorithm relies on the following expression for~$\nabla J(\theta)$,
which is a variant of the \emph{policy-gradient theorem} \cite[Chapter~13]{SB18}:
\begin{align*}
	\nabla J(\theta)
	&\propto
	\esp{
		\left( R - J(\theta) + v(S') - v(S) \right) \nabla \log \pi(A | S, \theta)
	},
\end{align*}
where $(S, A, R, S')$ is a quadruplet of random variables such that
$S \sim p(\,\cdot\,| \theta)$,
$A | S \sim \pi(\,\cdot\,| S, \theta)$,
and $(R, S') | (S, A) \sim P(\,\cdot\,,\,\cdot\,| S, A)$
(so that in particular $(S, A, R) \sim$ \ref{eq:stat}),
and $v$ is the state-value function.

\begin{algorithm}[htbp]
	\caption{Actor--critic algorithm \cite[Section~13.6]{SB18}
		to be called on \Cref{step:sga} of \Cref{algo:policy-gradient},
		with batch sizes equal to one.%
	}
	\label{algo:actor-critic}
	
	\begin{algorithmic}[1]
		\State \textbf{Input:}
		Positive and differentiable policy parametrization
		$(s, \theta, a) \mapsto \pi(a | s, \theta)$
		\vspace{0.25em}
		
		\State \textbf{Parameters:}
		Step sizes $\alpha_{\overline{R}} > 0$ and $\alpha_{v} > 0$
		\vspace{0.25em}
		
		\State \textbf{Initialization:}
		\begin{minipage}[t]{.7\textwidth}
			$\bullet$ $\overline{R} \gets 0$ \\
			$\bullet$ $V[s] \gets 0$ for each $s \in \cS$
		\end{minipage}
		\vspace{0.25em}
		
		\Procedure{Gradient}{$t$}
		\State $\delta \gets R_{t+1} - \overline{R} + V[S_{t+1}] - V[S_t]$
		\State Update $\overline{R} \gets \overline{R} + \alpha_{\overline{R}} \delta$ \label{step:overlineRac}
		\State Update $V[S_t] \gets V[S_t] + \alpha_v \delta$ \label{step:state-value-fct}
		\State \textbf{return}
		$\delta \, \nabla \log \pi(A_t | S_t, \Theta_t)$
		\EndProcedure
	\end{algorithmic}
\end{algorithm}

The pseudocode of the procedure \textsc{Gradient}
used in the actor--critic algorithm
is given in \Cref{algo:actor-critic}.
This procedure is to be implemented within \Cref{algo:policy-gradient}
with batch sizes equal to one, meaning that
$t_{m+1} = t_m + 1$ for each $m \in \bN$.
We assume for simplicity that all variables from \Cref{algo:policy-gradient}
are accessible inside \Cref{algo:actor-critic}.
The variable~$\overline{R}$ updated on \Cref{step:overlineRac}
is a biased estimate of $J(\Theta_m)$,
while the table~$V$ updated on \Cref{step:state-value-fct}
is a biased estimate of the state-value function
under policy~$\pi(\Theta_m)$.
Compared to \cite[Section~13.6]{SB18},
the value function is encoded by a table~$V$
and there are no eligibility traces.
If the state space~$\cS$ is infinite,
the table~$V$ is initialized at zero
over a finite subset of~$\cS$ containing the initial state~$S_0$
and expanded with zero padding whenever necessary.

\subsection{\gls{SAGE}-Based Policy-Gradient Method}
\label{app:sage}

\begin{algorithm}[htb]
	\caption{SAGE-updated policy-gradient method,
		to be called on \Cref{step:sga} of \Cref{algo:policy-gradient}.
	}
	\label{algo:sage-alternative}
	
	\begin{algorithmic}[1]
		\State \textbf{Input:}
		\begin{minipage}[t]{.8\textwidth}
			$\bullet$
			Positive and differentiable policy parametrization
			$(s, \theta, a) \mapsto \pi(a | s, \theta)$\\
			$\bullet$
			Jacobian matrix function $\theta \mapsto \rD \log \rho(\theta)$ \\
			$\bullet$
			Feature function $s \mapsto x(s)$
		\end{minipage}
		\vspace{0.25em}
		
		\State \textbf{Parameters:} \label{row:nu}
		\begin{minipage}[t]{.48\textwidth}
			Memory factor $\nu \in [0, 1]$
		\end{minipage}
		\vspace{0.25em}
		
		\State \textbf{Initialization:}
		\begin{minipage}[t]{.7\textwidth}
			Global~variables~$N_{-1}, M_{-1}, \overline{X}_{-1}, \overline{R}_{-1}, \overline{C}_{-1}, \overline{E}_{-1} \gets 0$
		\end{minipage}
		\vspace{0.25em}
		
		\Procedure{Gradient}{$m$}
		\State \label{row:N-bis}
		$N_m \gets \nu N_{m-1} + (t_{m+1} - t_m)$
		\State \label{row:M-bis}
		$M_m \gets \nu^2 M_{m-1} + (t_{m+1} - t_m)$
		\State \textbf{return}\label{step:sage}
		$
		\rD \log \rho(\Theta_m) \textsc{Covariance}(m) + \textsc{Expectation}(m)
		$
		\EndProcedure
		\vspace{0.25em}
		
		\Procedure{Covariance}{$m$}
		\State Update\label{step:overlineX}
		$
		\overline{X}_m
		\gets
		\nu \overline{X}_{m-1}
		+
		\sum_{t=t_m}^{t_{m+1}-1}
		x(S_{t})
		$
		\State Update\label{step:overlineR}
		$
		\overline{R}_m
		\gets
		\nu \overline{R}_{m-1}
		+
		\sum_{t=t_m}^{t_{m+1} - 1}
		R_{t+1}
		$
		\State Update\label{step:overlineC}
		$
		\overline{C}_m
		\gets
		\nu \overline{C}_{m-1}
		+
		\sum_{t=t_m}^{t_{m+1}-1}
		( x(S_t) - \frac1{N_m} \overline{X}_m ) ( R_{t + 1} - \frac1{N_m} \overline{R}_m )
		$
		\State \textbf{return}\label{step:return}
		$\frac{N_m}{{N_m}^2 - M_m} \overline{C}_m$
		if ${N_m}^2 > M_m$ else $\frac1{N_m} \overline{C}_m$
		\EndProcedure
		\vspace{0.25em}
		
		\Procedure{Expectation}{$m$}
		\State Update\label{step:overlineE}
		$
		\overline{E}_m
		\gets
		\nu \overline{E}_{m-1}
		+
		\sum_{t=t_m}^{t_{m+1}-1}
		R_{t+1}
		\nabla \log\pi(A_t | S_t, \Theta_m)
		$
		\State \textbf{return}
		$\frac1{N_m} \overline{E}_m$
		\EndProcedure
	\end{algorithmic}
\end{algorithm}

\Cref{algo:sage-alternative}
is an extension of \Cref{algo:sage}
that allows for batches of size~$1$.
The main advantage of \Cref{algo:sage-alternative}
over \Cref{algo:sage}
is that it estimates $\nabla J(\Theta_m)$
based not only on batch~$\cD_m$,
but also on previous batches,
depending on the memory factor~$\nu$
initialized on \Cref{row:nu}.
To simplify the signature of procedures in \Cref{algo:sage-alternative},
we assume variables $N_m$, $M_m$, $\overline{X}_m$, $\overline{R}_m$, $\overline{C}_m$, and $\overline{E}_m$ are global,
and that all variables from \Cref{algo:policy-gradient}
are accessible within \Cref{algo:sage-alternative},
in particular batch~$\cD_m$.
\Cref{row:N-bis}
in \Cref{algo:sage-alternative}
is the counterpart of
\Cref{row:N}
in \Cref{algo:sage}.
\Cref{row:M-bis}
in \Cref{algo:sage-alternative}
is used on \Cref{step:return}
to compute the counterpart of
the quotient $1 / (N_m - 1)$
in \Cref{row:C}
in \Cref{algo:sage}.
The \textsc{Covariance}($m$) procedure
in \Cref{algo:sage-alternative}
is the counterpart of
Lines~\ref{row:X}--\ref{row:C}
in \Cref{algo:sage}.
The \textsc{Expectation}($m$) procedure
in \Cref{algo:sage-alternative}
is the counterpart of
\Cref{row:E}
in \Cref{algo:sage}.
\Cref{algo:sage} can be seen as
a special case of \Cref{algo:sage-alternative}
with memory factor $\nu = 0$.
Note that terminology
in \Cref{algo:sage-alternative}
differs slightly compared to \Cref{algo:sage}:
bar notation refers to cumulative sums
instead of averages.

The subroutines \textsc{Covariance} and \textsc{Expectation}
compute biased covariance and mean estimates for
$\cov{R, x(S)}$ and $\esp{R \, \nabla \log \pi(A | S, \theta)}$,
where $(S, A, R) \sim$ \hyperref[eq:stat]{\textsc{stat}($\Theta_m$)},
consistently with \Cref{theo:sage}.
If the memory factor $\nu$ is zero,
these procedures return the usual sample mean and covariance estimates
taken over the last batch~$\cD_m$ (as in \Cref{algo:sage}),
and bias only comes from the fact
that the system is not stationary.
If $\nu$ is positive, estimates from previous batches are also taken into account,
so that the bias is increased in exchange for a (hopefully) lower variance.
In this case,
the updates on Lines \ref{step:overlineX}--\ref{step:overlineC} and \ref{step:overlineE}
calculate iteratively the weighted sample mean and covariance over the whole history,
where observations from epoch $m - \underline{m}$ have weight $\nu^{\underline{m}}$,
for each $\underline{m} \in \{0, 1, \ldots, m\}$.
When $m$ is large,
the mean returned by \textsc{Expectation}
is approximately equal to
the sample mean over batches $\cD_{m - M}$ through $\cD_m$,
where $M$ is a truncated geometric random variable,
independent of all other random variables,
such that $\prb{M = \underline{m}} \propto \nu^{\underline{m}}$
for each $\underline{m} \in \{0, 1, \ldots, m\}$;
if batches have constant size~$c$, then
we take into account approximately the last
$c (\esp{M} + 1) = \frac{c}{1 - \nu}$ steps.

\section{Examples} \label{app:examples}

This appendix provides detailed derivations
for the examples of \Cref{sec:examples}.
We consider the single-server queue with admission control of \Cref{sec:Admission-control-in-an-MM1}
in \Cref{app:mm1},
the load-balancing example of \Cref{sec:Load-balancing-system}
in \Cref{app:lb},
and the Ising model of \Cref{sec:ising}
in \Cref{app:ising}.

\subsection{Single-Server Queue with Admission Control} \label{app:mm1}

Consider the example of \Cref{sec:Admission-control-in-an-MM1},
where jobs arrive according to
a Poisson process with rate $\lambda > 0$,
and service times are exponentially distributed with rate $\mu > 0$.
Recall the long-run average reward is the difference between
an admission reward proportional to the admission probability
and a holding cost proportional to the mean queue size.
We first verify
that \Cref{ass:markov,ass:reward,ass:stat} are satisfied,
then we give a closed-form expression
for the objective function,
and lastly we discuss the assumptions of \Cref{sec:convergence}.
We consider a random threshold-based policy~$\pi_k$
of the form~\eqref{eq:mm1-policy}
for some $k \in \bN$
and some parameter~$\theta \in \Omega$,
where $\Omega = \{\theta \in \bR^{k+1}: \pi_k(\accept | k, \theta) < \frac\mu\lambda\}$.

\subsubsection{Product-Form Stationary Distribution}

The evolution of the number of jobs in the system
(either waiting or in service)
defines a birth-and-death process
with birth rate $\lambda \pi_k(\accept | s, \theta)$
and death rate $\mu \indicator{s \ge 1}$
in state~$s$, for each $s \in \{0, 1, 2, \ldots\}$.
This process is irreducible
because its birth and death rates are positive,
and it is positive recurrent because
$\lambda \pi_k(\accept | s, \theta) < \mu$ for each $s \in \cS$
by definition of~$\Omega$.
This verifies \Cref{ass:markov}.
The stationary distribution is given by
\begin{align}
	\nonumber
	p(s | \theta)
	&= \frac1{Z(\theta)}
	\prod_{q = 0}^{s - 1}
	\left( \frac\lambda\mu \pi(\accept | q, \theta) \right), \\
	\label{eq:mm1-stationary}
	&= \frac1{Z(\theta)}
	\left( \frac\lambda\mu \right)^s
	\left[
	\prod_{i = 0}^{k-1}
	\pi_k(\accept | i, \theta)^{\indicator{s \ge i+1}}
	\right]
	\pi_k(\accept | k, \theta)^{\max(s - k, 0)},
	\quad s \in \bN,
\end{align}
where the second equality follows
by injecting~\eqref{eq:mm1-policy},
and the value of $Z(\theta)$ follows by normalization.
We recognize \eqref{eq:p-pf} from \Cref{ass:stat},
with $n = d = k + 1$,
$\Phi(s) = (\frac\lambda\mu)^s$ for each $s \in \cS$,
$x_i(s) = \indicator{s \ge i + 1}$ for each $i \in \{0, 1, \ldots, k-1\}$
and $x_k(s) = \max(s - k, 0)$
for each $s \in \cS$,
and
$\rho_i(\theta) = \pi_k(\accept | i, \theta)$
for each
$
i
\in
\{0, 1, \ldots, k\}
$.
The function $\rho$ defined in this way is differentiable.
\Cref{ass:stat} is therefore satisfied,
as the distribution of the system seen at arrival times
is also~\eqref{eq:mm1-stationary} according to
the PASTA property~\citep{W82}.
For each $s \in \bN$ and $a \in \{\accept, \reject\}$,
$\nabla \log \pi_k(a | s, \theta)$
is the $(k+1)$-dimensional column vector
with value $\indicator{a \, = \, \accept} - \pi_k(\accept | i, \theta)$
in component $i = \min(s, k)$
and zero elsewhere, and
$\rD \log \rho(\theta)$ is the $(k+1)$-dimensional diagonal matrix
with diagonal coefficient
$1 - \pi_k(\accept | i, \theta)$ in position $i$,
for each $i \in \{0, 1, \ldots, k\}$.
This can be used to verify that
\Cref{ass:regularity-score-function} is satisfied.

\subsubsection{Objective Function}

The objective function is
$J(\theta) = \gamma \prb{A = \accept}
- \frac\eta\lambda \esp{S}$, where
\begin{align*}
	\prb{A = \accept}
	&= \sum_{i = 0}^{k-1} p(i | \theta) \pi_k(\accept | i, \theta)
	+ \left( 1 - \sum_{i = 0}^{k-1} p(i | \theta) \right) \pi_k(\accept | k, \theta), \\
	\esp{S}
	&= \sum_{i = 0}^{k - 1} i p(i | \theta)
	+ \frac{p(k | \theta)}{1 - \rho_k(\theta)}
	\left( k + \frac{\frac\lambda\mu \rho_k(\theta)}{1 - \frac\lambda\mu \rho_k(\theta)} \right), \\
	Z(\theta)
	&= \sum_{s = 0}^{k-1}
	\left( \prod_{i = 0}^{s-1} \frac\lambda\mu \rho_i(\theta) \right)
	+ \left( \prod_{i = 0}^{k-1} \frac\lambda\mu \rho_i(\theta) \right)
	\frac1{1 - \frac\lambda\mu \rho_k(\theta)},
\end{align*}
with the convention that empty sums are equal to zero
and empty products are equal to one.
All calculations remain valid
in the limit as $\pi_k(\accept | i, \theta) \to 1$
for some $i \in \{0, 1, \ldots, k\}$
(corresponding to $\theta_i \to +\infty$).
In the limit as $\pi_k(\accept | i, \theta) \to 0$
for some $i \in \{0, 1, \ldots, k\}$,
we can study the restriction of the birth-and-death process
to the state space $\{0, 1, \ldots, \uc\}$,
where $\uc = \min\{i \in \{0, 1, \ldots, k\}: \pi_i(\theta) = 0\}$.

\subsubsection{Assumptions of \Cref{sec:convergence}}

For any closed set $U \subset \Omega$, it can be shown that there exists a Lyapunov function $\mathcal{L}$ uniformly over $\theta \in U$ such that $\mathcal{L}(s,a) = \mathcal{L}(s) = \exp(cs) $ for some $c>0$, depending on $U$ and the model parameters. 
We look at the equivalent geometric ergodicity condition for continious Markov chains.
If $\theta \in U$ we have $\mu  - \lambda \pi_k(\theta) > \delta(U) >0$.
Then, for $s > k+1$, the generator of the Markov process $Q_{\theta}$ satisfies
\begin{align}
	Q_{\theta}\mathcal{L}(s) &= q_{\theta}(s-1|s) \mathcal{L}(s-1) + q_{\theta}(s+1|s) \mathcal{L}(s+1) + q_{\theta}(s|s) \mathcal{L}(s) \nonumber \\
	& = \mu  \exp(c (s-1)) + \lambda \pi_k(\theta)\exp(c (s+1))- (\mu + \lambda \pi_k(\theta)) \exp(c s)\nonumber\\
	& = \exp(c s) (\mu \exp(-c)  + \lambda \pi_k(\theta)) \exp(c) +  (\mu  + \lambda \pi_k(\theta)) \\
	& = - c \Bigl( \mu -  \lambda \pi_k(\theta) + O(c) \Bigr) \mathcal{L}(s).
	\label{eqn:MM1_Lyapunov_check}
\end{align}
For $c$ small enough, from \eqref{eqn:MM1_Lyapunov_check} we have that $Q_{\theta}\mathcal{L}(s) \leq -c \delta(U)/2 \mathcal{L}(s)$, so that for any $\theta \in U$ the Markov chain corresponding to the policy of $\theta$ is geometrically ergodic.
Hence, Assumptions~\ref{ass:geometric-ergodicity}, \ref{ass:regularity-score-function} and \ref{ass:growth_condition} are satisfied. In general, Assumption~\ref{ass:nondegenerate-maxima} does not hold for this example because maxima occur only as $|\theta| \to \infty$. As suggested by Proposition~\ref{prop:regularization}, by adding a small regularization term, we can guarantee Assumption~\ref{ass:nondegenerate-maxima} while simultaneously ensuring that the maximizer is bounded. 
In practice, using a regularization term can additionally present some benefits such as avoiding vanishing gradients and saddle points.

\subsubsection{Effective State Space}

The effective state space if captured in the term \eqref{eqn:L_star}.
Similarly to the continious-time Markov chain example from \eqref{eqn:MM1_Lyapunov_check}, we have that the Lyapunov function is $\mathcal{L}(s,a) = \exp(cs)$ for some $c > 0$ small enough.
Let $\rho$ be such that $1 > \rho \geq \lambda \pi_k(\theta)/\mu$ for any $\theta \in V$.
Let $c$ be small enough such that 
\begin{equation}
	\frac{\rho}{1 + \rho} \exp(c) + \frac{1}{1 + \rho} \exp(-c) \leq \Bigl( 1 - c\frac{1-\rho}{2(1+\rho)} \Bigr),
\end{equation}
Then, for any $s \geq k$,
\begin{equation}
	P_{\theta} \mathcal{L} (s) \leq \lambda  \mathcal{L}(s) + b,
\end{equation}
where 
\begin{align}
	\lambda &=  1 - c\frac{1-\rho}{2(1+\rho)}, \\
	b & = \exp(c_1 k) \quad \quad \text{ for some } \quad \quad c_1 > 0.
\end{align}
If we let $s_0 \in [k]$, then from the definition of $\mathcal{L}$ in \eqref{eqn:L_star} we will have that
\begin{equation}
	\mathcal{L}^{\star} = O\Bigl( \exp(ck) \Bigr).
	\label{eqn:MM1_example_Lyapunov}
\end{equation}

We note that in this case $\mathcal{L}^{\star} \sim \mathrm{Volume}(U) \sim \exp(\dim(\theta))$.
Hence, it encodes the volume of the optimization space where the algorithm operates.
We remark that the Lyapunov function encodes geometric ergodicity and allows to tackle \emph{any} type of rewards, as long as they satisfy $|r|_{\mathcal{L}} < \infty$.
Thus, for a specific reward $r$ better bounds could be attained.

We remark that geometric ergodicity is not equivalent to Foster stability, which is guaranteed by $P_{\theta} \mathcal{L} (s) \leq \mathcal{L} (s) - \delta$ for some $\delta > 0$ and all $s \in S$.
Foster stability implies stability (i.e., positive recurrence) of the Markov chain, and the existence of $\E[\mathcal{L}(s)]$.
In this case, a Lyapunov function of the type $\mathcal{L}(s) \simeq s^2$ would suffice to show positive recurrence.

\subsection{Load-Balancing System} \label{app:lb}

We now consider the load-balancing example of \Cref{sec:Load-balancing-system}.
Recall that jobs arrive according to a Poisson process with rate~$\lambda > 0$,
there are $n$ servers at which service times are distributed
exponentially with rates $\mu_1, \mu_2, \ldots, \mu_n$, respectively,
and the system can contain at most $c$ jobs, for some $c \in \bN_+$.
The goal is to choose a static random policy
that maximizes the admission probability.
We first verify that the system satisfies
\Cref{ass:markov,ass:reward,ass:stat},
then we provide an algorithm to evaluate the objective function
when the parameters are known;
this is used in particular for performance comparison
with the optimal policy in the numerical results.
Lastly, we discuss the assumptions of \Cref{sec:convergence}.
Throughout this section,
we assume that we apply
the policy~$\pi(\theta)$ defined by~\eqref{eq:lb-policy}
for some parameter~$\theta \in \bR^n$.

\subsubsection{Product-Form Stationary Distribution}

That \Cref{ass:markov} is satisfied
follows from the facts that the rates and probabilities
$\lambda, \mu_1, \mu_2, \ldots, \mu_n,
\pi_1(\theta), \pi_2(\theta), \ldots, \pi_n(\theta)$
are positive
and that the state space~$\cS$ is finite.
\Cref{ass:reward} is satisfied because the state space is finite.
This system can be modeled either
as a loss Jackson network with $n$ queues
(one queue for each server in the load-balancing system)
or as a closed Jackson network with $n + 1$ queues
(one queue for each server in the system,
plus another queue signaling available positions in the system,
with service rate~$\lambda$).
Either way, we can verify (for instance by writing the balance equations) that
the stationary distribution of the continuous-time Markov chain that describes the evolution of the system state is given by:
\begin{align} \label{eq:lb-stationary}
	p(s | \theta)
	&= \frac1{Z(\theta)}
	\prod_{i = 1}^n \left( \frac{\lambda}{\mu_i} \pi_i(\theta) \right)^{s_i},
	\quad s = (s_1, s_2, \ldots, s_n) \in \cS,
\end{align}
where $Z(\theta)$ follows by normalization.
This is exactly \eqref{eq:p-pf} from \Cref{ass:stat},
with $n = d$, $\Omega = \bR^n$,
$\Phi(s) = \prod_{i = 1}^n (\frac\lambda{\mu_i})^{s_i}$ for each $s \in \cS$,
$x_i(s) = s_i$ for each $i \in \{1, 2, \ldots, n\}$
and $s \in \cS$, and
$\rho_i(\theta) = \pi_i(\theta)$
for each $i \in \{1, 2, \ldots, n\}$.
The function~$\rho$ defined in this way is differentiable.
\Cref{ass:stat} is therefore satisfied,
as the distribution of the system seen at arrival times
is also~\eqref{eq:lb-stationary} according to the PASTA property.
Besides the sufficient statistics~$x$,
the inputs of \Cref{algo:sage} are
$\nabla \log \pi(a | s, \theta) = \ind_a - \pi(\theta)$,
where $\ind_a$ is the $n$-dimensional vector
with one in component~$a$
and zero elsewhere
and $\pi(\theta)$ is the policy seen as a (column) vector, and
$\rD \log \rho(\theta) = \textrm{Id} - \ind \pi(\theta)^\intercal$,
where $\textrm{Id}$ is the $n$-dimensional identity matrix,
$\ind$ is the $n$-dimensional vector with all-one components,
and $\pi(\theta)^\intercal$ is the (row) vector obtained by transposing~$\pi(\theta)$.
This latter equation can be used to verify
\Cref{ass:regularity-score-function}.

\subsubsection{Objective Function}

When all parameters are known and the number of servers is not too large,
the normalizing constant $Z(\theta)$
and admission probability $J(\theta)$
can be computed efficiently
using a variant of Buzen's algorithm~\citep{B73}
for loss networks.
Define the array
$G = (G_{\uc, \un})_{\uc \in \{0, 1, \ldots, c\}, \un \in \{1, 2, \ldots, n\}}$
by
\begin{align*}
	G_{\uc, \un}
	&= \sum_{\substack{s \in \bN^{\un}: \\ |s| \le \uc}}
	\prod_{i = 1}^{\un} \left( \frac\lambda{\mu_i} \rho_i(\theta) \right)^{s_i},
	\quad \uc \in \{0, 1, \ldots, c\},
	\quad \un \in \{1, 2, \ldots, n\}.
\end{align*}
The dependency of $G$ on $\theta$ is left implicit to alleviate notation.
The normalizing constant and admission probability
are given by $Z(\theta) = G_{c, n}$
and $J(\theta) = G_{c-1, n} / G_{c, n}$,
respectively.
Defining the array~$G$ allows us to calculate these metrics
more efficiently than by direct calculation,
as we have $G_{0, \un} = 1$
for each $\un \in \{1, 2, \ldots, n\}$,
and
\begin{align*}
	G_{\uc, 1}
	&= 1 + {{\textstyle \frac\lambda{\mu_1}} \rho_1(\theta)} {G_{\uc-1, 1}},
	&& \uc \in \{1, 2, \ldots, c\}, \\
	G_{\uc, \un}
	&= G_{\uc, \un-1} + {{\textstyle \frac\lambda{\mu_{\un}}} \rho_{\un}(\theta)} G_{\uc-1, \un},
	&& \uc \in \{1, 2, \ldots, c\},
	\quad \un \in \{2, 3, \ldots n\}.
\end{align*}

\subsubsection{Assumptions of \Cref{sec:convergence}}

Assumptions~\ref{ass:geometric-ergodicity}, \ref{ass:regularity-score-function}, and \ref{ass:growth_condition} are automatically satisfied
because the state space is finite (with $|\mathcal{S}| = \binom{n+c}{c}$).
Verifying Assumption~\ref{ass:nondegenerate-maxima} is challenging since it requires computing $\mathrm{Hess}_{\theta^{\star}} J$ at the maximizer $\theta^{\star}$,
which depends in an implicit manner on the parameters of the system such as the arrival rate $\lambda$, service rates $\mu_1, \mu_2, \ldots, \mu_n$, and policy $\pi(\theta^{\star})$.
However, the nondegeneracy property of the Hessian for smooth functions is a property that is commonly stable in the following sense: if a function satisfies this property, then it will still be satisfied after any small-enough smooth perturbation. In particular, smooth functions with isolated nondegenerate critical points---also known as Morse functions---are dense and form an open subset in the space of smooth functions \cite[Section 1.2]{nicolaescu2011invitation}. Thus, unless the example is adversarial or presents symmetries, we can expect Assumption~\ref{ass:nondegenerate-maxima} to hold.

\subsection{Ising Model and Glauber Dynamics} \label{app:ising}

Lastly, we focus on the example of \Cref{sec:ising}.
We consider the Markov chain defined by
applying the policy~\eqref{eq:ising-policy}
parameterized by some vector $\theta \in \bR^3$:
starting from an arbitrary initial configuration,
at every time step,
a coordinate is chosen uniformly at random,
and the agent flips or not the spin at this coordinate
according to the policy.

\subsubsection{Product-Form Stationary Distribution}

The Markov chain is irreducible because it has a positive probability
of transitioning from any configuration to any other as follows:
at every step, choose a coordinate at which the two configurations differ
and flip the spin at this coordinate.
The Markov chain is positive recurrent because its state space is finite.
Hence, \Cref{ass:markov,ass:reward} are satisfied.
We now focus on proving \Cref{ass:stat}.

Our goal is to verify that the Markov chain that describes
the random evolution of the state admits
the stationary distribution~\eqref{eq:ising-stationary-distribution},
which we recall here:
\begin{align*}
	p(\sigma, v | \theta)
	&= \frac1{Z(\theta)}
	e^{\beta(\theta) J I(\theta) + \beta(\theta) \mu F(\sigma | \theta)},
	\quad s = (\sigma, v) \in \cS,
	\quad \theta \in \bR^3,
\end{align*}
Observe that $p(\sigma, v | \theta)$ is independent of~$v$,
hence we can let $q(\sigma | \theta) \triangleq p(\sigma, \cdot | \theta)$
for each $\sigma \in \Sigma$.
The key argument to prove that this is indeed the stationary distribution
consists of observing that the policy~\eqref{eq:ising-policy} satisfies,
for each $s = (\sigma, v) \in \cS$,
\begin{align} \label{eq:ising-policy-2}
	\pi(\flip | \sigma, v, \theta)
	&= \frac{q(\sigma_{-v} | \theta)}{q(\sigma | \theta) + q(\sigma_{-v} | \theta)},
	&
	\pi(\notflip | \sigma, v, \theta)
	&= \frac{q(\sigma | \theta)}{q(\sigma | \theta) + q(\sigma_{-v} | \theta)},
\end{align}
where $\sigma_{-v} \in \Sigma$ is the configuration obtained
by flipping the spin at~$v$ compared to~$\sigma$, that is,
$\sigma_{-v}(w) = \sigma(w)$ for each $w \in \cV \setminus \{v\}$
and $\sigma_{-v}(v) = - \sigma(v)$.

The balance equation
for a particular state $s = (\sigma, v) \in \cS$ writes
\begin{align*}
	p(\sigma, v | \theta)
	&= \sum_{w \in \cV} p(\sigma, w | \theta)
	\pi(\notflip | \sigma, w, \theta) \frac1{d_1 d_2}
	+ \sum_{w \in \cV} p(\sigma_{-w}, w | \theta)
	\pi(\flip | \sigma_{-w}, w, \theta) \frac1{d_1 d_2}.
\end{align*}
Dropping the dependency on~$\theta$ to simplify notation,
and injecting~\eqref{eq:ising-policy-2} into the right-hand side
of this balance equation,
we obtain successively
\begin{align*}
	&\sum_{w \in \cV} p(\sigma, w)
	\frac{q(\sigma)}{q(\sigma) + q(\sigma_{-w})} \frac1{d_1 d_2}
	+ \sum_{w \in \cV} p(\sigma_{-w}, w)
	\frac{q((\sigma_{-w})_{-w})}
	{q(\sigma_{-w}) + q((\sigma_{-w})_{-w})}
	\frac1{d_1 d_2} \\
	&\overset{(1)}{=} \sum_{w \in \cV}
	(p(\sigma, w) + p(\sigma_{-w}, w))
	\frac{q(\sigma)}{q(\sigma) + q(\sigma_{-w})}
	\frac1{d_1 d_2}
	\overset{(2)}{=} q(\sigma) \sum_{w \in \cV} \frac1{d_1 d_2}
	= q(\sigma)
	\overset{(2)}{=} p(\sigma, v),
\end{align*}
where (1) follows by observing that $(\sigma_{-w})_{-w} = \sigma$
and~(2) by recalling that $q(\sigma) = p(\sigma, w)$
for each $(\sigma, v) \in \cS$.
This proves that the distribution~\eqref{eq:ising-stationary-distribution}
is indeed the stationary distribution of the Markov chain
that describes the evolution of the state
under the policy~\eqref{eq:ising-policy}.

Besides the sufficient statistics~$x$,
the inputs of \Cref{algo:sage} are given,
for each $\theta \in \bR^3$
and $s = (\sigma, v) \in \cS$,
by
\begin{align*}
	\rD \log \rho(\theta)
	&= \left[ \begin{matrix}
		\beta'(\theta) J & 0 & 0 \\
		\beta'(\theta) \mu \theta_\tleft(\theta) & \beta(\theta) \mu {h_{\tleft}}'(\theta) & 0 \\
		\beta'(\theta) \mu \theta_{\tright}(\theta) & 0 & \beta(\theta) \mu {h_{\tright}}'(\theta)
	\end{matrix} \right],
	\\
	\nabla \log \pi(a | \sigma, v, \theta)
	&= (\indicator{a = \notflip} - \pi(\notflip | \sigma, v, \theta))
	\nabla \delta(\sigma, v | \theta),
	\\
	\nabla \delta(\sigma, v | \theta)
	&= 2 \left[ \begin{matrix}
		\beta'(\theta) \sigma(v) (
		J \sum_{\substack{w \in \cV: \, w \sim v}} \sigma(w)
		+ \mu h(v | \theta)
		) \\
		\beta(\theta) \sigma(v) \mu {h_{\tleft}}'(\theta) \indicator{v_2 \le d_2 / 2} \\
		\beta(\theta) \sigma(v) \mu {h_{\tright}}'(\theta) \indicator{v_2 > d_2 / 2}
	\end{matrix} \right],
\end{align*}
where $\beta'(\theta)$ (resp.\ $h_\tleft'(\theta)$, $h_\tright'(\theta)$) is to be understood as the partial derivative of $\beta$ (resp.\ $h_\tleft$, $h_\tright$) with respect to $\theta_1$ (resp.\ $\theta_2$, $\theta_3$).

\subsubsection{Assumptions of Section~\ref{sec:convergence}}

In this example, note that since the state space, albeit large, is finite Assumption~\ref{ass:geometric-ergodicity} is satisfied.
By inspecting the load function and its derivatives we can also verify that Assumptions~\ref{ass:regularity-score-function}--\ref{ass:growth_condition} hold.
Finally, Assumption~\ref{ass:nondegenerate-maxima} does not hold as the optima in the parameter space is unbounded. 
Note that as we see in the simulations this assumption is not required in practice for the policy to converge.

\section{Proof of Theorem~\ref{prop:main_prop_convergence in probability_noncompact_case}}
\label{sec:appendix_proof_of_theorems}

In this section,
we prove Theorem~\ref{prop:main_prop_convergence in probability_noncompact_case}.
The outline of this proof is given in \Cref{sec:Proof-outlines}.

\subsection{Preliminaries}
\label{sec:appendix_preliminaries}

We are going to use concentration inequalities for Markov chains.
Such results are common in the literature (for example, see \citealt{karimi2019non}), and will be required to get a concentration bound of the plug-in estimators from \eqref{eqn:SGA-recursion-for-the-parameter-vector}.

Denote by $B_{\epsilon}(\theta)$ the open ball of radius $\epsilon$ centered at $\theta \in \Omega \subseteq \R^{n}$ and $\mathcal{Y} = \mathcal{S} \times \mathcal{A}$.
Given a function $q: \mathcal{Y} \to \R$ and the Lyapunov function $\mathcal{L}: \mathcal{Y} \to [1, \infty)$ from Assumption~\ref{ass:geometric-ergodicity}, define
\begin{equation}
	|q |_{\mathcal{L}}
	=
	\sup_{y \in \mathcal{Y}} \frac{|q(y)|}{\mathcal{L}(y)}
	.
	\label{eqn:Seminorm_wrt_q}
\end{equation}
Given a signed measure $\nu$, we also define the seminorm
\begin{equation}
	| \nu |_{\mathcal{L}}
	=
	\sup_{|q |_{\mathcal{L}} \leq 1} |\nu[q]| = \sup_{|q |_{\mathcal{L}} \leq 1} \Bigl| \int q(y) \nu(dy) \Bigr|
	.
	\label{eqn:Seminorm_wrt_mu}
\end{equation}
\Crefrange{eqn:Seminorm_wrt_q}{eqn:Seminorm_wrt_mu} imply that
\begin{equation}
	|\nu[q]|
	\leq
	| \nu |_{\mathcal{L}} | q |_{\mathcal{L}}.
\end{equation}

Note that we defined $| \cdot |_{\mathcal{L}}$ for a unidimensional function.
Given instead $m$ functions $q_i : \mathcal{Y} \to \R$, for the higher-dimensional function $q : \mathcal{Y} \to \R^m$ that satisfies for all $y \in \mathcal{Y}$, $q(y) = (q_1(y), \ldots, q_l(y))$, we define $| q |_{\mathcal{L}} = \sqrt{ \sum_{i=1}^l |q_i|_{\mathcal{L}}^2}$.

The following lemma yields the concentration inequalities required:

\begin{lemma}
	\label{lemma:bounded_expectation_Markov}
	
	Let $\process{ Y_n }{n \geq 1}$ be a geometrically ergodic Markov chain with invariant distribution $p$ and transition matrix $P(\,\cdot\,,\,\cdot\,)$.
	Let the Lyapunov function be $\mathcal{L}:\mathcal{Y} \to \R$.
	From geometric ergodicity, there exists $C > 0$ and $\lambda \in (0,1)$ such that for any $y \in \mathcal{Y}$,
	\begin{equation}
		\big| P^{m}(\,\cdot\,| y) - p(\cdot) \big|_{\mathcal{L}} \leq C \lambda^{m}.
	\end{equation}
	Let $\mathcal{F} = \sigma(Y_1)$ be the $\sigma$-algebra of $Y_1$.
	Let $q: \mathcal{Y} \to \R^m$ be a measurable function such that $|q|_{\mathcal{L}} < \infty$.
	For a finite trajectory $Y_1, \ldots, Y_{M}$ of the Markov chain, we define the empirical estimator for $p[q]$ as
	\begin{equation}
		\hat{p}_{M}[q] = \frac{1}{M} \sum_{i=1}^{M} q(Y_{i}).
	\end{equation}
	With these assumptions, there exists $C^{\prime}$ depending on $C$ and $\lambda$ such that
	\begin{equation}
		\Bigl| \expectationBig{ p[q] - \hat{p}_{M}[q] \Big| \mathcal{F}
		} \Bigr|
		\leq \frac{C^{\prime}| q |_{\mathcal{L}}}{M}\mathcal{L}(Y_{1}),
		\label{eqn:concentration_biased_gradient}
	\end{equation}
	and for $l \in \{1,2,4\}$,
	\begin{equation}
		\expectationBig{\big| p[q] - \hat{p}_{M}[q] \big|^l
			\Big| \mathcal{F}}
		\leq \frac{C^{\prime}| q |_{\mathcal{L}}^{l}}{M^{l/2}}\mathcal{L}^l(Y_{1})
		.
		\label{eqn:concentration_biased_gradient2}
	\end{equation}
\end{lemma}

\begin{proof}
	We refer to \cite[Proposition~12]{fort2003convergence} for a proof of \eqref{eqn:concentration_biased_gradient2}.
	What remains is to prove \eqref{eqn:concentration_biased_gradient}.
	
	Observe that for $y \in \mathcal{Y}$, $P(y) = P(\,\cdot\,|y)$ is a distribution over $\mathcal{Y}$.
	Conditional on $\mathcal{F}$, there exists $C>0$ such that
	\begin{align}
		\Bigl|
		\expectationBig{ \frac{1}{M} \sum_{i=1}^{M} q(Y_{i}) - p[q] \Big| \mathcal{F} }
		\Bigr|
		&
		\leq
		\frac{1}{M}
		\sum_{i=1}^{M}
		\bigl| P^{i}(Y_{1})[q] - p[q] \bigr|
		=
		\frac{1}{M}
		\sum_{i=1}^{M}
		\bigl| \bigl( P^{i}(Y_1) - p \bigr)[q] \bigr|
		\nonumber \\ &
		\leq
		\frac{1}{M}
		\sum_{i=1}^{M}
		\bigl| P^{i}(\,\cdot\,| Y_1) - p(\cdot) \bigr|_{\mathcal{L}} |q|_{\mathcal{L}} \mathcal{L}(Y_1)
		\leq
		\frac{| q |_{\mathcal{L}}}{M}
		\sum_{i=1}^{M}
		C \lambda^{i}\mathcal{L}(Y_1)
		\nonumber\\ &
		\leq
		\frac{C |q|_{\mathcal{L}}}{M(1-\lambda)} \mathcal{L}(Y_1)
		.
	\end{align}
	This concludes the proof.
\end{proof}

In epoch $m$, the Markov chain $\{S_{t}\}_{t \in [t_m, t_{m+1}]}$ with control parameter $\Theta_m$ has a Lyapunov function $\mathcal{L}_{v}$.
Intuitively, as a consequence of Assumption~\ref{ass:geometric-ergodicity}, we can show that the process does not drift to infinity on the event $\mathcal{B}_m$ (\emph{despite} the changing control parameter $\Theta_m$).

Specifically, for $m > 0$, let $\process{S_{t}}{i \in [t_m, t_{m+1}]}$ be the Markov chain trajectory with transition probabilities $P(\Theta_m)$, where $\Theta_m$ is given by the updates in \eqref{eqn:update_SGA_intro} and \eqref{eqn:SGA-recursion-for-the-parameter-vector} and initial state $S_0 \in \mathcal{S}$.
Recall that $\mathcal{B}_{m}$ is defined in \eqref{eqn:Definition_of_Bm}.
We can then prove the following:

\begin{lemma}
	\label{lemma:expectation_L_is_bounded}
	
	Suppose Assumption~\ref{ass:geometric-ergodicity} holds.
	There exists $D < \infty$ such that for $m > 0$, $\E\bigl[ \mathcal{L}_v(S_{t_{m+1}}) \allowbreak \indicator{ \mathcal{B}_{m} }\bigr] < D$.
	In particular, using the notation of Assumption~\ref{ass:geometric-ergodicity} we may choose $D = \mathcal{L}^{\star}$ as defined in \eqref{eqn:L_star}.
\end{lemma}

\begin{proof}
	We will give an inductive argument.
	A similar argument can be found in the paper of \cite{atchade2017perturbed}.
	
	First, observe that for $m=0$, $S_0$ is fixed.
	Thus, there exists a $D$ such that $\mathcal{L}_{v}(S_0) \leq D$.
	
	Next, assume that $\E[\mathcal{L}_v(S_{t_m}) \indicator{ \mathcal{B}_{m-1} }] \leq D$.
	On the event $\mathcal{B}_{m}$, Assumption~\ref{ass:geometric-ergodicity} holds since $\Theta_1, \ldots, \Theta_{m-1}, \Theta_m \in V_{\mathfrak{r}, \delta}(\theta^{\star}) \subset U$.
	Thus, on the event $\mathcal{B}_{m}$, and when additionally conditioning on $S_{t_{m+1} - 1}$ and $\Theta_m$, the following holds true:
	\begin{align}
		\E\bigl[ \mathcal{L}_v(S_{t_{m+1}}) \indicator{ \mathcal{B}_{m} } \bigr]
		&
		\leq
		\E\bigl[ \E[ \mathcal{L}_v(S_{t_{m+1}}) \indicator{ \mathcal{B}_{m} }|S_{t_{m+1}-1}] \bigr]
		\nonumber \\ &
		=
		\E\bigl[\indicator{ \mathcal{B}_{m} } P_{\Theta_{m}} \mathcal{L}_v(S_{t_{m+1}-1})\bigr]
		\label{eqn:upper_bound_W}
		\\ &
		\leq
		\E\bigl[ \indicator{ \mathcal{B}_{m} }[\lambda \mathcal{L}_v(S_{t_{m+1}-1}) + b] \bigr]
		.
		\nonumber
	\end{align}
	The last step followed from Assumption~\ref{ass:geometric-ergodicity}.
	
	Observe finally that the bound in \eqref{eqn:upper_bound_W} can be iterated by conditioning on $S_{t_{m+1}-2}$; so on and so forth.
	After $t_{m+1} - t_m$ iterations, one obtains
	\begin{equation}
		\E \bigl[\mathcal{L}_v(S_{t_{m+1}}) \indicator{ \mathcal{B}_{m} } \bigr]
		\leq
		\lambda
		\E \bigl[ \mathcal{L}_v(S_{t_m}) \indicator{ \mathcal{B}_{m} } \bigr]
		+
		\frac{b}{1-\lambda}
		.
	\end{equation}
	Noting that $\indicator{ \mathcal{B}_{m} } \leq \indicator{ \mathcal{B}_{m-1} }$, the claim follows by induction if we choose $D$ large enough such that $\lambda D + b/(1-\lambda)\leq D$, that is $D \geq \mathcal{L}^{\star}$.
\end{proof}

\subsection{Proof of Theorem~\ref{prop:main_prop_convergence in probability_noncompact_case}}
\label{sec:Proof_Theorem2}

To prove \Cref{prop:main_prop_convergence in probability_noncompact_case}, we more--or--less follow the arguments of \cite[Theorem~25]{fehrman2020convergence}.
Modifications are however required because we consider a Markovian setting instead.
Specifically, we rely on the bounds in  Lemmas~\ref{lem:conditional_convergence_in_basin} to \ref{lem:probability_not_leaving_basin} instead of the bounds in \cite[Proposition~20, Proposition~21, Proposition~24]{fehrman2020convergence}, respectively.

Let us begin by bounding
\begin{equation}
	\probability{ J^\star - J(\Theta_m) > \epsilon | \mathcal{B}_0}
	.
\end{equation}
Here, $\mathcal{B}_0 = \{ \Theta_0 \in V_{\mathfrak{r}, \delta}(\theta^{\star}) \}$---recall \eqref{eqn:Definition_of_Bm}. \refTheorem{prop:main_prop_convergence in probability_noncompact_case} assumes that we initialize in a set $V$ which we will specify later but satisfies $V \subset V_{\mathfrak{r}, \delta}(\theta^{\star})$. Since we can initialize $\Theta_0$ with positive probability in $V$, we have that $\probability{\mathcal{B}_0} \geq \probability{\Theta_0 \in V} > 1/c > 0$ for some $c>0$. Thus, we will focus on finding an upper bound of
\begin{equation}
	\probability{ J^\star - J(\Theta_m) > \epsilon | \mathcal{B}_0 } \leq c \probability{ \{J^\star - J(\Theta_m) > \epsilon\} \cap \mathcal{B}_0 }.
	\label{eqn:Upper_bound_conditional_probability_total_probability}
\end{equation}

Denote the orthogonal projection of $\Theta_m$ onto $\mathcal{M} \cap U$ by $\tilde{\Theta}_m = \mathfrak{p}(\Theta_m)$.
We can relate the objective gap $J^{\star} - J(\Theta_m)$ to the distance $D_m := \mathrm{dist}( \Theta_m, \mathcal{M} \cap U )$ as follows.
Since $J$ is twice continuously differentiable with maximum $J^{\star}$ attained at $\mathcal{M} \cap U$, the function $J(\theta)$ with $\theta \in V_{\mathfrak{r}, \delta}(\theta^{\star})$ is locally Lipschitz with constant $\mathfrak{l}_{\mathfrak{r}, \delta}(\Theta^{\star}) > 0$.
On the event $\mathcal{B}_m$, we have $\Theta_m \in V_{\mathfrak{r}, \delta}(\theta^{\star})$ and therefore we have the inequality
\begin{equation}
	J^{\star} - J(\Theta_m)
	=
	J(\tilde{\Theta}_m) - J(\Theta_m)
	\leq
	\mathfrak{l}_{\mathfrak{r}, \delta}(\theta^{\star})
	\bigl| \tilde{\Theta}_m - \Theta_m \bigr|
	=
	\mathfrak{l}_{\mathfrak{r}, \delta}(\theta^{\star})
	D_m
	.
	\label{eqn:Relation_between_Jstar_and_Dm}
\end{equation}
Consequently, we have the bound
\begin{equation}
	\probability{ \{J^\star - J(\Theta_m) > \epsilon \} \cap \mathcal{B}_m }
	\leq
	\probabilityBig{
		\Bigl\{
		D_m \geq \frac{ \epsilon }{ \mathfrak{l}_{\mathfrak{r}, \delta}(\theta^{\star}) }
		\Bigr\}
		\cap
		\mathcal{B}_m
	}.
	\label{eqn:Probability_bound_between_the_objective_gap_and_the_manifold_distance}
\end{equation}
If we define $\epsilon' = \epsilon /\mathfrak{l}_{\mathfrak{r}, \delta}(\theta^{\star})$, the right-hand side of \eqref{eqn:Probability_bound_between_the_objective_gap_and_the_manifold_distance} can also be written as
\begin{align}
	\probability{ \{ D_m \geq \epsilon' \}  \cap \mathcal{B}_m }
	& =
	\expectation{ \indicator{ D_m \geq \epsilon' } \indicator{ \mathcal{B}_m } }
	\nonumber \\ &
	=
	\expectation{ \indicator{ D_m \indicator{B_m} \geq \epsilon' } }
	=
	\probability{ D_m \indicator{B_m} \geq \epsilon' }
	\label{eqn:Moving_indicator_Bm_into_the_probability_function}
\end{align}
by the positivity of $D_m$.

Next, we use (i) the law of total probability noting that $\mathcal{B}_m \subset \mathcal{B}_0$, (ii) the bound \eqref{eqn:Probability_bound_between_the_objective_gap_and_the_manifold_distance} and the inequality $\probability{A \cap B} \leq \probability{A}$ for any two events $A,B$, and finally, (iii) the equality \eqref{eqn:Moving_indicator_Bm_into_the_probability_function}. We obtain
\begin{align}
	\probability{ \{ J^\star - J(\Theta_m) > \epsilon \} \cap \mathcal{B}_0 }
	&
	\eqcom{i}
	\leq
	\probability{ \{J^\star - J(\Theta_m)) > \epsilon \} \cap \mathcal{B}_m}
	+
	\probability{ \{J^\star - J(\Theta_m)) > \epsilon \} \cap \overline{\mathcal{B}_m} } \nonumber \\
	&
	\eqcom{ii}
	\leq
	\probability{ \{ D_m \geq \epsilon' \}  \cap \mathcal{B}_m }
	+
	\probability{ \overline{\mathcal{B}_{m}} } \nonumber \\
	&
	\eqcom{iii}
	\leq
	\probability{ D_m \indicator{\mathcal{B}_{m}} \geq \epsilon'}
	+
	\probability{\overline{\mathcal{B}_{m}}} \nonumber \\
	&
	\leq
	\probability{ D_m \indicator{\mathcal{B}_{m-1}} \geq \epsilon'}
	+
	\probability{\overline{\mathcal{B}_{m}}}
	=
	\textnormal{Term I}
	+
	\textnormal{Term II}
	.
	\label{eqn:proof_theorem2_1}
\end{align}

Term I can be bounded by using Markov'`s inequality and Lemma~\ref{lem:conditional_convergence_in_basin}.
This shows that
\begin{equation}
	\textnormal{Term I}
	\leq
	c \epsilon'^{-2} \mathcal{L}^{\star} m^{-\sigma - \kappa}
	.
	\label{eqn:proof_theorem2_3}
\end{equation}

Term II can be bounded by Lemma~\ref{lem:probability_not_leaving_basin}.
Specifically, one finds that there exists a constant $c>0$ such that, if $\Theta_0 \in V_{\mathfrak{r}/2, \delta}(\Theta^{\star})$,
\begin{equation}
	\textnormal{Term II}
	\leq
	1 - \mathrm{exp}\Bigl(-\frac{c\alpha^2}{\delta^2\ell}\Bigr)
	+
	c\delta^{-2}\ell^{-1}m^{1-\sigma-\kappa} + c\alpha\frac{(m^{1-3/2\sigma - \kappa/2} + \ell^{-1/2}m^{1-5\sigma/8 - \kappa/2})}{(\mathfrak{r}/2 - 2\delta)_{+}}
	.\label{eqn:lemma_concentration_1}
\end{equation}
Note next that for any $\alpha \in (0, \alpha_0]$ and $c>0$ there exists $\delta_0$ such that for any $\delta \in (0, \delta_0]$ there exists $\ell_0$ such that if $\ell \in [\ell_0, \infty)$ there exists a constant $c^{\prime}>0$ such that we have the inequality
$
1 - \exp{ ( - {c\alpha^2} / {\delta^2 \ell} ) }
\leq
c^{\prime} {\alpha^2} / {\delta^2 \ell}
$. We can substitute this bound in \eqref{eqn:lemma_concentration_1} to yield
\begin{equation}
	\textnormal{Term II}
	\leq
	c^{\prime} \frac{\alpha^2}{\delta^2 \ell}
	+
	\textnormal{\emph{idem}}
	.
	\label{eqn:lemma_concentration_2}
\end{equation}

Bounding
\eqref{eqn:proof_theorem2_1} by the sum of \eqref{eqn:proof_theorem2_3} and \eqref{eqn:lemma_concentration_2}, and substituting the bound in \eqref{eqn:Upper_bound_conditional_probability_total_probability} reveals that there exists a constant $c'' > 0$ such that if $\Theta_0 \in V_{\mathfrak{r}/2, \delta}(\theta^{\star})$ then
\begin{align}
	\probability{ J^\star - J(\Theta_m) > \epsilon | \mathcal{B}_0 }
	&
	\leq
	c'' ( \epsilon' )^{-2} \mathcal{L}^{\star} m^{-\sigma - \kappa} + c'' \alpha^2 \delta^{-2} \ell^{-1} + c''\delta^{-2} \ell^{-1}m^{1-\sigma-\kappa}
	\nonumber \\ &
	\phantom{\leq}
	+
	c'' \alpha \frac{(m^{1-3/2\sigma - \kappa/2}
		+ \ell^{-1/2}m^{1-5\sigma/8 - \kappa/2})}{(\mathfrak{r}/2 - 2\delta)_{+}}
	.
	\label{eqn:Ultimate_bound_on_the_probability_of_Dm_large}
\end{align}
Note that the exponents of $m$ in \eqref{eqn:Ultimate_bound_on_the_probability_of_Dm_large} satisfy that since $\sigma \in (2/3,1)$, $1-3/2\sigma - \kappa/2 \leq - \kappa/2$ as well as $1-5\sigma/8 - \kappa/2 < 1- \sigma/2 - \kappa/2$. Finally, let the initialization set be $V = V_{\mathfrak{r}/2, \delta}(\theta^{\star})$. Note that since $\{\Theta_0 \in V\} \subset \mathcal{B}_0$ there exists a constant $c'''>0$ such that
\begin{equation}
	\probability{ J^\star - J(\Theta_m) > \epsilon | \Theta_0 \in V }
	\leq
	c'''\probability{ J^\star - J(\Theta_m) > \epsilon | \mathcal{B}_0 }
	.
	\label{eqn:Bound_B_0_and_V}
\end{equation}
Substituting the upper bound \eqref{eqn:Ultimate_bound_on_the_probability_of_Dm_large} in \eqref{eqn:Bound_B_0_and_V} concludes the proof.
\QuodEratDemonstrandum

\subsection{Proof of Lemma~\ref{lem:Bounded_norms_on_the_gradient_estimator}}
\label{sec:Proof_of_bounded_norms_gradient_estimator}

For simplicity, we will denote $t_{m+1} - t_m = T_m$, $X_t = x(S_t)$ throughout this proof.
We also temporarily omit the summation indices for the epoch.
We note that the policies defined in \eqref{eqn:defintion_softmax_policy} satisfy that for $(s, a) \in \mathcal{S} \times \mathcal{A}$,
\begin{align*}
	\bigl(
	\nabla \log \pi(a | s, \theta)
	\bigr)_{i, a'}
	& =
	\begin{cases}
		\indicator{a = a'} - \pi(a' | s, \theta) & \text{if } i = h(s), \\
		0                                        & \text{otherwise}.
	\end{cases}
\end{align*}
In particular, there exists $\Cl{asd} >0$ such that for any $(s, a) \in \mathcal{S} \times \mathcal{A}$, $|\nabla \log \pi(a | s, \theta)| < \Cr{asd}$.
The proof below, however, can also be extended to other policy classes.

\subsubsection{Proof of \eqref{eqn:lemma4_i}}

Observe that if the event $\mathcal{B}_m$ holds, that then the definitions in \eqref{eqn:SGA-recursion-for-the-parameter-vector} also imply that
\begin{align}
	\eta_m & = \nabla J(\Theta_m) - H_{m} =\nabla J(\Theta_m) - (\rD \log \rho(\Theta_m)^\intercal \overline{C}_m + \overline{E}_m)
	\nonumber \\ &
	=
	\nabla J(\Theta_m) - (\rD \log \rho(\Theta_m)^\intercal \frac1{T_{m+1}} \sum_{t = t_m}^{t_{m+1} - 1}
	\left( X_t - \overline{X}_m \right)
	r(S_t, A_t) \nonumber \\
	&\phantom{\rD \log \rho(\Theta_m)^\intercal aaaaaaaaaaaaaaaaaaaaa}+ \frac1{T_m} \sum_{t = t_m}^{t_{m+1} - 1}
	r(S_t, A_t) \nabla \log \pi(A_t | S_t, \Theta_m)
	\nonumber \\ &
	=
	\rD \log \rho(\Theta_m)^\intercal \Bigl( \cov{R, S} - \frac1{T_m} \sum_{t = t_m}^{t_{m+1} - 1}
	\left( X_t - \overline{X}_m \right)
	r(S_t, A_t) \Bigl)
	\nonumber \\ &
	\phantom{\rD \log \rho(\Theta_m)^\intercal }+ \Bigl( \esp{R \nabla \log \pi(A | S, \Theta_m)} - \frac1{T_m} \sum_{t = t_m}^{t_{m+1} - 1}
	r(S_t, A_t) \nabla \log \pi(A_t | S_t, \Theta_m)\Bigr)
	&
	\phantom{ = \nabla J(\Theta_m) - = \nabla J(\Theta_m) - \frac{1}{T_m}\sum_{i=1}^{T_{m+1}} g(\nu_{m}[q], X_{m+1}^{(i)})} - \frac{1}{T_{m+1}}\sum_{i=1}^{T_m} g(\hat{\nu}_m[q], X_{m+1}^{(i)})
	\nonumber \\ &
	=
	\rD \log \rho(\Theta_m)^\intercal \tilde{\eta}_m + \tilde{\zeta}_m
	.
	\label{eqn:proof_bounded_norms_gradient_1}
\end{align}
We will deal with the terms $\tilde{\eta}_m$ in and $\tilde{\zeta}_m$ in \eqref{eqn:proof_bounded_norms_gradient_1} one--by--one.

Let us first consider the \nth{1} term, $\tilde{\eta}_m$.
Define
\begin{gather}
	A
	=
	\expectation{(X - \E[X])R} - \frac{1}{T_m}\sum_{t} (X_{t} - \E[X]) r(S_t, A_t)
	,
	\nonumber \\
	B
	=
	\frac{1}{T_m}\Bigl(\sum_{t} r(S_t,A_t)\Bigr) (\E[X] -\bar{X}_m)
	,
	\label{eqn:proof_bounded_norms_gradient_3}
\end{gather}
and observe that
\begin{align}
	\tilde{\eta}_m
	=
	A + B
	.
	\label{eqn:proof_bounded_norms_gradient_2}
\end{align}

We look first at $A$ in \eqref{eqn:proof_bounded_norms_gradient_3}.
Recall that $\process{Y_t}{t>0} = \process{(S_t, A_t)}{t>0}$ is the chain of state-action pairs (see Section~\ref{sec:assumptions_pertaining_convergence}).
Define the function $g: \mathcal{S} \times \mathcal{A} \to \R^{n}$ as
\begin{equation}
	g(y)
	=
	g((s,a))
	=
	\bigl( x(s) - \E[x(s)] \bigr) r(y)
	.
\end{equation}
Then, we can rewrite
\begin{equation}
	A
	=
	\E[g(Y)] - \frac{1}{T_m}\sum_{t} g(Y_t)
	.
\end{equation}

We are now almost in position to apply Lemma~\ref{lemma:bounded_expectation_Markov} to $A$.
Observe next that the law of total expectation implies that
\begin{equation}
	\E[\eta_{m} \indicator{\mathcal{B}_{m}}| \mathcal{F}_m]
	=
	\sum_{a \in \mathcal{A}} \E[\eta_{m} \indicator{\mathcal{B}_{m}}| \mathcal{F}_m, A_{t_m}=a]\pi(a|S_m, \Theta_m)
	,
	\label{eqn:proof_lemma4_action_state_function}
\end{equation}
Without loss of generality, it therefore suffices to consider the case that we have one action $A_{t_m}=a \in \mathcal{A}$.
For the first term we have that there exists a constant $\Cl{1}>0$ such that
\begin{align}
	\bigr| \E[A \indicator{\mathcal{B}_{m}} | \mathcal{F}_m, A_{t_m}=a] \bigl| & = \Bigr| \E \bigr[ \E[g(Y)] - \frac{1}{T_m}\sum_{t} g(Y_t) \indicator{\mathcal{B}_{m}}| Y_0 = (S_{t_m}, A_{t_m}) \bigl] \Bigl|
	\nonumber \\ &
	\eqcom{Lemma~\ref{lemma:bounded_expectation_Markov}} = \frac{\Cr{1}|g|_{\mathcal{L}}}{T_m}\mathcal{L}((S_{t_m}, a))
	,
	\label{eqn:proof_bounded_norms_gradient_A}
\end{align}
where we can use that $|g|_{\mathcal{L}} < \infty$ due to Assumption~\ref{ass:growth_condition}.

For the term $B$ in \eqref{eqn:proof_bounded_norms_gradient_3}.
We can add and subtract again the following terms and obtain
\begin{align}
	B
	&
	= \frac{1}{T_m}\Bigl(\sum_{t} r(S_t,A_t)\Bigr) (\E[X] -\bar{X}_m) - \E[R](\E[X] -\bar{X}_m)
	+ \E[R](\E[X] -\bar{X}_m)
	\nonumber \\ &
	= C + D
	,
	\label{eqn:proof_bounded_norms_gradient_4}
\end{align}
where
\begin{align}
	C
	&=
	(\E[X] -\bar{X}_m)\Bigl(\frac{1}{T_m}\sum_{t} r(S_t,A_t) - \E[R]\Bigr),
	\nonumber \\
	D
	&=
	\E[R](\E[X] -\bar{X}_m)
	.
	\label{eqn:proof_bounded_norms_gradient_5}
\end{align}

For the term $D$ in \eqref{eqn:proof_bounded_norms_gradient_5} we can readily use the concentration of Lemma~\ref{lemma:bounded_expectation_Markov} to obtain
\begin{equation}
	\E\bigr[ \E[R](\E[X] -\bar{X}_m)\indicator{\mathcal{B}_{m}}| \mathcal{F}_m, A_{t_m} = a \bigl] \leq \E[R]\frac{|x(S)|_{\mathcal{L}}}{T_m} \mathcal{L}(S_{t_m}, a),
	\label{eqn:proof_bounded_norms_gradient_D}
\end{equation}
where we have $|x(S)|_{\mathcal{L}} < \infty$ from Assumption~\ref{ass:growth_condition} and $\E[R] < J^{\star}$.

For the term $C$, we use Cauchy--Schwartz together with Lemma~\ref{lem:Bounded_norms_on_the_gradient_estimator}. In particular, we have
\begin{align}
	\Bigl| \E \big[ (\E[X] -\bar{X}_m)&\Bigl(\frac{1}{T_m}\sum_{t} r(S_t,A_t) - \E[R]\Bigr)\indicator{\mathcal{B}_m} \big| \mathcal{F}_{m}, A_{t_m} = a \big] \Bigl| \leq \nonumber \\
	&\phantom{}  \bigr|\E \bigr[ |\E[X] -\bar{X}_m|^2\indicator{\mathcal{B}_m} \big| \mathcal{F}_{m}, A_{t_m} = a \bigl] \bigl|^{1/2} \times \nonumber \\
	&\phantom{aaaaaaaaaaaa} \Bigr|\E \Bigr[ \bigr| \frac{1}{T_m}\sum_{t} r(S_t,A_t) - \E[R] \bigl|^2 \indicator{\mathcal{B}_m}\Big| \mathcal{F}_{m}, A_{t_m} = a \Bigl] \Bigl|^{1/2}.
\end{align}
For both terms we can repeat the same argument to that in \eqref{eqn:proof_lemma4_action_state_function} together with Lemma~\ref{lemma:bounded_expectation_Markov} to show that
\begin{align}
	\bigr|\E \bigr[ |\E[X] -\bar{X}_m|^2\indicator{\mathcal{B}_m} \big| \mathcal{F}_{m}, A_{t_m} = a \bigl] \bigl|^{1/2} &\leq \C \frac{|X|_{\mathcal{L}}^{1/2}}{T_m^{1/2}} \mathcal{L}(S_{t_m}, a), \nonumber \\
	\Bigr|\E \Bigr[ \bigr| \frac{1}{T_m}\sum_{t} r(S_t,A_t) - \E[R] \bigl|^2 \indicator{\mathcal{B}_m}\Big| \mathcal{F}_{m}, A_{t_m} = a \Bigl] \Bigl|^{1/2} &\leq \C \frac{|R|_{\mathcal{L}}^{1/2}}{T_m^{1/2}} \mathcal{L}(S_{t_m},a).
	\label{eqn:proof_bounded_norms_gradient_6}
\end{align}
Therefore multiplying both bounds in \eqref{eqn:proof_bounded_norms_gradient_6} and using Assumption~\ref{ass:growth_condition} to bound the $\mathcal{L}$-norms, we obtain that there exists $\Cl{2} > 0$ such that
\begin{equation}
	|\E[C | \mathcal{F}_{m}, A_{t_m} = a]|  \leq \frac{\Cr{2}}{T_m} \mathcal{L}(S_{t_m}, a)^2.
	\label{eqn:proof_bounded_norms_gradient_C}
\end{equation}
Adding the bounds \eqref{eqn:proof_bounded_norms_gradient_A}, \eqref{eqn:proof_bounded_norms_gradient_C}, and \eqref{eqn:proof_bounded_norms_gradient_D} together we have now
\begin{align}
	|\E[\tilde{\eta}_{m} \indicator{\mathcal{B}_{m}}| \mathcal{F}_m, A_{t_m} = a]| &\leq \frac{\C}{T_m}\mathcal{L}^{2}(S_{t_m}, a).
\end{align}
Finally, averaging this bound over all actions in \eqref{eqn:proof_lemma4_action_state_function}, we obtain
\begin{equation}
	|\E[\tilde{\eta}_m \indicator{\mathcal{B}_{m}}| \mathcal{F}_m]| \leq \frac{\Cl{3}}{T_m}\bigl( \sum_{a} \mathcal{L}(S_{t_m},a)^{2} \pi(a|S_{t_m}, \Theta_m) \bigr) \leq \frac{\Cr{3}}{T_m} \mathcal{L}_{4}(S_{t_m})^{1/2}.
	\label{eqn:proof_bounded_norms_gradient_eta}
\end{equation}
Now we use Assumption~\ref{ass:regularity-score-function}. We can write
\begin{align}
	|\E[\nabla \log(\Theta_m) \tilde{\eta}_{m} \indicator{\mathcal{B}_m} | \mathcal{F}_{m}]| &= | \nabla \log(\Theta_m) \E[\tilde{\eta}_{m} \indicator{\mathcal{B}_m} | \mathcal{F}_{m}]| \nonumber \\
	&\leq C |\E[\tilde{\eta}_{m} \indicator{\mathcal{B}_m} | \mathcal{F}_{m}]| \nonumber \\
	&\leq \frac{\C}{T_m}\mathcal{L}(S_{t_m}).
\end{align}

Let us now consider the \nth{2} term, $\tilde{\zeta}_m$.
Define a function of $Y=(S,A)$ as
\begin{equation}
	g(Y)
	=
	r(Y)\nabla \log \pi(A|S,\theta)
	,
	\label{eqn:proof_bounded_norms_gradient_g_definition_zeta}
\end{equation}
so that
\begin{equation}
	\zeta_m
	= \E[g(Y)] - \frac{1}{T_m} \sum_{t} g(Y_t).
\end{equation}
By combining the argument of \eqref{eqn:proof_lemma4_action_state_function} with the fact that $|g(Y)|_{\mathcal{L}} < \infty$ by Assumption~\ref{ass:growth_condition}, we find that
\begin{equation}
	|\E[\tilde{\zeta}_{m} \indicator{\mathcal{B}_m} | \mathcal{F}_{m}]|
	\leq
	\frac{\C}{T_m}\mathcal{L}(S_{t_m}).
	\label{eqn:proof_bounded_norms_gradient_zeta}
\end{equation}
Adding \eqref{eqn:proof_bounded_norms_gradient_eta} and \eqref{eqn:proof_bounded_norms_gradient_zeta} together with their largest exponents yields
\begin{align}
	| \E[\eta_{m} \indicator{\mathcal{B}_{m}}| \mathcal{F}_m] |
	&
	\leq
	\frac{\Cl{4}}{T_m} \sum_{a} \mathcal{L}(S_{t_m}, a)^2 \pi(a|S_{t_m})
	\nonumber \\ &
	\leq
	\frac{\Cr{4}}{T_m} \Bigl( \sum_{a} \mathcal{L}(S_{t_m}, a)^4 \pi(a|S_{t_m}) \Bigr)^{1/2}
	\leq
	\frac{\Cr{4}}{T_m} \mathcal{L}_{4}(S_{t_m})^{1/2}
	.
\end{align}
This concludes the proof of \eqref{eqn:lemma4_i}.
\QuodEratDemonstrandum

\resetconstant
\subsubsection{Proof of \eqref{eqn:lemma4_ii}}

Note that by using the fact that for a vector-valued random variable $Z$ we have that $\E[|Z|^2] \geq \E[|Z|]^2$, the case for $p=1$ follows from the case $p=2$.

We focus on the case $p=2$.
By using the identity $(a + b) \leq 2 a^2 + b^2$, we estimate
\begin{align}
	&
	\E[ |\rD \log \rho(\Theta_m)^\intercal \tilde{\eta}_m + \tilde{\zeta}_m|^2 \indicator{\mathcal{B}_{m}} | \mathcal{F}_m]
	\nonumber \\  &
	\leq 2 (\E[ |\rD \log \rho(\Theta_m)^\intercal \tilde{\eta}_m|^2 \indicator{\mathcal{B}_{m}} | \mathcal{F}_m] + \E[ |\tilde{\zeta}_m|^2 \indicator{\mathcal{B}_{m}} | \mathcal{F}_m])
	\nonumber \\  &
	\eqcom{\ref{ass:regularity-score-function}} \leq 2 \C^2\E[ |\tilde{\eta}_m|^2 \indicator{\mathcal{B}_{m}} | \mathcal{F}_m] + 2\E[ |\tilde{\zeta}_m|^2 \indicator{\mathcal{B}_{m}} | \mathcal{F}_m].
	\label{eqn:proof_bounded_norms_gradient_7}
\end{align}
We again use the law of total expectation with the action set in \eqref{eqn:proof_lemma4_action_state_function} and condition on the action $A_m=a$.

For the term involving $\tilde{\zeta}_m$ in \eqref{eqn:proof_bounded_norms_gradient_7} we can again use the definition of $g$ in \eqref{eqn:proof_bounded_norms_gradient_g_definition_zeta}.
We bound
\begin{align}
	\E[ |\tilde{\zeta}_m|^2 \indicator{\mathcal{B}_{m}} | \mathcal{F}_m, A_{t_m} = a] & = \E \bigl[ |\E[g(Y) - \frac{1}{T_m} \sum_{t} g(Y)|^2| Y_0= (S_{t_m}, a) \bigr] \nonumber           \\
	& \eqcom{Lemma~\ref{lemma:bounded_expectation_Markov}} \leq \frac{\C}{T_m} \mathcal{L}(S_{t_m}, a)^2.
	\label{eqn:proof_bounded_norms_gradient_g_definition_zeta2}
\end{align}

For the term involving $\tilde{\eta}_m$ in \eqref{eqn:proof_bounded_norms_gradient_7}, we use the same definition for the terms $A, C$ and $D$ from \eqref{eqn:proof_bounded_norms_gradient_A}, \eqref{eqn:proof_bounded_norms_gradient_C} and \eqref{eqn:proof_bounded_norms_gradient_D} as in the proof of \eqref{eqn:lemma4_i}.
We have the bound
\begin{align}
	\E[ |\tilde{\eta}_m|^2 \indicator{\mathcal{B}_{m}} | \mathcal{F}_m, A_{t_m} = a] \leq 3(\E[ |A|^2 \indicator{\mathcal{B}_{m}}
	&
	| \mathcal{F}_m, A_{t_m} = a] + \E[ |C|^2 \indicator{\mathcal{B}_{m}} | \mathcal{F}_m, A_{t_m} = a]
	\nonumber \\ &
	+ \E[ |D|^2 \indicator{\mathcal{B}_{m}} | \mathcal{F}_m, A_{t_m} = a]).
	\label{eqn:proof_bounded_norms_gradient_8}
\end{align}
For the terms pertaining to $A$ and $D$ in \eqref{eqn:proof_bounded_norms_gradient_8} the same argument as those used for $\tilde{\zeta}_m$ in \eqref{eqn:proof_bounded_norms_gradient_g_definition_zeta} and \eqref{eqn:proof_bounded_norms_gradient_g_definition_zeta2} can be used to show that
\begin{align}
	\E[ |A|^2 \indicator{\mathcal{B}_{m}} | \mathcal{F}_m, A_{t_m} = a]
	&
	\leq \frac{\C}{T_m} \mathcal{L}(S_{t_m}, a)^2
	\nonumber \\
	\E[ |D|^2 \indicator{\mathcal{B}_{m}} | \mathcal{F}_m, A_{t_m} = a]
	&
	\leq \frac{\C}{T_m} \mathcal{L}(S_{t_m}, a)^2
	.
\end{align}

The only remaining term to bound in \eqref{eqn:proof_bounded_norms_gradient_8} is $C$.
We use again Cauchy--Schwartz's inequality
\begin{align}
	\E \Big[ \Bigl|(\E[X] -\bar{X}_m)
	&
	\Bigl(\frac{1}{T_m}\sum_{t} r(S_t,A_t) - \E[R]\Bigr)\indicator{\mathcal{B}_m} \Bigl|^4 \Big| \mathcal{F}_{m}, A_{t_m} = a \Big] \leq
	\nonumber \\ &
	\phantom{} \bigr|\E \bigr[ |\E[X] -\bar{X}_m|^2\indicator{\mathcal{B}_m} \big| \mathcal{F}_{m}, A_{t_m} = a \bigl] \bigl|^{1/2} \times
	\nonumber \\ &
	\phantom{aaaaaaaaaaaa} \Bigr|\E \Bigr[ \bigr| \frac{1}{T_m}\sum_{t} r(S_t,A_t) - \E[R] \bigl|^4 \indicator{\mathcal{B}_m}\Big| \mathcal{F}_{m}, A_{t_m} = a \Bigl] \Bigl|^{1/2}m,
\end{align}
and by Lemma~\ref{lemma:bounded_expectation_Markov} the following hold
\begin{align}
	\bigr|\E \bigr[ |\E[X] -\bar{X}_m|^4\indicator{\mathcal{B}_m} \big| \mathcal{F}_{m}, A_{t_m} = a \bigl] \bigl|^{1/2}
	&
	\leq \C \frac{|X|_{\mathcal{L}}^{1/2}}{T_m} \mathcal{L}(S_{t_m}, a)^2,
	\nonumber \\
	\Bigr|\E \Bigr[ \bigr| \frac{1}{T_m}\sum_{t} r(S_t,A_t) - \E[R] \bigl|^4 \indicator{\mathcal{B}_m}\Big| \mathcal{F}_{m}, A_{t_m} = a \Bigl] \Bigl|^{1/2} & \leq \C \frac{|R|_{\mathcal{L}}^{1/2}}{T_m} \mathcal{L}(S_{t_m},a)^2.
	\label{eqn:proof_bounded_norms_gradient_9}
\end{align}
The bound for $C$ thus becomes
\begin{equation}
	\E[|C|^2 | \mathcal{F}_{m}, A_{t_m} = a] \leq \frac{\C}{T_m^2} \mathcal{L}(S_{t_m}, a)^4.
	\label{eqn:proof_bounded_norms_gradient_Cprime}
\end{equation}
Upper bounding all terms by the largest exponents and adding over the different actions, we finally obtain
\begin{equation}
	\E[|\eta_{m}|^2 \indicator{\mathcal{B}_{m}}| \mathcal{F}_m]
	\leq
	\frac{\C}{T_m} \sum_{a} \mathcal{L}(S_{t_m}, a)^4 \pi(a|S_{t_m}), \Theta_m
	\leq
	\frac{\C}{T_m} \mathcal{L}_{4}(S_{t_m})
	.
\end{equation}
That is it.
\QuodEratDemonstrandum

\subsection{Proof of Lemma~\ref{lem:conditional_convergence_in_basin}}
\label{sec:Proof_of_conditional_convergence_basin}

We will again use the notation $t_{m+1} - t_m = T_m$ and without loss of generality we will assume that $T_m=\ell m^{\sigma/2 + \kappa}$ instead of $\lfloor \ell m^{\sigma/2 + \kappa} \rfloor$. This can be assumed since for $m \geq 1$ there exist constants $c_l, c_u > 0$ such that $c_l \ell m^{\sigma/2 + \kappa} \leq t_{m+1} - t_m \geq c_u \ell m^{\sigma/2 + \kappa}$.
The proof of Lemma~\ref{lem:conditional_convergence_in_basin} follows the same steps as in \cite[Proposition 20]{fehrman2020convergence}. However, we have to quickly diverge and adapt the estimates to the case that there the variance of $H_m$ depends on the states of a Markov chain.
From the assumptions, it can be shown that there is a unique differentiable orthogonal projection map $\mathfrak{p}: V_{\mathfrak{r}, \delta}(\theta^{\star}) \to \mathcal{M} \cap U$ from $V_{\mathfrak{r}, \delta}(\theta^{\star}) \cap U$ onto $V_{\mathfrak{r}, \delta}(\theta^{\star}) \cap \mathcal{M} \cap U$. The distance of $\Theta_m$ to the set of minima can then be upper bounded by the distance to the projection $\mathfrak{p}: V_{\mathfrak{r}, \delta}(\theta^{\star}) \to \mathcal{M} \cap U$ of $\Theta_{m-1}$ by
\begin{align}
	\mathrm{dist}(\Theta_m, \mathcal{M} \cap U)^2 &\leq |\Theta_m - \mathfrak{p}(\Theta_{m-1})|^2 \nonumber \\
	& \leq |\Theta_{m-1} - \mathfrak{p}(\Theta_{m-1}) - \alpha_{m-1} \nabla J(\Theta_{m-1}) \nonumber\\
	\phantom{\Theta_{m-1} - } & \phantom{\mathfrak{p}(\Theta_{m-1}) \Theta_{m-1} - \mathfrak{p}(\Theta_{m-1}) } + (
	\alpha_{m-1} \nabla J(\Theta_{m-1})
	-
	\alpha_{m-1} H_{m-1}
	)|^2.
	\label{eqn:proof_lemma4_1}
\end{align}
After expanding \eqref{eqn:proof_lemma4_1} and taking expectations, however, the effect of bias already appears, and we must diverge from the analysis from \cite[Equation~44]{fehrman2020convergence} thereafter.
In particular, the effect of the bias of $H_{m-1}$ needs to be handled in the terms
\begin{equation}
	\expectationBig{
		2
		\Bigl\langle
		\Theta_{m-1}
		- \mathfrak{p}(\Theta_{m-1})
		- \alpha_{m-1} \nabla J(\Theta_{m-1})
		,
		\alpha_{m-1} \nabla J(\Theta_{m-1})
		-
		\alpha_{m-1} H_{m-1}
		\Bigr\rangle
		\indicator{\mathcal{B}_{m-1}}
	},
	\label{eqn:Cross_term}
\end{equation}
and
\begin{align}
	\expectationBig{
		\Big|
		\alpha_{m-1} \nabla J(\Theta_{m-1})
		-
		\alpha_{m-1} H_{m-1}
		\Big|^2
		\indicator{\mathcal{B}_{m-1}}
	}
	= (\alpha_{m-1})^2 \expectationBig{
		|\eta_{m-1}|^2
		\indicator{\mathcal{B}_{m-1}}
	}.
	\label{eqn:Quadratic_term}
\end{align}
We specifically require bounds of these terms without relying on independence of the iterands.

We focus on \eqref{eqn:Quadratic_term} first. Recall for $m > 0$, that $\mathcal{F}_m$ is the sigma algebra defined in \eqref{eqn:sigma-algebra}. By using the tower property of the conditional expectation and conditioning on $\mathcal{F}_{m-1}$, from Lemma~\ref{lem:Bounded_norms_on_the_gradient_estimator} together with the fact that $T_m < cT_{m-1}$ for some $c>0$, we obtain directly
\begin{align}
	\eqref{eqn:Quadratic_term}  = (\alpha_{m-1})^2 \expectationBig{\expectationBig{
			|\eta_{m-1}|^2
			\indicator{\mathcal{B}_{m-1}} \big| \mathcal{F}_{m-1}}
	} & \eqcom{Lemma~\ref{lem:Bounded_norms_on_the_gradient_estimator}}\leq (\alpha_{m-1})^2 \frac{\C}{T_m} \expectation{\mathcal{L}_4(S_{t_{m-1}})^2 \indicator{\mathcal{B}_{m-1}}
	}
	.
	\label{eqn:proof_lemma4_2}
\end{align}

Let us next bound \eqref{eqn:Cross_term}. Note that this term does not vanish due to dependence of the samples conditional on $\mathcal{F}_{m-1}$. In our case, however, we have a Markov chain trajectory whose kernel will depend on $\Theta_{m-1}$. Let
\begin{equation}
	Z_{m-1} = \Theta_{m-1}
	- \mathfrak{p}(\Theta_{m-1})
	- \alpha_{m-1} \nabla J(\Theta_{m-1}).
\end{equation}
We use the law of total expectation again on \eqref{eqn:Cross_term}. Note that $Z_{m-1}$ and $\mathcal{B}_{m-1}$ are $\mathcal{F}_{m-1}$-measurable.
\begin{align}
	\eqref{eqn:Cross_term}
	& \leq 2\alpha_{m-1} \E \Bigl[ \bigl\langle  \indicator{ \mathcal{B}_{m-1} }            Z_{m-1}, \E[\eta_{m-1} | \mathcal{F}_{m-1}] \bigr\rangle \Bigr] \nonumber \\
	& \eqcom{i} \leq  2\alpha_{m-1}
	\expectationBig{|
		Z_{m-1}|^2 \indicator{ \mathcal{B}_{m-1} }            }^{1/2} \expectationBig{
		|\E[\eta_{m-1} \indicator{ \mathcal{B}_{m-1} }| \mathcal{F}_{m-1}]|^2
	}^{1/2}  \nonumber \\
	& \eqcom{ii} \leq  2\alpha_{m-1}
	\expectationBig{|
		Z_{m-1}|^2\indicator{ \mathcal{B}_{m-1} }            }^{1/2} \expectationBig{\indicator{ \mathcal{B}_{m-1} } \mathcal{L}_4(S_{t_{m-1}})^2
	}^{1/2}\frac{\C}{T_m},
	\label{eqn:error_bound}
\end{align}
where (i) have used Cauchy--Schwartz and (ii) \refLemma{lem:Bounded_norms_on_the_gradient_estimator} and the fact that for some $c>0$, $T_m < cT_{m-1}$.

The terms in \eqref{eqn:proof_lemma4_2} and \eqref{eqn:error_bound} containing $\mathcal{L}_4(S_{t_m})$ can be upper bounded as follows. From the definition of \eqref{eqn:Lyapunov_function_state_chain} and since $v \geq 16$, by a generalized mean inequality and the fact that $\mathcal{L}(s,a) \geq 1$ for any $(s, a) \in \mathcal{S} \times \mathcal{A}$ we have
\begin{equation}
	\mathcal{L}_{4}(s) \leq \mathcal{L}_{v}(s)^{4/v}  \leq \mathcal{L}_{v}(s)^{1/4}.
	\label{eqn:inequality_lyapunov_functions}
\end{equation}
Now, by \refLemma{lemma:expectation_L_is_bounded}, $\mathcal{L}^{\star}$ is such that for all $m \in \mathbb{N}$
\begin{equation}
	\expectationBig{\indicator{ \mathcal{B}_{m-1} } \mathcal{L}_4(S_{t_{m-1}})^2} \leq
	\expectationBig{\indicator{ \mathcal{B}_{m-2} } \mathcal{L}_4(S_{t_{m-1}})^2}
	\eqcom{ \ref{eqn:inequality_lyapunov_functions}} \leq \expectationBig{\indicator{ \mathcal{B}_{m-2} } \mathcal{L}_v(S_{t_{m-1}})
	} \leq \mathcal{L}^{\star}.
\end{equation}

For the other term in \eqref{eqn:error_bound}, we can use the same bound used in \cite[Equation~41]{fehrman2020convergence}: There exists constants $y, c > 0$ depending on $J, \theta^{\star}$  and $\mathfrak{r}_0$ such that on the event $\mathcal{B}_{m-1}$ we have
\begin{align}
	|
	Z_{m-1}
	|^2
	& \leq \bigl(1-\alpha_{m-1} y \bigr)^2 \mathrm{dist}(\Theta_{m-1}, \mathcal{M} \cap U)^2 + c \bigl(1-\alpha_{m-1} y \bigr)\alpha_{m-1} \mathrm{dist}(\Theta_{m-1}, \mathcal{M} \cap U)^3 \nonumber\\
	& + c (\alpha_{m-1})^2 \mathrm{dist}(\Theta_{m-1}, \mathcal{M} \cap U)^4
	.
	\label{eqn:bound1}
\end{align}
The bound in \eqref{eqn:bound1} characterizes the fact that, close to the manifold of maximizers, the projection is differentiable and can be approximated by an orthogonal expansion of $J$ around the manifold of maximizers.
The error terms of this expansion can be bounded depending on the Hessian at $\mathfrak{p}(\Theta_{m-1}) \in \mathcal{M} \cap U$, $\mathrm{Hess}_{\mathfrak{p}(\Theta_{m-1})}\ J$.
We refer to \cite[Proposition 17]{fehrman2020convergence} for a proof of this fact.

We will now use an induction argument to show the claim of the lemma.
Namely, we will assume for the time being that for $m-1$ we have
\begin{equation}
	\expectationBig{
		(\mathrm{dist}(
		\Theta_{m-1}
		,
		\mathcal{M} \cap U
		) \wedge \delta )^2
		\indicator{ \mathcal{B}_{m-1} }
	} \leq \delta^{2}c(\alpha )(m-1)^{-\sigma - \kappa},
	\label{eqn:proof_lemma4_3}
\end{equation}
where $c(\alpha )>0$ is a function of $a$ to be determined. We want to show \eqref{eqn:proof_lemma4_3} for $m$. To do so we will use \eqref{eqn:bound1} to bound $Z_{m-1}$.
Suppose that there exists a sequence $\{b_l\}_{l>0} \subset \R_+$ such that we have
\begin{equation}
	\expectationBig{|
		Z_{m-1}
		|^2 \indicator{ \mathcal{B}_{m-1} }} \leq b_{m-1}.
	\label{eqn:bound_first_term}
\end{equation}
Using \eqref{eqn:bound_first_term} in \eqref{eqn:error_bound} yields that for some $\Cl{5}>0$ we have:
\begin{equation}
	\eqref{eqn:Cross_term}  \leq 2 (b_{m-1})^{1/2} \alpha_{m-1} (\mathcal{L}^{\star})^{1/2} \frac{\Cr{5}}{T_m}.
	\label{eqn:bound2}
\end{equation}
From the expansion of \eqref{eqn:proof_lemma4_1} and combining the bounds of \eqref{eqn:bound1} and \eqref{eqn:bound2} together we obtain
\begin{align}
	\expectationBig{
		\mathrm{dist}(
		\Theta_m
		,
		\mathcal{M} \cap U
		)^2
		\indicator{ \mathcal{B}_{m-1} }
	} & \leq b_{m-1} + 2 (b_{m-1})^{1/2} \alpha_{m-1} (\mathcal{L}^{\star})^{1/2} \frac{\Cr{5}}{T_m}
	\nonumber \\
	& \phantom{b_{m-1}}+ (\alpha_{m-1})^2 \frac{\C}{T_m}\mathcal{L}^{\star}.
	\label{eqn:proof_lemma4_5}
\end{align}
We show now that from the induction hypothesis, if \eqref{eqn:proof_lemma4_3} holds, then we also have the bound
\begin{equation}
	b_{m-1} \leq c(\alpha ) \delta^2 m^{-\sigma-\kappa} - \delta^2\frac{\alpha y}{2}c(\alpha )(m-1)^{-\sigma-\kappa} m^{-\sigma}.
	\label{eqn:lemma_intermediate}
\end{equation}
Indeed, taking expectations in \eqref{eqn:bound1} and using the bound \eqref{eqn:proof_lemma4_3} yields
\begin{align}
	b_{m-1} \leq \bigl(1-\alpha_{m-1}y \bigr)^2 c(\alpha )(m-1)^{-\sigma - \kappa}
	&+ c(\alpha ) \bigl(1-\alpha_{m-1}y \bigr)\alpha_{m-1} \delta c(\alpha )(m-1)^{-\sigma - \kappa}  \nonumber \\
	&+ c(\alpha ) (\alpha_{m-1})^2 \delta^{2} c(\alpha )(m-1)^{-\sigma - \kappa}.
	\label{eqn:lemma_intermediate__}
\end{align}
Recall that $\alpha_{m-1} = \alpha m^{-\sigma/2 - \kappa}$. Adding and subtracting $c(\alpha ) m^{-\sigma-\kappa}$ in \eqref{eqn:lemma_intermediate__}, we obtain  that
\begin{align*}
	&b_{m-1} \leq  c(\alpha ) m^{-\sigma-\kappa}  \nonumber \\
	& +  c(\alpha ) m^{-\sigma}(m-1)^{-\sigma-\kappa}\left(m^{\sigma} -  (m-1)^{\sigma + \kappa}m^{-\kappa}  - 2\alpha y + \frac{\alpha^2 y}{m^{\sigma}} +  \bigl(1-\frac{\alpha y}{m^{\sigma}} \bigr)\alpha \delta  + \delta^{2}\frac{\alpha^2 y^2}{m^{\sigma}} \right).
\end{align*}
Note now that there exists $m_0(a)>0$ such that if $m \geq m_0(a)$, we have
\begin{equation}
	m^{\sigma} -  (m-1)^{\sigma + \kappa}m^{-\kappa}  - \alpha y + \frac{\alpha^2 y}{m^{\sigma}} < -\frac{\alpha y}{2}.
\end{equation}
Indeed, note that the latter equation can be satisfied for $m \geq m_0(a)$ since there exists a constant $c>0$ depending on $\sigma$ and $\kappa$ such that
\begin{align}
	m^{\sigma} -  (m-1)^{\sigma + \kappa}m^{-\kappa} & \leq m^{-\kappa}(  m^{\sigma + \kappa} -  (m-1)^{\sigma + \kappa}) \nonumber \\
	& \leq m^{-\kappa} (\sigma + \kappa)\max[(m-1)^{\sigma + \kappa - 1}, m^{\sigma + \kappa - 1}] \nonumber \\ & \leq \Cl{6}(\sigma + \kappa) m^{\sigma- 1}.
\end{align}
In this case we have that
\begin{equation}
	m_0(\alpha) = \Bigl( \frac{2\Cr{6}(\sigma + \kappa)}{y \alpha} \Bigr)^{1-\sigma}  >  \frac{c^{\prime}}{\alpha^{1-\sigma}}.
	\label{eqn:proof_lemma4_m}
\end{equation}
Then for $m > m_0(\alpha)$, we will have
\begin{align*}
	b_m \leq & c(\alpha) m^{-\sigma-\kappa} +  c(\alpha ) m^{-\sigma}(m-1)^{-\sigma-\kappa}\left(  -\frac{\alpha y}{2} +  \bigl(1-\frac{\alpha y}{m^{\sigma}} \bigr)\alpha \delta  + \delta^{2}\frac{\alpha^2 y^2}{m^{\sigma}} \right).
\end{align*}
Choose $\delta \in (0, \delta_1(\alpha)]$, where $\delta_1(\alpha)$ is a bound that we will choose appropriately, such that for any $m \geq m_0(\alpha)$ we have
\begin{equation}
	\bigl(1-\frac{\alpha y}{m^{\sigma}} \bigr)\alpha \delta  + \delta^{2}\frac{\alpha^2 y^2}{m^{\sigma}} \leq \alpha y.
	\label{eqn:proof_lemma4_4}
\end{equation}
Thus, from \eqref{eqn:proof_lemma4_3} we obtain \eqref{eqn:lemma_intermediate}.
With \eqref{eqn:lemma_intermediate} with an appropriate choice of $c(\alpha )$, we can now show \eqref{eqn:proof_lemma4_3} for $m$. We will namely choose $c(\alpha )$ as follows
\begin{equation}
	c(\alpha ) = \max\Bigl(\frac{c^{\prime}}{\alpha^{(1- \sigma)(\sigma + \kappa)}}, \frac{4 C^2 \mathcal{L}^{\star}  + 4yC \mathcal{L}^{\star} \alpha \ell}{\delta^2 \ell^2 y^2} \Bigr),
	\label{eqn:proof_lemma4_6}
\end{equation}
where recall that $\delta \in (0, \delta_1(\alpha)]$ and $\delta_1(\alpha)$ were chosen so that \eqref{eqn:proof_lemma4_4} holds.
Let $L= \ell^{-1}$. Substituting the bound of \eqref{eqn:lemma_intermediate} into \eqref{eqn:proof_lemma4_5} and recalling that $T_m = m^{\kappa + \sigma/2} \ell$ yields
\begin{align}
	&\expectationBig{
		(\mathrm{dist}(
		\Theta_m
		,
		\mathcal{M} \cap U
		))^2
		\indicator{ \mathcal{B}_{m-1} }
	} \leq c(\alpha ) \delta^2 m^{-\sigma-\kappa} - \frac{\alpha y}{2}  c(\alpha ) \delta^2(m-1)^{-\sigma-\kappa} m^{-\sigma} \nonumber \\
	&\phantom{aaaa}+ 2(c(\alpha ) \delta^2 m^{-\sigma-\kappa} - \frac{\alpha \lambda}{2}c(\alpha ) \delta^2(m-1)^{-\sigma-\kappa} m^{-\sigma})^{1/2} \alpha m^{-\sigma} (\mathcal{L}^{\star})^{1/2} \frac{\Cr{5}}{T_m} + \mathcal{L}^{\star} m^{-2\sigma}\frac{\alpha^2 \Cr{5}}{T_m} \nonumber \\
	&\leq c(\alpha ) \delta^2 m^{-\sigma-\kappa} + m^{-\sigma}(2\sqrt{c(\alpha )} \delta \Cr{5} a (\mathcal{L}^{\star})^{1/2}L m^{-\sigma - 3\kappa/2} \nonumber \\
	&\phantom{\leq c(\alpha ) \delta^2 m^{-\sigma-\kappa}}+ \Cr{5} \mathcal{L}^{\star} \alpha^2 L m^{-3\sigma/2 - \kappa} - c(\alpha ) \delta^2 \alpha y(m-1)^{-\sigma - \kappa}) \nonumber \\
	&\leq c(\alpha ) \delta^2 m^{-\sigma-\kappa} + m^{-\sigma}(m-1)^{-\sigma - \kappa}(2\sqrt{c(\alpha )} \delta \Cr{5} a (\mathcal{L}^{\star})^{1/2} L  + \Cr{5}\mathcal{L}^{\star} \alpha^2 L - c(\alpha ) \delta^2 \alpha y).
	\label{eqn:lemma_intermediate2}
\end{align}
By the choice of $c(\alpha )$ in \eqref{eqn:proof_lemma4_6}, for any $\kappa \geq 0$ we have the following inequality
\begin{equation}
	2\sqrt{c(\alpha )} \delta \Cr{6} (\mathcal{L}^{\star})^{1/2} L  + \Cr{6}\mathcal{L}^{\star} aL  - c(\alpha ) \delta^2 y < 0.
\end{equation}
Hence, with this choice of $c(\alpha )$, in  \eqref{eqn:lemma_intermediate2} the latter term in the right-hand side is negative for any $m \geq 2$ and the induction step follows if $m > m_0(\alpha)$.
That is, we have for some $c > 0$ that and when $m > m_0(\alpha)$ that
\begin{equation}
	\expectationBig{
		\mathrm{dist}(
		\Theta_m
		,
		\mathcal{M} \cap U
		)^2
		\indicator{ \mathcal{B}_{m-1} }
	}
	\leq  c\max \Bigl( \frac{\delta^2}{a^{(1-\sigma)(\sigma + \kappa)}}, \frac{\mathcal{L}^{\star}(1+\alpha \ell)}{\ell^2} \Bigr)  m^{-\sigma-\kappa}.
	\label{eqn:_2}
\end{equation}

We have left to show that the induction hypothesis holds in \eqref{eqn:proof_lemma4_3} for some $m$.
Recall that $m > m_0(\alpha)$ is the only restriction we needed on the starting point for the induction argument to work---$\delta$ was already chosen depending on $\alpha$ in \eqref{eqn:proof_lemma4_4}.
From the choice
\begin{equation}
	m_0(\alpha) \geq \frac{c^{\prime}}{\alpha^{1-\sigma}},
\end{equation}
if $m \leq m_0(\alpha)$, the following slightly changed version of \eqref{eqn:proof_lemma4_3} will hold; namely
\begin{equation}
	\expectationBig{
		(\mathrm{dist}(
		\Theta_m
		,
		\mathcal{M} \cap U
		)^2 \wedge \delta^2)
		\indicator{ \mathcal{B}_{m-1} }
	}
	\leq  \delta^2 c(\alpha )  m^{-\sigma-\kappa}.
	\label{eqn:lemma_conditional_convergence}
\end{equation}
Hence, by same arguments conducted with \eqref{eqn:lemma_conditional_convergence} instead of\eqref{eqn:proof_lemma4_3}, we have shown by induction that \eqref{eqn:lemma_conditional_convergence} holds for $m>0$.

For convenience, we will further show that there exists a constant $\Cl{7}>0$ such that for all $m > 0$ we have
\begin{equation}
	\expectationBig{
		(\mathrm{dist}(
		\Theta_m
		,
		\mathcal{M} \cap U
		)^2 \wedge \delta^2)
		\indicator{ \mathcal{B}_{m-1} }
	}
	\leq  \Cr{7}  \mathcal{L}^{\star} m^{-\sigma-\kappa}.
	\label{eqn:lemma_conditional_convergence_final_rate_2}
\end{equation}
Fix $\Cr{7}> 0$. Choose $\delta_0 \leq \delta_1(\alpha)$ depending on $\alpha$ small enough and $\ell_0 > 0$ large enough such that for $\delta \in (0, \delta_0]$ and $\ell \in [\ell_0, \infty)$ we have that
\begin{align}
	\frac{c^{\prime} \delta^2}{\alpha^{(1-\sigma)(\sigma + \kappa)}} &< \Cr{7} \geq \Cr{7} \mathcal{L}^{\star}, \nonumber \\
	\frac{cD(1+\alpha \ell)}{\ell^2} &< \Cr{7} \mathcal{L}^{\star},
	\label{eqn:proof_lemma5_1}
\end{align}
With the conditions in  \eqref{eqn:proof_lemma5_1}, the proof of the lemma follows noting that $\delta^2 c(\alpha ) = \delta^2 c(\alpha, \ell) < \Cr{7} \mathcal{L}^{\star}$.

\subsection{Proof of Lemma~\ref{lem:maximal_excursion_event_probability}}

\label{sec:Proof_maximal_excursion_event_probability}

We will again use the notation that $t_{m+1} - t_m = T_m$ and without loss of generality assume that $T_m = \ell m^{\sigma/2 + \kappa}$ as in \Cref{sec:Proof_of_conditional_convergence_basin}.
The proof of Lemma~\ref{lem:maximal_excursion_event_probability} also mainly follows the steps of \cite{fehrman2020convergence}. However, we again need to take care of the terms that the bias and lack of independence generate in the analysis.

The bounding starts noting the inequality
\begin{equation}
	\expectationBig{
		\max_{1 \leq l \leq m}
		\bigl|
		\Theta_l - \Theta_0
		\bigr|
		\indicator{ \mathcal{B}_{l-1} }} \leq \sum_{l=1}^m \expectation{|\Theta_{l} - \Theta_{l-1}|^2 \indicator{ \mathcal{B}_{l-1} }}^{1/2}.
	\label{eqn:proof_lemma6_1}
\end{equation}
We will show that there exists a constant $c>0$ such that for $l \in [m]$ we have
\begin{equation}
	\expectation{|\Theta_{l+1} - \Theta_{l}|^2 \indicator{ \mathcal{B}_{l} }}^{1/2} \leq c \alpha \Bigl( l^{-3/2\sigma - \kappa/2} + \sqrt{\frac{1}{\ell}} l^{-5\sigma/8 - \kappa/2} \Bigr),
	\label{eqn:proof_lemma6_2}
\end{equation}
where the exponents of $\sigma$ and $\kappa$ already differ from the result of \cite{fehrman2020convergence}, and are required to account for the lack of independence and bias.
Following the steps from \cite{fehrman2020convergence}, in the neighborhood $V_{\mathfrak{r}, \delta}(\theta^{\star})$, for each $l \leq m$ there is a random variable $\epsilon_{l} : \mathcal{B}_{l} \to \R^{n}$ and there exists a constant $c > 0$ such that
\begin{equation}
	|\epsilon_{l}| < c \mathrm{dist}(\Theta_{l}, \mathcal{M} \cap U)^2,
	\label{eqn:proof_lemma6_epsilon}
\end{equation}
and such that on the event $\mathcal{B}_{l}$ we have
\begin{equation}
	\nabla J(\Theta_{l}) = \mathrm{Hess}_{\mathfrak{p}(\Theta_{l})}(\Theta_{l} - \mathfrak{p}(\Theta_{l})) + \epsilon_{l}.
\end{equation}
Recalling the definition of $\eta_l$ in \eqref{eqn:plugin_estimator_gradient_difference}, we have then the equality
\begin{equation}
	\Theta_{l+1} = \Theta_{l} - \alpha_l \mathrm{Hess}_{\mathfrak{p}(\Theta_{l})}(\Theta_{l} - \mathfrak{p}(\Theta_{l})) - \alpha_l \epsilon_{l} + \alpha_l \eta_{l}.
	\label{eqn:proof_lemma6_3}
\end{equation}
Define
\begin{equation}
	\tilde{\Theta}_{l} = \Theta_{l} - \alpha_l \mathrm{Hess}_{\mathfrak{p}(\Theta_{l})}(\Theta_{l} - \mathfrak{p}(\Theta_{l})).
	\label{eqn:proof_lemma6_4}
\end{equation}
We use the triangle inequality with in \eqref{eqn:proof_lemma6_3} separating $\Theta_{l+1} - \Theta_{l}$ as the summands of $\Theta_{l+1} - \tilde{\Theta}_{l}$ and $\tilde{\Theta}_{l} - \Theta_{l}$.

We estimate first $|\Theta_{l+1} - \tilde{\Theta}_{l}|^2$.
In our case, after expanding $\expectation{|\Theta_{l+1} - \tilde{\Theta}_{l}|^2 \indicator{ \mathcal{B}_{l} }}$, we diverge from \cite[Equation~58]{fehrman2020convergence} and we need to bound
\begin{equation}
	\alpha_l^2\expectationBig{\indicator{ \mathcal{B}_{l} } \langle \epsilon_{l} , \eta_{l}\rangle }.
	\label{eqn:maximal_excursion_error}
\end{equation}
Similar to the proof of Lemma~\ref{lem:conditional_convergence_in_basin}, we can condition on $\mathcal{F}_{l}$ and using that $\epsilon_{l}$ and $\mathcal{B}_{l}$ are $\mathcal{F}_{l}$-measurable together with the Cauchy--Schwartz inequality, we have
\begin{align}
	\alpha_l^2\expectationBig{\indicator{ \mathcal{B}_{l} } \bigl\langle \epsilon_l , \eta_{l} \bigr\rangle } & \leq \alpha_l^2\expectationBig{ \bigl\langle \indicator{ \mathcal{B}_{l} } \epsilon_l, \expectationbig{\eta_l \indicator{ \mathcal{B}_{l} }| \mathcal{F}_{l}} \bigr\rangle} \nonumber \\
	& \leq \alpha_l^2\expectationBig{ \indicator{ \mathcal{B}_{l} } |\epsilon_l|^2}^{1/2} \expectationBig{|\expectationbig{\eta_l \indicator{ \mathcal{B}_{l} }|\mathcal{F}_{l}}|^2 }^{1/2}.
	\label{eqn:proof_lemma6_5}
\end{align}
Since $\indicator{\mathcal{B}_{m}} \leq \indicator{\mathcal{B}_{m-1}}$, we can bound
\begin{equation}
	\expectationBig{\big| \expectation{\eta_l \indicator{ \mathcal{B}_{l}} | \mathcal{F}_{l}}\big|^{2}}^{1/2} \eqcom{Lemma~\ref{lem:Bounded_norms_on_the_gradient_estimator}}\leq \expectationBig{\indicator{\mathcal{B}_{l} } \frac{\C^2}{T_{l}^2} \mathcal{L}_v(S_{t_{l}})}^{1/2} \eqcom{Lemma~\ref{lemma:expectation_L_is_bounded}}\leq \frac{\C}{T_{l}}.
\end{equation}
For the remaining term in \eqref{eqn:proof_lemma6_5},  recall that on the event $\mathcal{B}_{l}$, since $\Theta_l \in V_{\mathfrak{r}, \delta}(\theta^{\star})$, we have that $\mathrm{dist}(
\Theta_l
,
\mathcal{M} \cap U
) \leq \delta$. Hence, we can bound for any $l > 0$ that
\begin{align}
	\expectationBig{ \indicator{ \mathcal{B}_{l} } |\epsilon_l|^2}^{1/2} & \eqcom{\ref{eqn:proof_lemma6_epsilon}}\leq (\alpha_l)^2 \E[ \mathrm{dist}(
	\Theta_l
	,
	\mathcal{M} \cap U
	)^4 \indicator{ \mathcal{B}_{l} }]^{1/2}\frac{\Cl{8}}{T_{l+1}} \nonumber \\
	& \leq (\alpha_l)^2\delta^2 \E[\mathrm{dist}(
	\Theta_l
	,
	\mathcal{M} \cap U
	)^2\indicator{ \mathcal{B}_{l} }]^{1/2} \frac{\Cr{8}}{T_{l+1}} \nonumber \\
	& \leq (\alpha_l)^2\delta^2 \E[\mathrm{dist}(
	\Theta_l
	,
	\mathcal{M} \cap U
	)^2\indicator{ \mathcal{B}_{l-1} }]^{1/2} \frac{\Cr{8}}{T_{l+1}} \nonumber \\
	& \eqcom{Lemma~\ref{lem:conditional_convergence_in_basin}} \leq (\alpha_l)^2 \delta^2 l^{-\sigma/2-\kappa/2} \frac{\C}{T_{l}}.
	\label{eqn:proof_lemma6_6}
\end{align}
The estimation of the remaining terms in the expansion of $\expectation{|\Theta_{l} - \tilde{\Theta}_{l-1}|^2 \indicator{ \mathcal{B}_{l-1} }}$ can be conducted in the same way as that in \cite{fehrman2020convergence}, to which we refer for the details to the interested reader.
Together with the estimate of \eqref{eqn:proof_lemma6_6} that accounts for the biases we have that
\begin{align}
	\E[|\Theta_{l} - \tilde{\Theta}_{l-1}|^2 \indicator{ \mathcal{B}_{l} }] & \leq \C(\alpha_l)^2\delta^2 \expectationBig{\mathrm{dist}(
		\Theta_l
		,
		\mathcal{M} \cap U
		)^2\indicator{ \mathcal{B}_{l} }} \nonumber\\
	& + 2 \delta \expectationBig{\mathrm{dist}(
		\Theta_l
		,
		\mathcal{M} \cap U
		)^2\indicator{ \mathcal{B}_{l}}}^{1/2} \frac{\Cl{9}}{T_{l}}  + (\alpha_l)^2 \frac{\Cl{9}}{T_{l}} \nonumber \\
	& \leq \C(\alpha_l)^2\Bigl[ \delta^2 l^{-\sigma-\kappa} + 2 \delta l^{-\sigma/2-\kappa/2} \frac{1}{T_{l}} + \frac{1}{T_{l}} \Bigr].
	\label{eqn:proof_lemma6_7}
\end{align}
Substituting $T_{l} = t_{l+1} - t_{l} = l^{\kappa + \sigma/2} \ell$ and using $\alpha_{l} < \alpha_{l-1} = \alpha l^{-\sigma}$ into \eqref{eqn:proof_lemma6_7} yields the bound
\begin{align}
	\expectationbig{|\Theta_{l} - \tilde{\Theta}_{l-1}|^2 \indicator{ \mathcal{B}_{l-1} }} & \leq \C \frac{\alpha^2}{ l^{2\sigma}}\Bigl(\delta^2 \frac{1}{l^{\sigma+\kappa}}  + 2 \delta \frac{1}{l^{\sigma +3\kappa/2\ell}} +  \frac{1}{l^{\kappa + \sigma/2}} \Bigr) \nonumber \\
	& \leq \C  \frac{\alpha^2}{l^{5\sigma/4 + \kappa}\ell},
	\label{eqn:proof_lemma6_8}
\end{align}
where in the last inequality we have taken the term with the highest order.
Using the previous bounds from \refLemma{lem:conditional_convergence_in_basin} we can show that
\begin{equation}
	\expectationbig{|\Theta_{l} - \tilde{\Theta}_{l}|^2 \indicator{ \mathcal{B}_{l} }} \leq \alpha_l^2 \expectationbig{\mathrm{dist}(\Theta_{l},\mathcal{M} \cap U) \indicator{ \mathcal{B}_{l} }} \leq \C \frac{a^2}{l^{3\sigma + \kappa}},
	\label{eqn:proof_lemma6_9}
\end{equation}
so that using the triangle inequality and combining the bounds of \eqref{eqn:proof_lemma6_8} and \eqref{eqn:proof_lemma6_9} we obtain
\begin{equation}
	\expectationbig{|\Theta_{l+1} - \Theta_{l}|^2 \indicator{ \mathcal{B}_{l} }}^{1/2} \leq \Cl{10} \alpha \Bigl( l^{-3/2\sigma - \kappa/2} + \sqrt{\ell}^{-1} l^{-5\sigma/8 - \kappa/2} \Bigr).
	\label{eqn:proof_lemma6_10}
\end{equation}
Hence, since $\sigma \in (2/3,1)$ adding the bound \eqref{eqn:proof_lemma6_10} in \eqref{eqn:proof_lemma6_1} yields
\begin{align}
	\expectationBig{
		\max_{1 \leq l \leq m}
		\bigl|
		\Theta_l - \Theta_0
		\bigr|
		\indicator{ \mathcal{B}_{l-1} }} & \leq \sum_{l=1}^m \Cr{10} \alpha (l^{-3/2\sigma - \kappa/2} + \sqrt{\ell}^{-1} l^{-\sigma - \kappa/2}) \nonumber \\
	& \leq \C \alpha (m^{1-3/2\sigma - \kappa/2} + \sqrt{\ell}^{-1}m^{1 -5\sigma/8 - \kappa/2}). \nonumber
\end{align}

\subsection{Proof of Lemma~\ref{lem:probability_not_leaving_basin}}
\label{sec:Proof_probability_not_leaving_basing}

The proof mimicks the proof strategy of \cite[Proposition~24]{fehrman2020convergence}, but modifications are required due to our Markovian assumptions and appearances of biases.
Specifically, we must carefully consider the adverse effects that these biases could have on the probability that the iterates exit the basin of attraction.
Concretely, our effort will go into firstly proving the following sufficiently strong analogue of \cite[Equation~75]{fehrman2020convergence} that is applicable to our problem:

\begin{lemma}
	\label{lem:Probability_of_exit_and_Bm_is_bounded}
	
	There exist constants $c_1, c_2 > 0$ such that
	\begin{equation}
		\probability{\mathrm{dist}(\Theta_{m}, \mathcal{M} \cap U) > \delta, \mathcal{B}_{m-1} }
		\leq
		\frac{c_1 \alpha^2}{\delta^2 \ell m^{2\sigma}} \mathbb{P}[\mathcal{B}_{m-1}]
		+
		\frac{c_2}{\delta^4 \ell m^{\sigma + \kappa}}
		.
	\end{equation}
\end{lemma}

The proof of Lemma~\ref{lem:Probability_of_exit_and_Bm_is_bounded} can be found in \Cref{sec:Proof_of_lemma__Probability_of_exit_and_Bm_is_bounded}.

Once Lemma~\ref{lem:Probability_of_exit_and_Bm_is_bounded} has been established, we secondly estimate the combined probability that any of the iterates escape in directions tangential to the manifold. The proof of this fact, which is analogous to \cite[Equation~78--79]{fehrman2020convergence}, can be found in \Cref{sec:Proof_of_lemma__An_exit_event_is_unlikely_on_a_next_step}.

\begin{lemma}
	\label{lem:An_exit_event_is_unlikely_on_a_next_step}
	
	If $\Theta_0 \in V_{\mathfrak{r}/2, \delta}(\theta^{\star})$, then
	\begin{equation}
		\sum_{l=1}^{m} \probability{\mathrm{dist}(\Theta_{l}, \mathcal{M} \cap U) < \delta, \Theta_{l} \notin V_{\mathfrak{r}, \delta}(\theta^{\star}), \mathcal{B}_{l-1}} \leq \probabilityBig{
			\max_{1 \leq l \leq m}
			\bigl|
			\Theta_l - \Theta_0
			\bigr| \indicator{\mathcal{B}_{l-1}} > R/2 - 2\delta,
		}
		.
		\label{eqn:proof_lemma7_6a}
	\end{equation}
\end{lemma}

\subsubsection{Proof that Lemmas~\ref{lem:Probability_of_exit_and_Bm_is_bounded} and \ref{lem:An_exit_event_is_unlikely_on_a_next_step} imply Lemma~\ref{lem:probability_not_leaving_basin}}

First, note that the recursion
\begin{equation}
	\probability{\mathcal{B}_{m}} =  \probability{\Theta_{m} \in V_{\mathfrak{r}, \delta}(\theta^{\star}), \mathcal{B}_{m-1}}
	=
	\probability{\mathcal{B}_{m-1}}
	-
	\probability{\Theta_{m} \notin V_{\mathfrak{r}, \delta}(\theta^{\star}), \mathcal{B}_{m-1}}
	\label{eqn:proof_lemma7_intro}
\end{equation}
can be iterated whenever we can control and bound the following probabilities
\begin{align}
	\probability{\Theta_{m} \notin V_{\mathfrak{r}, \delta}(\theta^{\star}), \mathcal{B}_{m-1}}
	&
	=
	\probability{\mathrm{dist}(\Theta_{m}, \mathcal{M} \cap U) > \delta, \mathcal{B}_{m-1}} \nonumber \\ &
	\phantom{=}
	+ \probability{\mathrm{dist}(\Theta_{m}, \mathcal{M} \cap U) \leq \delta, \Theta_{m} \notin V_{\mathfrak{r}, \delta}(\theta^{\star}), \mathcal{B}_{m-1}}
	.
	\label{eqn:proof_lemma7_01}
\end{align}
Using Lemma~\ref{lem:Probability_of_exit_and_Bm_is_bounded} and induction on \eqref{eqn:proof_lemma7_intro} and \eqref{eqn:proof_lemma7_01}, it follows that for some $c>0$,
\begin{equation}
	\probability{\mathcal{B}_{m} } \geq \prod_{l=1}^{m} \Bigl( 1 - \frac{c\alpha^2}{\delta^2 \ell l^{2\sigma}}\Bigr)_{+} - \sum_{l=1}^{m} \frac{c}{\ell \delta^4  l^{\sigma + \kappa}} - \sum_{l=1}^{m} \probability{\mathrm{dist}(\Theta_{l}, \mathcal{M} \cap U) < \delta, \Theta_{l} \notin V_{\mathfrak{r}, \delta}(\theta^{\star}), \mathcal{B}_{l-1}}
	.
	\label{eqn:proof_lemma7_7}
\end{equation}
We use Lemma~\ref{lem:An_exit_event_is_unlikely_on_a_next_step} together with Lemma~\ref{lem:maximal_excursion_event_probability} and Markov's inequality to obtain the bound
\begin{align}
	\sum_{l=1}^{m} \probability{\mathrm{dist}(\Theta_{l}, \mathcal{M} \cap U) < \delta, \Theta_{l} \notin V_{\mathfrak{r}, \delta}(\theta^{\star}), \mathcal{B}_{l-1}} \leq c\alpha \frac{(m^{1-3/2\sigma - \kappa/2} + \ell^{-1/2}m^{1-5\sigma/8 - \kappa/2})}{(\mathfrak{r}/2 - 2\delta)_{+}}.
	\label{eqn:proof_lemma7_02}
\end{align}
Thus, substituting \eqref{eqn:proof_lemma7_02} in \eqref{eqn:proof_lemma7_7}, for some $c>0$ we have
\begin{equation}
	\probability{\mathcal{B}_{m}}  \geq \prod_{l=1}^{m} \Bigl( 1 - \frac{c\alpha^2}{\delta^2 \ell l^{2\sigma}}\Bigr)_{+} - \sum_{l=1}^{m} \frac{c}{\ell \delta^4  l^{\sigma + \kappa}} - c\alpha \frac{(m^{1-3/2\sigma - \kappa/2} + \ell^{-1/2}m^{1-5\sigma/8 - \kappa/2})}{(\mathfrak{r}/2 - 2\delta)_{+}}.
	\label{eqn:proof_lemma7_5}
\end{equation}
Note first that since $\sigma \in (2/3,1)$ and $\kappa \geq 0$, if $\sigma + \kappa \neq 1$, then there exists a constant $\Cl{20} > 0$ such that
\begin{equation}
	\sum_{l=1}^{m} \frac{c}{\ell \delta^4  l^{\sigma + \kappa}} \leq \Cr{20} m^{1- \sigma -\kappa}.
	\label{eqn:proof_lemma7_9}
\end{equation}
Lastly, there also exists a constant $c>0$, $\alpha_0>0$, $\delta_0$ such that if $\alpha \in (0, \alpha_0]$ and $\delta \in (0, \delta_0]$ then there exists $\ell_0 > 0$ such that if $\ell \in [\ell_0, \infty)$ then
\begin{equation}
	\prod_{l=1}^{m} \Bigl( 1 - \frac{c\alpha^2}{\delta^2 \ell l^{2\sigma}}\Bigr)_{+} \geq \exp\Bigl(-\frac{c\alpha^2}{\delta^2 \ell}\Bigr).
	\label{eqn:proof_lemma7_6b}
\end{equation}
Lower bounding \eqref{eqn:proof_lemma7_5} using \eqref{eqn:proof_lemma7_9} and \eqref{eqn:proof_lemma7_6b} yields Lemma~\ref{lem:probability_not_leaving_basin}.
\QuodEratDemonstrandum

\subsubsection{Proof of Lemma~\ref{lem:Probability_of_exit_and_Bm_is_bounded}}
\label{sec:Proof_of_lemma__Probability_of_exit_and_Bm_is_bounded}

We follow first \cite[Equation~69]{fehrman2020convergence}, by fixing $\delta_1$ small enough such that $\delta \in (0, \delta_1]$, on the event $\mathcal{B}_{m-1}$ it is shown by \cite{fehrman2020convergence} that we have the inequality
\begin{equation}
	\mathrm{dist}(\Theta_{m}, \mathcal{M} \cap U) \leq \Bigl(1 - \frac{\lambda \alpha_{m-1}}{2}\Bigr)\mathrm{dist}(\Theta_{m-1}, \mathcal{M} \cap U) + \alpha_{m-1} |\eta_{m-1}|
	.
	\label{eqn:probability_lemma_1}
\end{equation}

We consider now the event $\{\mathrm{dist}(\Theta_{m}, \mathcal{M} \cap U) > \delta \} \cap \mathcal{B}_{m-1}$.
This event occurs when in \eqref{eqn:probability_lemma_1}, either $\Theta_{m-1} \in V_{\mathfrak{r}, \delta/2}(\theta^{\star})$ and $|\eta_{m-1}| \geq \alpha_{m-1}\delta/2$, or $\Theta_{m-1} \in V_{\mathfrak{r}, \delta}(\theta^{\star}) \backslash V_{\mathfrak{r}, \delta/2}(\theta^{\star})$ and the gradient term can have smaller size.
Mathematically, this translates into the inequality
\begin{align}
	&
	\probability{\mathrm{dist}(\Theta_{m}, \mathcal{M} \cap U) > \delta, \mathcal{B}_{m-1} }
	\leq
	\probabilityBig{|\eta_{m-1}| \geq \frac{\delta }{2\alpha_{m-1}}, \Theta_{m-1} \in V_{\mathfrak{r}, \delta/2}(\theta^{\star}), \mathcal{B}_{m-2}}
	\label{eqn:intereq1}
	\\ &
	+ \probabilityBig{|\eta_{m-1}| \geq \frac{\delta \lambda}{2}, \Theta_{m-1} \in V_{\mathfrak{r}, \delta}(\theta^{\star}) \backslash V_{\mathfrak{r}, \delta/2}(\theta^{\star}), \mathcal{B}_{m-2}}
	=:
	P_1
	+
	P_2
	.
	\nonumber
\end{align}

Contrary to what is done in the proof of \cite[Proposition~24]{fehrman2020convergence}, we cannot use an independence property to estimate the probabilities $P_1$ and $P_2$ in \eqref{eqn:intereq1}.
After all, the Markov chain's behavior at epoch $m-1$ depends on $\Theta_{m-1}$.

In order to overcome this issue we will use the characterization of $\eta_{m-1}$ in Lemma~\ref{lem:Bounded_norms_on_the_gradient_estimator}.
Recall Lemma~\ref{lem:Bounded_norms_on_the_gradient_estimator}, and note that it implies
\begin{align}
	&
	\expectationBig{
		\indicator{\mathcal{B}_{m-1}}
		\indicatorBig{|\eta_{m-1}| \geq \frac{\delta }{2\alpha_{m-1}}}
		\Big|
		\mathcal{F}_{m-1}
	}
	= \probabilityBig{
		|\eta_{m-1}| \geq \frac{\delta }{2\alpha_{m-1}}, \mathcal{B}_{m-1} \big| \mathcal{F}_{m-1}
	}
	\nonumber \\ &
	\leq
	\frac{\expectation{|\eta_{m-1}|^2 \indicator{\mathcal{B}_{m-1}}\ |\mathcal{F}_{m-1}}}{\frac{\delta^2}{4(\alpha_{m-1})^2}}
	\leq
	\frac{4\Cl{11} (\alpha_{m-1})^2 \mathcal{L}_4(S_{t_{m-1}})}{\delta^2 T_m}
	\label{eqn:proof_lemma7_1}
\end{align}
since there exist a constant $c > 0$ such that $T_m < cT_{m-1}$.

Let us first bound $P_1$ in \eqref{eqn:intereq1}.
We can write
\begin{align}
	P_1
	&
	\eqcom{i} = \expectationBig{\indicatorbig{|\eta_{m-1}| \geq \frac{\delta }{2\alpha_{m-1}}} \indicator{\Theta_{m-1} \in V_{\mathfrak{r}, \delta/2}(\theta^{\star})} \indicator{\mathcal{B}_{m-2}} \indicator{\mathcal{B}_{m-1}}}
	\nonumber \\ &
	=
	\expectationBig{\indicator{\Theta_{m-1} \in V_{\mathfrak{r}, \delta/2}(\theta^{\star})} \indicator{\mathcal{B}_{m-2}} \expectationbig{\indicator{\mathcal{B}_{m-1}} \indicator{|\eta_{m-1}| \geq \frac{\delta }{2\alpha_{m-1}}} | \mathcal{F}_{m-1}}}
	\nonumber \\ &
	\eqcom{\ref{eqn:proof_lemma7_1}}\leq
	\frac{4 \Cr{11} (\alpha_{m-1})^2}{T_m \delta^2} \expectationBig{\indicator{\Theta_{m-1} \in V_{\mathfrak{r}, \delta/2}(\theta^{\star})} \indicator{\mathcal{B}_{m-2}} \mathcal{L}_4(S_{t_{m-1}})},
	\label{eqn:intereq2}
\end{align}
where for (i) we have used the fact that $\{\Theta_{m-1} \in V_{\mathfrak{r}, \delta/2}(\theta^{\star})\} \cap \mathcal{B}_{m-2} \subset \mathcal{B}_{m-1}$.

We deal now with the remaining term in \eqref{eqn:intereq2}.
Differently to the independent and unbiased case we need to control the bias and use the tail probability that the Lyapunov function is larger than a certain bound in order to estimate the deviation probability. This step is the crucial different step compared to \cite{fehrman2020convergence}, where we have to explicitly use Assumption~\ref{ass:geometric-ergodicity} and \ref{ass:growth_condition}. Note that a Cauchy--Schwartz inequality in \eqref{eqn:intereq2} will not yields an inequality strong enough. See the remark after the proof for further details.

Before bounding the remaining term in \eqref{eqn:intereq2}, we obtain the necessary inequalities. Recall from \refLemma{lemma:expectation_L_is_bounded} that since $\E[\mathcal{L}_4(S_{t_{m-1}})^4 \indicator{ \mathcal{B}_{m-2}}] < \E[\mathcal{L}_v(S_{t_{m-1}}) \indicator{ \mathcal{B}_{m-2}}] < D <  \infty$, then by Markov's inequality we have that there exists $D >0$ such that for any $m > 0$,
\begin{equation}
	\probability{ \mathcal{L}(S_{t_{m-1}}) > m^{s}, \mathcal{B}_{m-2}} \leq D^4 m^{-4s}.
\end{equation}
Note also that under the moment assumptions the following holds
\begin{align}
	\expectationBig{\mathcal{L}(S_{t_{m-1}}) \indicator{ \mathcal{B}_{m-2} }\indicator{\mathcal{L}(S_{t_{m-1}}) > m^{s}} }
	&
	= \int_{m^s}^{\infty} \probability{\mathcal{L}(S_{t_{m-1}}) > t, \mathcal{B}_{m-2}} \d{t}
	\nonumber \\ &
	=
	\int_{m^s}^{\infty} \frac{D^4}{t^4} \d{t}
	\leq D^4 m^{-3s + 1}
	.
	\label{eqn:intereq3}
\end{align}
We use the \eqref{eqn:intereq3} to bound \eqref{eqn:intereq2} as follows
\begin{align}
	&
	\expectationBig{\indicator{\Theta_{m-1}
			\in V_{\mathfrak{r}, \delta/2}(\theta^{\star})}\mathcal{L}_4(S_{t_{m-1}}) \indicator{ \mathcal{B}_{m-2} }}
	\nonumber \\ &
	\leq
	\expectationBig{
		\indicator{\Theta_{m-1} \in V_{\mathfrak{r}, \delta/2}(\theta^{\star})}\mathcal{L}_4(S_{t_{m-1}})\indicator{ \mathcal{B}_{m-2} }
		\Bigl(
		\indicator{\mathcal{L}_4(S_{t_{m-1}}) > m^{s}} + \indicator{\mathcal{L}_4(S_{t_{m-1}}) \leq m^{s}}
		\Bigr)
	}
	\nonumber \\ &
	\leq \expectationBig{\indicator{\Theta_{m-1} \in V_{\mathfrak{r}, \delta/2}(\theta^{\star})} m^{s}\indicator{ \mathcal{B}_{m-2} }} + \expectationBig{\mathcal{L}(S_{t_{m-1}}) \indicator{ \mathcal{B}_{m-2} }\indicator{\mathcal{L}(S_{t_{m-1}}) > m^{s}} }
	\nonumber \\ &
	\eqcom{\ref{eqn:intereq3}} \leq m^{s} \probability{\Theta_{m-1} \in V_{\mathfrak{r}, \delta/2}(\theta^{\star}), \mathcal{B}_{m-2}} + \Cl{12} D m^{-3s + 1}
	\leq m^{s} \probability{\mathcal{B}_{m-1}} + \Cr{12} D m^{-3s + 1}
	.
	\label{eqn:proof_lemma7_2}
\end{align}
Thus, using \eqref{eqn:proof_lemma7_2}, we can bound $P_1$ in \eqref{eqn:intereq1}.
Specifically,
\begin{equation}
	P_1
	\leq
	\frac{4\C (\alpha_{m-1})^2}{T_m \delta^2}(m^{s} \mathbb{P}[\mathcal{B}_{m-1}] + m^{-3s + 1})
	.
	\label{eqn:Ultimate_bound_on_P1}
\end{equation}
This completes our bound for $P_1$.

We now bound $P_2$ in \eqref{eqn:intereq1}.
Repeating the argumentation behind \eqref{eqn:Ultimate_bound_on_P1}, we can show that
\begin{align}
	P_2
	\leq
	\frac{4 \C }{T_m \lambda^2 \delta^2}
	\Bigl(
	m^{s}
	\mathbb{P}
	\bigl[
	\Theta_{m-1} \in V_{\mathfrak{r}, \delta}(\theta^{\star}) \backslash V_{\mathfrak{r}, \delta/2}(\theta^{\star}), \mathcal{B}_{m-2}
	\bigr]
	+
	m^{-3s + 1}
	\Bigr)
	.
	\label{eqn:intereq4}
\end{align}
Using the facts
(i)
$
\{
\Theta_{m-1} \in V_{\mathfrak{r}, \delta}(\theta^{\star})
\backslash
V_{\mathfrak{r}, \delta/2}(\theta^{\star})
\}
\subseteq
\{
\mathrm{dist}(\Theta_{m-1}, \mathcal{M} \cap U)
\geq
\delta / 2
\}
$,
with
(ii) an application of Lemma~\ref{lem:conditional_convergence_in_basin} and Markov's inequality,
reveals that
\begin{equation}
	\probability{
		\Theta_{m-1} \in V_{\mathfrak{r}, \delta}(\theta^{\star}) \backslash V_{\mathfrak{r}, \delta/2}(\theta^{\star}),
		\mathcal{B}_{m-2}
	}
	\eqcom{i}\leq
	\probabilityBig{
		\mathrm{dist}(\Theta_{m-1}, \mathcal{M} \cap U) \geq \frac{\delta}{2},
		\mathcal{B}_{m-2}
	}
	\eqcom{ii}
	\leq
	\frac{4}{\delta^2} \C m^{-\sigma-\kappa}
	.
	\label{eqn:intereq5}
\end{equation}
Applying the bound in \eqref{eqn:intereq4} to \eqref{eqn:intereq5} yields
\begin{equation}
	P_2
	\leq
	\frac{4 \C }{T_m \lambda^2 \delta^4}
	\Bigl(
	m^{s} m^{-\sigma-\kappa} +  m^{-3s+1}
	\Bigr)
	.
	\label{eqn:proof_lemma7_3}
\end{equation}
This completes the bound for $P_2$ in \eqref{eqn:intereq1}.

Lastly, we return to \eqref{eqn:intereq1} and select the parameter.
Let us now combine \eqref{eqn:proof_lemma7_2} and \eqref{eqn:proof_lemma7_3} and return to bounding the left-hand side of \eqref{eqn:intereq1}.
Specifically, observe that we proved that
\begin{align}
	\probability{\mathrm{dist}(\Theta_{m}, \mathcal{M} \cap U) > \delta, \mathcal{B}_{m-1} }
	&
	\leq
	\frac{4 \C (\alpha_{m-1})^2}{T_m \delta^2}
	\bigl(
	m^{s} \mathbb{P}[\mathcal{B}_{m-1}] + m^{-3s + 1}
	\bigr)
	\nonumber \\ &
	\phantom{\leq}
	+
	\frac{4 \C}{T_m \delta^4}
	\bigl(
	m^{s-\sigma-\kappa} +  m^{-3s+1}
	\bigr)
	.
	\label{eqn:intereq6}
\end{align}

We now specify $s = \kappa + \sigma/2$ in \eqref{eqn:intereq6}. Without loss of generality we will again assume that $T_m=\ell m^{\sigma/2 + \kappa}$ instead of $\lfloor \ell m^{\sigma/2 + \kappa} \rfloor$---there is namely only a constant changed. By choosing the smallest exponents in $m$ in \eqref{eqn:intereq6} for all $m > 0$ we have
\begin{align}
	\probability{\mathrm{dist}(\Theta_{m}, \mathcal{M} \cap U) > \delta, \mathcal{B}_{m-1} }
	&
	\leq
	\Cl{15} \frac{a^2}{\delta^2 \ell m^{2\sigma}} \mathbb{P}[\mathcal{B}_{m-1}] + \frac{\Cr{15}}{\delta^4 \ell}
	\bigl(
	m^{-3\sigma - 4\kappa +1} + m^{- \sigma - \kappa}
	\bigr)
	.
\end{align}
Since $\sigma \in (2/3, 1)$, then $-3\sigma - 4\kappa +1 < - \sigma - \kappa$ for any $\kappa \geq 0$.
Upper bounding the leading orders in $m$
completes the proof of Lemma~\ref{lem:Probability_of_exit_and_Bm_is_bounded}.
\QuodEratDemonstrandum

\begin{remark}
	A Cauchy--Schwartz inequality in \eqref{eqn:intereq2} would only yield a factor $\probability{\mathcal{B}_{m-1}}^{1/2} > \probability{\mathcal{B}_{m-1}}$, which would not be sufficient. Similarly, we could have used \refLemma{lemma:expectation_L_is_bounded} directly and obtain a bound on $\expectation{ \indicator{\mathcal{B}_{m-2}} \mathcal{L}_4(S_{t_{m-1}})}$.
	However, this would not give an inequality that can be iterated inductively and is sharp enough.  We can directly simplify this term to obtain $\mathbb{P}(\mathcal{B}_{m-1})$ in the inequality only when $\mathcal{L}_4(S_{t_{m-1}})$ is bounded.
\end{remark}

\subsubsection{Proof of Lemma~\ref{lem:An_exit_event_is_unlikely_on_a_next_step}}
\label{sec:Proof_of_lemma__An_exit_event_is_unlikely_on_a_next_step}

In the work of \cite{fehrman2020convergence}, it is \cite[Lemma~23]{fehrman2020convergence} that establishes \cite[Equations~78--79]{fehrman2020convergence} directly.
Since \cite[Lemma~23]{fehrman2020convergence} is solely a geometric argument, and does not concern the stochastic process, it also applies in our Markovian setting.
\QuodEratDemonstrandum

\section{The Compact Case}
\label{sec:compact_case}

In the case that the set of maxima $\mathcal{M}$ is compact, we can improve the convergence rate of Theorem~\ref{prop:main_prop_convergence in probability_noncompact_case}. We will namely assume the following
\begin{assumption}[Compactness, Optional] \label{ass:compact}
	The open subset~$U$ defined in \Cref{ass:nondegenerate-maxima}
	is such that $\cM \cap U$ is compact.
\end{assumption}
Under this additional assumption we have the following
\begin{theorem}[Compact Case]
	\label{prop:main_prop_convergence in probability_compact_case}
	Suppose that Assumptions~\ref{ass:markov} to \ref{ass:compact} hold, except that \eqref{eqn:step_and_batchsizes} is now relaxed to allow for $\sigma \in (0, 1)$ and $\kappa \in [0, \infty)$.
	For every maximizer $\theta^\star \in \cM$,
	there exist constants $c > 0$ and $\alpha_0 > 0$ such that,
	for every $\alpha \in (0, \alpha_0]$, there exists a neighborhood $V$ of $\theta^\star$
	such that there exists $\ell_0 > 0$ such that for any $\ell \in [\ell_0, \infty)$,
	$m \in \bN_+$,
	and $\epsilon \in (0, 1)$,
	\begin{align}
		\prb{J(\Theta_m) < J^\star - \epsilon | \Theta_0 \in V}
		&\le
		c \left(
		\epsilon^{-2} m^{-\sigma - \kappa}
		+ \frac{m^{1 - \sigma - \kappa}}{\ell} + \frac{\alpha^2}\ell
		\right).
		\label{eqn:theorem_3_bound}
	\end{align}
\end{theorem}

The term proportional to $\alpha m^{-\kappa/2} + \alpha m^{1 - \sigma/2 - \kappa/2}\ell^{-\frac12}$ is not in Theorem~\ref{prop:main_prop_convergence in probability_compact_case} compared to Theorem~\ref{prop:main_prop_convergence in probability_noncompact_case}.
This term estimates the probability that the iterates escape $V$ along directions almost parallel to those of $\mathcal{M}$.
As it turns out, in the compact case such event cannot occur.
The bound in \eqref{eqn:theorem_3_bound} thus holds when the set of maxima is, for example, a singleton $\mathcal{M} \cap U = \{x_0\}$.

\subsection{Proof of Theorem~\ref{prop:main_prop_convergence in probability_compact_case}}
\label{sec:Proof_compact_case}
The proof is the same as with Theorem~\ref{prop:main_prop_convergence in probability_noncompact_case}, but we can omit the last term in \eqref{eqn:lemma_concentration_2} by showing that we can choose $\mathfrak{r}$ arbitrarily large.
The argument is as follows.
If the manifold $\mathcal{M} \cap U$ is compact, it can be covered by a finite number $k$ of local tubular neighborhoods $V_i = V_{\mathfrak{r}_i, \delta_i}(\theta_i)$ where $\theta_i \in \mathcal{M} \cap U$ and $\mathcal{M} \cap U \subset \cup_{i\in [k]} V_i$. Choose $\delta = \min_{i \in [k]} \delta_i$.
Then, any $\theta \in U$ such that $\mathrm{dist}(\theta, \mathcal{M} \cap U) < \delta$ will satisfy that $\mathfrak{p}(\theta) \in \mathcal{M} \cap U$, where $\mathfrak{p}$ is the unique local orthogonal projection on $\mathcal{M} \cap U$ from \eqref{eqn:definition_V_set_alternative}.
Now, from compactness, for any $\theta^{\star} \in \mathcal{M} \cap U$ there exists $\tilde{\mathfrak{r}} > 0$ such that $\mathcal{M} \cap U \subset B_{\tilde{\mathfrak{r}}}(\theta^{\star})$.
For any $\mathfrak{r} \geq \tilde{\mathfrak{r}}$ we thus have that $V_{\mathfrak{r}, \delta}(\theta^{\star}) = V_{\tilde{\mathfrak{r}}, \delta}(\theta^{\star})$ is a tubular neighborhood containing $\mathcal{M} \cap U$.
Then, we can choose $\mathfrak{r}$ arbitrarily large and conclude that the last term in the bound for the probability in Theorem~\ref{prop:main_prop_convergence in probability_noncompact_case} vanishes if $\mathcal{M} \cap U$ is a compact manifold.
More details on tubular neighborhoods and their existence for embedded manifolds can be found in the work of \cite{lee2012smooth}.

\section{Proof of Proposition~\ref{prop:lower_bound}}
\label{sec:Proof_lower_bound}

We consider the following setting. Let $D < 1$. We consider $\theta \in \R$ and a function $f$ such that in $\R \backslash [-D, D]$ satisfies $f(\theta) = 0$ and in $[-D/2, D/2]$ satisfies
\begin{equation}
	f(\theta) = 1 -\theta^2.
\end{equation}
In $[-D, -D/2] \cup [D/2, D]$, we define $f$ such that it is smoothly and monotonically interpolated between $[-D/2, D/2]$ and $\R \backslash [-D, D]$.

We let $H_m$ be such that $H_m = 0$ in $\R \backslash [-D, D]$. Hence, the set $\R \backslash [-D, D]$ is an absorbing set that is 1-suboptimal.
In $[-D/2, D/2]$, we will consider $\eta_m = \nabla f(\Theta_{m}) - H_m$ to be a random variable that, conditional on $\mathcal{F}_m$, is unbiased and has a second moment for all $m$ but approximates a heavy tailed random variable.
In particular, for $\beta > 0$, we define $\eta_m$ such that there exists $c > 0$ such that for any $m$, we have
\begin{equation}
	\mathbb{P}[|\eta_m| > s| \mathcal{F}_{m}] \geq \frac{c}{s^{2+\beta} T_m} \quad \text{for} \quad s > D.
	\label{eqn:Proof_lower_bound_1}
\end{equation}
Note that this constraint on $\eta_m$ is compatible with the finite second moment condition from \eqref{eqn:lower_bound_lemma4}.
If moreover $\alpha \leq 1$ and $\sqrt{\epsilon} < 2D$, then we can bound under the previous conditions
\begin{align}
	\prb{f(\Theta_m) < f^\star - \epsilon | \Theta_0 \in V} & \eqcom{i} \geq \prb{f(\Theta_m) < f^\star - \epsilon | \Theta_0 = \theta^{\min}} \nonumber \\
	& = \prb{ |\Theta_m| >  \sqrt{\epsilon} | \Theta_0 = \theta^{\min}}\nonumber \\
	& \eqcom{ii} \geq \prb{\sup_{l \leq m} |\Theta_l| > 2D | \Theta_0 = \theta^{\min}} \nonumber \\
	& \geq \prb{|\Theta_1| > 2D | \Theta_0 = \theta^{\min}} \nonumber \\
	&  = \prb{ |\theta^{\min} + \alpha_1\eta_1| > 2D | \Theta_0 = \theta_0} \nonumber \\
	&  \eqcom{iii} \geq \prb{ \alpha_1|\eta_1| > D | \Theta_0} \nonumber \\
	& \eqcom{\ref{eqn:Proof_lower_bound_1}} \geq c\frac{\alpha_1^{2+\beta}}{D^{2+\beta}T_1} \nonumber \\
	& \geq c\frac{\alpha^{2+\beta}}{D^{2+\beta}\ell},
\end{align}
where (i) we have used that for any $V = [-\delta, \delta]$ with $\delta < D$,
\begin{align}
	\prb{f(\Theta_m) < f^\star - \epsilon | \Theta_0 \in V} &= \int_{\theta \in V} \prb{f(\Theta_m) < f^\star - \epsilon | \Theta_0 =\theta}d \prb{\Theta_0 =\theta | \Theta_0 \in V} \nonumber \\
	& \geq \min_{\theta \in V} \prb{f(\Theta_m) < f^\star - \epsilon | \Theta_0 =\theta} \nonumber \\
	& \geq \prb{f(\Theta_m) < f^\star - \epsilon | \Theta_0 =\theta^{\min}}
	\label{eqn:Proof_lower_bound2}
\end{align}
for some $\theta^{\min} \in V$.
In (ii), we have used the fact that from the definition of $f$, we have the inclusion of events $ \{ \sup_{l \leq m} |\Theta_l| > 2D \} \in \{ |\Theta_m| > 2D \}$, since the set $\R \backslash [-D,D]$ is absorbent for the process $\{ \Theta_t\}_{t \geq 0}$.
In (iii), we have used that $\theta^{\min}$ belongs at least to $[-D, D]$, since otherwise it cannot be the minimum as defined in \eqref{eqn:Proof_lower_bound2}.
To guarantee that $\epsilon \in (0,1)$ we may choose $D=1/2$, for example.

\section{Proof of Proposition~\ref{prop:regret}}
\label{sec:proof_proposition_performance_gap}

We define the total number of samples $T$ up to epoch $m$ as
\begin{equation}
	T = t_{m+1} = \sum_{k=1}^{m} \ell k^{\kappa + \sigma/2} = \ell \Theta \Bigl(m^{\kappa + \frac{\sigma}{2} + 1} \Bigr).
	\label{eqn:sample_complexity_sum}
\end{equation}
and so for a given $T \geq 1$, define $m = \lceil (T/\ell)^{1/(\kappa + \frac{\sigma}{2} + 1)} \rceil$. Note that according to definition in \eqref{eqn:def_effective_epoch}, we have $m(T) \leq m \leq m(T)+1$.

We show first an intermediate result in Lemma~\ref{lem:regret_1}. 
Recall the definiton of the set $V$ in \eqref{eqn:definition_V_set}. 
Since the closure of $V$ is compact we have that $\sup_{\theta \in V} |J(\theta)|$ exists. From Theorem~\ref{prop:main_prop_convergence in probability_noncompact_case} we directly obtain:%
\begin{lemma}
	Under the same assumptions and setting as in Theorem~\ref{prop:main_prop_convergence in probability_noncompact_case},
	assume either (i) there exists some $b > 0$ such that $|r(s,a)| < b$ for any $(s,a) \in \mathcal{S} \times \mathcal{A}$ or (ii) the event $\mathcal{B}_m=\{\Theta_t \in V: t \in [m]\}$ holds.
	Under condition (i) we have that
	\begin{align*}
		\E[J^\star - J(\Theta_m)| \Theta_0 \in V]
		&\leq
		3 (\mathcal{L}^{\star})^{\frac{1}{3}}\Bigl(\frac{cb}{2}\Bigr)^{\frac{1}{3}} m^{-\frac{(\sigma + \kappa)}{3}}
		+ 2 \frac{bc}{\ell} m^{1 - (\sigma +\kappa)} + 2bc\frac{\alpha^2}{\ell} \nonumber \\
		& \phantom{ m^{1 - (\sigma +\kappa)}  m^{1 - (\sigma +\kappa)}  m^{1 - (\sigma +\kappa)} } + 2bc \alpha m^{-\kappa/2}
		+ 2\frac{bc\alpha}{\ell} m^{1 - (\sigma + \kappa)/2}.
	\end{align*}
	Under condition (ii), we have that if $b = \sup_{\theta \in V} |J(\theta)|$ and $\probability{\mathcal{B}_m} > 1/2$, then
	\begin{equation}
		\E[J^\star - J(\Theta_m)| \mathcal{B}_m] \leq 3 (\mathcal{L}^{\star})^{\frac{1}{3}} (cb)^{\frac{1}{3}} m^{-\frac{(\sigma + \kappa)}{3}}.
	\end{equation}
	\label{lem:regret_1}
\end{lemma}
\begin{proof}
	Under condition (i), optimizing the following bound over $\epsilon > 0$,
	\begin{equation}
		\E[J^\star - J(\Theta_m)| \Theta_0 \in V] \leq \prb{J(\Theta_m) < J^\star - \epsilon | \Theta_0 \in V}2b + \epsilon,
	\end{equation}
	immediately yields the result by using the bound from \eqref{eqn:theorem_2_bound}.
	For condition (ii), using \eqref{eqn:proof_theorem2_3}, we have directly that
	\begin{align}
		\frac{1}{2} \probability{ \{J^\star - J(\Theta_m)) > \epsilon \} \big| \mathcal{B}_m} &\leq \probability{ \{J^\star - J(\Theta_m)) > \epsilon \} \big| \mathcal{B}_m}  \probability{\mathcal{B}_m} \nonumber\\
		&= \probability{ \{J^\star - J(\Theta_m)) > \epsilon \} \cap \mathcal{B}_m} \nonumber \\
		&\leq c\epsilon^{-2} \mathcal{L}^{\star} m^{-(\sigma + \kappa)}.
		\label{eqn:proof_regret_1}
	\end{align}
	Finally, we repeat the same argument as in part (i) using the new bound $b$.
\end{proof}

\subsection{Proof of Proposition~\ref{prop:regret}(i)}
Recall that both $V$ and $\alpha$ are fixed. 
Let $\tilde{\Theta}_t$ for $t \in [T]$ be defined as in \Cref{sec:regret}. 
Then using Lemma~\ref{lem:regret_1} and the definition of $m$ in terms of $T$ in \eqref{eqn:sample_complexity_sum}, we have that there exists a constant $c>0$ independent of $\ell \geq \ell_0$ such that for any $T$ we have
\begin{align}
	\E[J^\star - J(\Theta_{m(T)})| \Theta_0 \in V]
	\leq
	c\Bigl(\Bigl(\frac{T}{\ell}\Bigr)^{-\frac{(\sigma + \kappa)}{3(\kappa + \frac{\sigma}{2} + 1)}} &+ \frac{1}{\ell^{1/2}}\Bigl(\frac{T}{\ell}\Bigr)^{\frac{1 - (\sigma +\kappa)/2}{(\kappa + \frac{\sigma}{2} + 1)}} \nonumber \\
	& + \frac{1}{\ell}\Bigl(\frac{T}{\ell}\Bigr)^{\frac{1 - \sigma +\kappa}{(\kappa + \frac{\sigma}{2} + 1)}}  + T^{\frac{-\kappa}{2(\kappa + \frac{\sigma}{2} + 1)}} + \frac{\alpha^2}{\ell} \Bigr).
	\label{eqn:sample_complexity_orders_T}
\end{align}
Note that by looking at the orders in \eqref{eqn:sample_complexity_orders_T} we can make $\kappa$ large to obtain an approximation for the exponents. I particular, for any $\zeta > 0$ there exists $\kappa_0(\zeta) > 0$ such that if $\kappa \geq \kappa_0(\epsilon)$, then
\begin{equation}
	\E[J^\star - J(\Theta_{m(T)})| \Theta_0 \in V] \leq c\Bigl( (\mathcal{L}^{\star})^{\frac{1}{3}} \ell^{\frac{1}{3} + \zeta}T^{-\frac{1}{3}+\zeta} + \ell^{1/2 + \zeta}T^{-\frac{1}{2}+\zeta} + \ell^{2/3 + \zeta}T^{-1+ \zeta}  + \frac{\alpha^2}{\ell} \Bigr).
	\label{eqn:sample_complexity_max_order_T}
\end{equation}

\subsection{Proof of Proposition~\ref{prop:regret}(ii)}
Repeating the same argument as in (i) we obtain that
\begin{equation}
	\E[J^\star - J(\Theta_{m(T)})| \mathcal{B}_{m(T)}] \leq c (\mathcal{L}^{\star})^{\frac{1}{3}} \ell^{\frac{1}{3} + \zeta}T^{-\frac{1}{3}+\zeta}.
\end{equation}
The bound on the probability $\probability{\mathcal{B}_{m(T)}}$ is given in \eqref{eqn:lemma_concentration_2} together with the remark on the exponents thereafter. In terms of $T/\ell$, this observation yields
\begin{equation}
	\probability{\mathcal{B}_{m(T)}} \geq 1 - c\Bigl(\frac{\alpha^2}{\ell} + \ell^{1/2 + \zeta}T^{-\frac{1}{2}+\zeta} + \ell^{2/3 + \zeta}T^{-1+\zeta} \Bigr).
	\label{eqn:proof_regret_4}
\end{equation}
Finally, we make $\ell_0$ large enough to guarantee that if $T \geq T_0$ for some $T_0 > 0$, we have $\probability{\mathcal{B}_{m(T)}} \geq 1/2$. Then note that $\probability{\mathcal{B}_{m(k)}} \geq \probability{\mathcal{B}_{m(T_0)}}$ for any $k \leq T_0$.

\section{Proof of Corollary~\ref{cor:regret}}
\label{sec:proof_regret_corollary}
We will use the same notation as in the proof of Proposition~\ref{prop:regret} in \Cref{sec:proof_proposition_performance_gap}. 
Moreover, for an epoch $l$ corresponding samples $t \in [t_{l}, t_{l+1}]$, we will consider that $\Theta_{l}$ is fixed for any sample in the epoch and define the parameter corresponding to sample $t$ as $\tilde{\Theta}_t = \Theta_{m(l)}$, where $m(l)$ was defined in \eqref{eqn:def_effective_epoch}.
We need first the following inequalities.
\begin{lemma}
	Under the same assumptions and notation as in Theorem~\ref{prop:main_prop_convergence in probability_noncompact_case}, we fix $\alpha$ and $\delta$ satisfying such assumptions. 
	Then for any $1 > \zeta > 0$ there exists $\kappa(\zeta) \geq 0$, $c>0$ and $\ell_0>0$ such that for any $\ell \geq \ell_0$, $\kappa \geq \kappa(\zeta)$ and $T \geq 1$,
	(i) if $r(s,a)$ is bounded,
	\begin{multline*}
		\expectationBig{ TJ^{\star} - \sum_{t=1}^{T} r(S_t,A_t) \Big| \Theta_0 \in V} \leq \expectationBig{T J^{\star} - \sum_{t=1}^{T} J(\tilde{\Theta}_t)  \Big| \Theta_0 \in V} \\
		+ c \bigl( \alpha^2T\ell^{-1} + \ell^{1/2 + \zeta}T^{\frac{1}{2}+\zeta} + \ell^{2/3 + \zeta} T^{\zeta} + T^{\zeta}\bigr),
	\end{multline*}
	and (ii) if $r(s,a)$ is unbounded,
	\begin{equation*}  
		\expectationBig{T J^{\star} - \sum_{t=1}^{T}  r(S_t,A_t) \Big| \mathcal{B}_{m(T)}} \leq 
		\expectationBig{T J^{\star} - \sum_{t=1}^{T} J(\tilde{\Theta}_t)  \Big| \mathcal{B}_{m(T)}}  + c \Bigl(\frac{T}{\ell}\Bigr)^{\frac{1}{\sigma/2+\kappa}}.
	\end{equation*}
	\label{lemma:regret_2}
\end{lemma}
\begin{proof}
	We show (ii) first. We use that
	\begin{equation}
		\expectationBig{T J^{\star} - \sum_{t=1}^{T} r(S_t,A_t) \Big| \mathcal{B}_{m(T)}} = \expectationBig{T J^{\star} - \sum_{t=1}^{T} J(\tilde{\Theta}_t)  \Big| \mathcal{B}_{m(T)}} + \expectationBig{ \sum_{t=1}^{T} J(\tilde{\Theta}_t) -  r(S_t, A_t)  \Big| \mathcal{B}_{m(T)}}.
		\label{eqn:proof_regret_3}
	\end{equation}
	From the same argument as that of \eqref{eqn:proof_regret_4}, we may pick $\ell_0$ such that $\mathbb{P}[\mathcal{B}_{m(T)}] > 1/2$ and for some $c>0$
	\begin{equation}
		\Big| \expectationBig{\sum_{t=1}^{T} J(\tilde{\Theta}_t) -  r(S_t, A_t)  \Big| \mathcal{B}_{m(T)}}\Big|  \leq c \Big|\expectationBig{\indicator{\mathcal{B}_{m(T)} }\sum_{t=1}^{T} J(\tilde{\Theta}_t) -  r(S_t, A_t)  \Big| \Theta_0 \in V}\Big|.
		\label{eqn:proof_regret_5}
	\end{equation}
	We need to bound only the last term in \eqref{eqn:proof_regret_5}. 
	From Assumption \ref{ass:growth_condition}, we have that $|r(s,a)| \leq c \mathcal{L}(s,a)$. 
	From \eqref{eqn:concentration_biased_gradient} we obtain for epoch $m$ that if $\Theta_m \in V$ there exists a constant $C>0$ such that
	\begin{equation}
		\Big| \expectationBig{\sum_{t=t_{m}}^{t_{m+1}} J(\tilde{\Theta}_t) - r(S_t,A_t) \Big| \mathcal{F}_m} \Big| \leq C \mathcal{L}_{4}(S_{t_m}, A_{t_m}).
		\label{eqn:proof_lemma_regret_4}
	\end{equation}
	Recall from \eqref{eqn:def_effective_epoch} that $m(T) = \min\{ m \in \N: \ell m^{\sigma/2 + \kappa} \geq T\}$. 
	From Lemma~\ref{lemma:expectation_L_is_bounded}, we know that there exists a constant $c>0$ such that for any $n \geq 1$ $\expectation{\mathcal{L}_4(S_{t_n}, A_{t_n}) \indicator{\mathcal{B}_n}} \leq c$. Let $\mathcal{F}_{n}$ be defined as in \eqref{eqn:sigma-algebra}. 
	Recall that $\indicator{\mathcal{B}_m} \leq \indicator{\mathcal{B}_n}$ if $n < m$. By using in the following the tower property of the conditional expectation in (i) we have
	\begin{align}
		&\Big| \expectationBig{\indicator{\mathcal{B}_{m(T)}} \sum_{t=1}^{T} J(\tilde{\Theta}_t) - r(S_t,A_t) \Big| \Theta_0 \in V} \Big| \nonumber \\
		&\leq \sum_{n=1}^{m(T)} \Big| \expectationBig{\indicator{\mathcal{B}_n} \sum_{t=t_n}^{t_{n+1}} J(\tilde{\Theta}_t) - r(S_t,A_t)\Big| \Theta_0 \in V} \Big| \nonumber \\
		& \eqcom{i} \leq \sum_{n=1}^{m(T)}\expectationBig{ \Big| \expectationBig{\indicator{\mathcal{B}_n} \sum_{t=t_n}^{t_{n+1}} J(\tilde{\Theta}_t) - r(S_t,A_t)  \Big| \mathcal{F}_n}\Big|~ \Big| \Theta_0 \in V}  \nonumber \\
		& \eqcom{\ref{eqn:proof_lemma_regret_4}} \leq c\sum_{n=1}^{m(T)}\expectationBig{ \indicator{\mathcal{B}_n} \mathcal{L}_4(S_{t_n}, A_{t_n}) \Big| \Theta_0 \in V}  \nonumber \\
		& \eqcom{Lemma~\ref{lemma:expectation_L_is_bounded}} \leq c\sum_{m=1}^{m(T)} C 
		\leq c \Bigl(\frac{T}{\ell}\Bigr)^{\frac{1}{\sigma/2+\kappa}}.
		\label{eqn:proof_lemma_regret_1}
	\end{align}
	Substituting \eqref{eqn:proof_lemma_regret_1} in \eqref{eqn:proof_regret_5} yields the result.
	We now show (i) using a similar argument. 
	First note that for $n \in [m(T)]$ we have
	\begin{align}
		&\Big|\expectationBig{\sum_{t=t_n}^{t_{n+1}} J(\tilde{\Theta}_t) - r(S_t,A_t) \Big| \Theta_0 \in V}\Big| \nonumber \\
		&\leq \Big|\expectationBig{\indicator{\mathcal{B}_n} \sum_{t=t_n}^{t_{n+1}} J(\tilde{\Theta}_t) - r(S_t,A_t) \Big| \Theta_0 \in V}\Big|
		+ \Big|\expectationBig{\indicator{\overline{\mathcal{B}_n}} \sum_{t=t_n}^{t_{n+1}} J(\tilde{\Theta}_t) - r(S_t,A_t) \Big| \Theta_0 \in V}\Big| \nonumber\\
		& \eqcom{a} \leq C + c\mathbb{P}[\overline{\mathcal{B}_n}] \label{eqn:proof_corollary1} \\
		& \leq C + c\ell n^{\sigma/2+\kappa}\Bigl(\frac{\alpha^2}{\ell}  + \frac{n^{1 - (\sigma +\kappa)}}{\ell}  + n^{-\kappa/2}
		+ \frac{n^{1 - (\sigma + \kappa)/2}}{\ell} \Bigr), \nonumber
	\end{align}
	where in (a) we have used the same argument as in \eqref{eqn:proof_lemma_regret_1} for the first term, and for the second term, we have used that since the reward is bounded, $|J(\tilde{\Theta}_t) - r(S_t,A_t)|$ is also bounded, regardless of the stability of $\tilde{\Theta}_t$. 
	We are left with a constant times $\mathbb{P}[\overline{\mathcal{B}_n}]$ for the second term.
	We add the remaining terms in \eqref{eqn:proof_corollary1} for $n \in [m(T)]$ and use the inequality $\sum_{i=1}^{h} i^{\eta} \leq C h^{\eta+1}$ for $\eta \geq 0$. 
	Setting $m(T)$ in terms of $T$ according to \eqref{eqn:sample_complexity_sum}, we are left with
	\begin{align}
		& \Big|\expectationBig{\sum_{t=t_n}^{t_{n+1}} J(\tilde{\Theta}_t) - r(S_t,A_t) \Big| \Theta_0 \in V}\Big| \nonumber \\
		& \leq c \Bigl( \Bigl(T^{\frac{1}{\sigma/2+\kappa}} \ell^{-\frac{1}{\sigma/2+\kappa}} +  T \frac{\alpha^2}{\ell} + \ell^{1/2 + \zeta}T^{\frac{1}{2}+\zeta} + \ell^{2/3 + \zeta} T^{\zeta} + T^{\zeta}\Bigr) \nonumber \\
		& \leq c \bigl( \alpha^2T\ell^{-1} + \ell^{1/2 + \zeta}T^{\frac{1}{2}+\zeta} + \ell^{2/3 + \zeta} T^{\zeta} + T^{\zeta}\bigr),
	\end{align}
	where we have used that $T^{-1/(\sigma/2+\kappa)} \ell^{-1/(\sigma/2+\kappa)} \leq T/\ell$ for any $T \geq \ell$.
\end{proof}

\subsection{Proof of Corollary~\ref{cor:regret}(i)}

Let $\ell \geq \ell_0$ be fixed, where $\ell_0$ is given by the conditions of Theorem~\ref{prop:main_prop_convergence in probability_noncompact_case}. We add for each $t \in [T]$ the performance gap of Proposition~\ref{prop:regret} which yields
\begin{align}
	&\expectationBig{T J^{\star} - \sum_{t=1}^{T} J(\tilde{\Theta}_t)\Big| \Theta_0 \in V} \nonumber \\
	&\leq \sum_{t=1}^{T} c\bigl(  (\mathcal{L}^{\star})^{\frac{1}{3}} \ell^{\frac{1}{3} + \zeta}t^{-\frac{1}{3}+\zeta} +  \alpha^2 \ell^{-1} + \ell^{1/2 + \zeta}t^{-\frac{1}{2}+\zeta} + \ell^{2/3 + \zeta} t^{-1 + \zeta} + t^{-1 + \zeta}\bigr) \nonumber \\
	&\leq c\bigl( (\mathcal{L}^{\star})^{\frac{1}{3}}  \ell^{\frac{1}{3} + \zeta}T^{\frac{2}{3}+\zeta} +  \alpha^2 \ell^{-1} + \ell^{1/2 + \zeta}T^{\frac{1}{2}+\zeta} + \ell^{2/3 + \zeta} T^{\zeta} + T^{\zeta}\bigr).
	\label{eqn:proof_corollary2}
\end{align}
Use now Lemma~\ref{lemma:regret_2} together with \eqref{eqn:proof_corollary2}. In this manner we obtain the bound:
\begin{multline*}
	\expectationBig{TJ^{\star} - \sum_{t=1}^{T} r(S_t,A_t) \Big| \Theta_0 \in V} \\
	\leq  c\bigl( (\mathcal{L}^{\star})^{\frac{1}{3}} \ell^{\frac{1}{3} + \zeta}T^{\frac{2}{3}+\zeta} + \alpha^2 T\ell^{-1} + \ell^{1/2 + \zeta}T^{\frac{1}{2}+\zeta} + \ell^{2/3 + \zeta} T^{\zeta} + T^{\zeta}\bigr).
\end{multline*}
Then, for $\ell \geq \ell_0$ fixed and for any $T > 0$ the following holds
\begin{equation}
	\expectationBig{TJ^{\star} - \sum_{t=1}^{T} r(S_t,A_t) \Big| \Theta_0 \in V} \leq c\Bigl((\mathcal{L}^{\star})^{\frac{1}{3}} T^{\frac{2}{3}+\zeta} + \frac{\alpha^2}{\ell} T \Bigr).
\end{equation}
Similarly, if we choose $\ell$ depending on a given $T$ fixed, then setting $\ell = T^{1/4}$ we obtain
\begin{equation}
	\expectationBig{TJ^{\star} - \sum_{t=1}^{T} r(S_t,A_t) \Big| \Theta_0 \in V} \leq cT^{\frac{3}{4}+\zeta}.
\end{equation}

\subsection{Proof of Corollary~\ref{cor:regret}(ii)}

Let $\ell \geq \ell_0$ be fixed, where $\ell_0$ is given by the conditions of Theorem~\ref{prop:main_prop_convergence in probability_noncompact_case} and satisfies the same conditions as \eqref{eqn:proof_regret_4}. We repeat the argument used in (i) by using Proposition~\ref{prop:regret} and Lemma~\ref{lemma:regret_2}. We obtain that if $\sigma/2+\kappa \geq 3/2$ then
\begin{align}
	\expectationBig{ TJ^{\star} - \sum_{t=1}^{T} r(S_t,A_t) \Big| \mathcal{B}_{m(T)}} &\leq c (\mathcal{L}^{\star})^{\frac{1}{3}} \ell^{\frac{1}{3} + \zeta}T^{\frac{2}{3}+\zeta} + c \Bigl(\frac{T}{\ell}\Bigr)^{\frac{1}{\sigma/2+\kappa}} \leq
	c (\mathcal{L}^{\star})^{\frac{1}{3}} T^{\frac{2}{3}+\zeta}.
\end{align}

\section{Proof of Proposition~\ref{prop:regularization}}
\label{sec:Proof_regularization}

To prove the proposition we will show that for almost all $\tilde{\pi}$ in the Lebesgue measure of the class of policies defined in \eqref{eqn:defintion_softmax_policy}, the function $J_{\tilde{\pi}}(\theta)$ is Morse.
Morse functions are smooth functions $f$ such that every critical point of $f$ is nondegenerate, that is, for any $x$ such that $\nabla_x f =0$ we have that $\mathrm{Hess}_{x} f$ is nonsingular. Hence, all critical points are isolated.
If the function $J_{\tilde{\pi}}(\theta)$ is Morse and furthermore satisfies that $J_{\tilde{\pi}}(\theta) \to -\infty$ as $|\theta| \to \infty$, it will then have bounded isolated maxima.

We show first that for almost all $\tilde{\pi}$, the function $J_{\tilde{\pi}}(\theta)$ is a Morse function. To do so, we will implicitly use the fact that Morse functions are dense and form an open subset in the space of smooth functions (see \citealt{nicolaescu2011invitation}).

We introduce first notation.
For a finite dimensional smooth manifold $M$, we denote by $T_x M$ and $T^*_x M$ the tangent and cotangent spaces at $x \in M$, respectively.
When $M = \R^{u}$, for $f: \R^{u} \to \R$ we will denote the (covariant) derivative and gradient of $f$ at $x$ by $d_x f \in T^{*}_x M$ and $\nabla_x f \in T_x M$, respectively. In local coordinates $(w_1, \ldots, w_{u})$, we have namely
\begin{align}
	d_x f &= \sum_{i=1}^{u} \frac{\partial f(x)}{\partial w_i} dw_i, \nonumber \\
	\nabla_x f &= \sum_{i=1}^{u} \frac{\partial f(x)}{\partial w_i} \frac{d}{d w_i},
\end{align}
where $dw_i(\frac{d}{d w_i}) = \indicator{i=j}$. In this notation and since $M = \R^{u}$, we have then
\begin{equation}
	d_x (d f) = \sum_{i=1}^{u} d_x \Bigl(\frac{\partial f(x)}{\partial w_i} dw_i \Bigr) = \sum_{i=1}^{u}\sum_{j=1}^{u} \frac{\partial^2 f(x)}{\partial w_j \partial w_i} dw_j \otimes dw_i  = \mathrm{Hess}_x f \in T^{*}_x M \otimes T^{*}_x M.
\end{equation}

We require the following lemmas and definitions.
\begin{definition}
	Let $M$ and $N$ be two manifolds and let $B$ be a submanifold of $N$. We say a smooth map $f: M \to N$ is transversal to $B$ if for every point $x \in M$ such that $f(x) \in B$ we have
	\begin{equation}
		d_{x} f(T_x M) + T_{f(x)} B = T_{f(x)} N.
	\end{equation}
\end{definition}

We will use the following result that has is its core an application of Sard's theorem that states that in a map between smooth manifolds, the set of critical points has measure zero in the image.
\begin{lemma}[Parametric Transversality Theorem \citep{guillemin2010differential}]
	Let $Z, M$ and $N$ be smooth manifolds and let $B$ be a smooth submanifold of $N$.
	Let $F:Z \times M \to N$ be a smooth submersion, that is, the differential map is surjective everywhere.
	If $F$ is transversal to $B$, then for almost every $z \in Z$, the map
	\begin{equation}
		F_{z}(m) = F(z,m)
	\end{equation}
	is transversal to $B$.
	\label{lemma:parametric_transversality}
\end{lemma}

When appropriate, we will make explicit the dependence of $v \in T^{*}_x M$ on $x$ by writing $(x,v) \in T^{*}_x M$. We can now show the following,
\begin{lemma}
	Let $M = \R^{u}$ and let $f : M \to \R$ be a smooth map.
	Consider the map $\tilde{f}: M \to T^{*}M$ given for $x \in M$ by
	\begin{equation}
		\tilde{f}(x) =(x, d_x f) \in T^*_x M.
	\end{equation}
	Let $B \subset T^{*}M$ be the zero section submanifold, that is, $B(x) = (x, 0) \in T^{*}_x M$ for every $x$.
	Then $x$ is a nondegenerate critical point of $f$ if and only if $\tilde{f}$ is transversal to $B$ at $x$ and $\nabla_x f = 0$.
	\label{lemma:transversality_to_nondegeneracy}
\end{lemma}
\begin{proof}
	$x$ is a critical nondegenerate point if and only if $\nabla_x f = 0$ and $\mathrm{Hess}_x f \in T^*_x M \otimes T^*_x M$ is nonsingular. For any $\nu \in T_x M$, we have then that
	\begin{equation}
		d_x\tilde{f}(\nu) = (\nu, \mathrm{Hess}_{x}f(\nu)).
	\end{equation}
	By definition, $\tilde{f}$ is transversal to $B$ if and only if for every $x \in M$,
	\begin{align}
		d_{x}\tilde{f}(T_x M) + T_{x} M \oplus 0 &= (Id \oplus \mathrm{Hess}_{x}(f))(T_{x} M) + T_{x} M \oplus 0  \nonumber \\
		&= T_{x} M \oplus \mathrm{Hess}_{x}f(T_{x} M) \nonumber\\
		&=T_{x} M \oplus T^*_{x} M,
	\end{align}
	which is true if and only if $\mathrm{Hess}_{x}f$ is nonsingular.
\end{proof}

From the last two lemmas it follows that by adding an appropriate perturbation to a function, the perturbed function is nondegenerate.
This result is well-known in the literature in the context of genericity of Morse functions and can be generalized to general smooth manifolds; see \citep{guillemin2010differential}.

\begin{lemma}
	Let $M = \R^{u}$. Let $f: M \to \R$ and $g_i: M \to \R$ for $i \in [l]$ be smooth functions such that for every $x \in M$, $\mathrm{span}(\{d_x g_i\}_{i=1}^{l}) = T^*_x M$. Then for almost every $z=(z_1, \ldots, z_{l}) \in \R^{u}$ we have that
	\begin{equation}
		f_{z}(\cdot) = f(\cdot) + \sum_{i=1}^l z_i g_i(\cdot)
	\end{equation}
	is a Morse function.
	\label{lemma:morse_from_transversality}
\end{lemma}
\begin{proof}
	Define the smooth function $F: \R^{l} \times M \to T^{*}M$ given by
	\begin{equation}
		F(z,x) =(x, d_x f + \sum_{i=1}^l z_i d_x g_i) = (x, d_x f_z).
	\end{equation}
	The derivative of this map at $(z,x)$ evaluated at $(\eta,\chi) \in T_z\R^{l} \times T_x M$ is then
	\begin{equation}
		d_{(z,x)}F(\eta,\chi) = (\chi, \mathrm{Hess}_x f_{z}(\chi) + \sum_{i=1}^{l} \eta_i d_x g_i) \in T_{F(z, x)}(T^* M) \simeq T_{x} M \oplus T^{*}_{x} M.
	\end{equation}
	For every $x$, we have $\mathrm{span}(\{d_x g_i\}_{i=1}^{l}) = T^*_x M$, then $d_{(z,x)}F(T_z\R^{l}, T_x M) = T_{F(z, x)}(T^* M)$ and $d_{(z,x)}F$ is surjective.
	Thus, $F$ is a submersion and is therefore transversal to the zero section of $T^{*}M$ and by Lemma~\ref{lemma:parametric_transversality} for almost every $z \in Z$ the map $F_z(x) = F(z,x)$ is transversal to the zero section of $T^{*}M$. Finally, by Lemma~\ref{lemma:transversality_to_nondegeneracy} we can conclude that for almost every $z \in Z$, the critical points of $f_z$ are nondegenerate, that is, $f_z$ is a Morse function.
\end{proof}

We are now in position to show the proposition. Recall from the definition of the policy in \eqref{eqn:defintion_softmax_policy} that there is an index set $\mathcal{I}$ and a function $h:\mathcal{S} \to \mathcal{I}$ that determines the parameter dependence of $\{ \theta_{i,a}: (i,a) \in \mathcal{I} \times \mathcal{A} \}$. For $s \in \mathcal{I}$, let $z_{(a,i)} = \tilde{\pi}(a|i)$ and denote $\tilde{\zeta}(i) = \sum_{s \in \mathcal{S}: h(s) = i} \zeta(s)$. We can write
\begin{align}
	d_\theta \mathcal{R}_{\tilde{\pi}}(\theta) &= b \sum_{s \in \mathcal{S}} \zeta(s)
	\sum_{a \in \mathcal{A}} \tilde{\pi}(a|s) d_\theta \log(\pi(a|s, \theta)) \nonumber \\
	&= b \sum_{s \in \mathcal{S}} \zeta(s) \sum_{a \in \mathcal{A}} \tilde{\pi}(a|s) \Bigl( \sum_{a^{\prime} \in \mathcal{A}} (\indicator{a = a'} - \pi(a' | s, \theta))d\theta_{h(s),a^{\prime}} \Bigl) \nonumber \\
	&= b \sum_{s \in \mathcal{S}} \zeta(s) \sum_{a^{\prime} \in \mathcal{A}} (\tilde{\pi}(a|s) - \pi(a' | s, \theta))d\theta_{h(s),a^{\prime}} \nonumber \\
	&= b \sum_{i \in \mathcal{I}} \sum_{a \in \mathcal{A}} \tilde{\zeta}(i)(\tilde{\pi}(a|i) - \pi(a | i, \theta))d\theta_{i,a} \nonumber \\
	&= b \sum_{(i,a) \in \mathcal{I} \times \mathcal{A}} \tilde{\zeta}(i)(z_{(i,a)} - \pi(a | i, \theta))d\theta_{i,a}.
	\label{eqn:Proof_regularization1}
\end{align}
If $\tilde{\zeta}(i) > 0$ for all $i \in \mathcal{I}$, it is clear from \eqref{eqn:Proof_regularization1} that the terms $\{d \theta_{i,a}\}_{(i,a) \in \mathcal{I} \times \mathcal{A}}$ span $T^{*}_{\theta} \R^{|\mathcal{A}| \times |\mathcal{I}|}$ for each $\theta$, since $\pi(a|s,\theta) \neq 0$ for any finite $\theta$.  By Lemma~\ref{lemma:morse_from_transversality} and the assumption on $\zeta$, we immediately obtain that for almost all policies $\tilde{\pi}$, the function
\begin{equation}
	J_{\tilde{\pi}}(\theta) = J(\theta) - b \mathcal{R}_{\bar{\pi}}(\theta).
	\label{eqn:proof_proposition3_1}
\end{equation}
is Morse and has nondegenerate critical points---including the maximum.
Finally, the set of maxima of \eqref{eqn:proof_proposition3_1} will be nonempty. Indeed, the function $-b\mathcal{R}_{\bar{\pi}}(\theta) \to -\infty$ whenever for any $s \in \mathcal{S}$, $\pi(\,\cdot\,|s) \to \partial \Delta(\mathcal{S})$. Thus, by continuity, the set of maxima belongs to a compact set.

\vskip 0.2in
\bibliography{paper}

\end{document}